\documentclass{article}

\usepackage[accepted]{icml2024}

\usepackage{amsmath,amsfonts,bm,amsthm,amssymb,mathrsfs,algorithm,algorithmic,url,xcolor,mathtools}
\usepackage{booktabs,enumitem,multicol,caption,subcaption,pifont,tikz,listings,wrapfig,lipsum,circledsteps,arydshln,hyperref,multirow}
\usepackage{relsize,footnote,makecell,graphicx,microtype,thmtools}

\hypersetup{
    colorlinks,
    linkcolor={blue!50!black},
    citecolor={blue!50!black},
    urlcolor={red!50!black}
}

\theoremstyle{plain}
\newtheorem{theorem}{Theorem}[section]

\theoremstyle{definition}

\newtheorem{assumption}[theorem]{Assumption}
\theoremstyle{remark}
\newtheorem{remark}[theorem]{Remark}
\newtheorem*{theorem*}{Theorem}

\def\Secref#1{Section~\ref{#1}}
\def\eqref#1{(\ref{#1})}
\def\eqref#1{Eq.~(\ref{#1})}
\def\1{\bm{1}}

\def\vzero{{\bm{0}}}

\def\vc{{\bm{c}}}

\def\vx{{\bm{x}}}

\def\mI{{\bm{I}}}

\DeclareMathAlphabet{\mathsfit}{\encodingdefault}{\sfdefault}{m}{sl}
\SetMathAlphabet{\mathsfit}{bold}{\encodingdefault}{\sfdefault}{bx}{n}

\def\gB{{\mathcal{B}}}

\def\gD{{\mathcal{D}}}

\def\gL{{\mathcal{L}}}
\def\gM{{\mathcal{M}}}
\def\gN{{\mathcal{N}}}

\def\gP{{\mathcal{P}}}

\def\sA{{\mathbb{A}}}

\def\sS{{\mathbb{S}}}

\newcommand{\E}{\mathbb{E}}

\newcommand{\R}{\mathbb{R}}

\newcommand{\KL}{D_{\mathrm{KL}}}

\newcommand{\cbr}[1]{\left\{#1\right\}}
\newcommand{\br}[1]{\left(#1\right)}
\newcommand{\sbr}[1]{\left[#1\right]}
\newcommand{\given}{\,|\,}

\newcommand*\intd{\mathop{}\!\mathrm{d}}

\newcommand{\ie}{\textit{i.e.}}
\newcommand{\eg}{\textit{e.g.}}
\newcommand{\wrt}{\textit{w.r.t.}}

\newcommand{\myparagraph}[1]{\textbf{#1~~}}

\newcommand{\ord}{{\mathrm{ord}}}

\definecolor{brickred}{rgb}{0.8, 0.25, 0.33}
\definecolor{darkspringgreen}{rgb}{0.09, 0.45, 0.27}
\definecolor{applegreen}{rgb}{0.55, 0.71, 0.0}
\definecolor{brightmaroon}{rgb}{0.76, 0.13, 0.28}
\definecolor{burgundy}{rgb}{0.5, 0.0, 0.13}
\usepackage[capitalize,noabbrev]{cleveref}

\usepackage[textsize=tiny]{todonotes}

\icmltitlerunning{A Dense Reward View on Aligning Text-to-Image Diffusion with Preference}

\begin{document}

\twocolumn[
\icmltitle{A Dense Reward View on Aligning Text-to-Image Diffusion with Preference}



\icmlsetsymbol{equal}{*}

\begin{icmlauthorlist}
\icmlauthor{Shentao Yang}{equal,yyy}
\icmlauthor{Tianqi Chen}{equal,yyy}
\icmlauthor{Mingyuan Zhou}{yyy}
\end{icmlauthorlist}

\icmlaffiliation{yyy}{The University of Texas at Austin}

\icmlcorrespondingauthor{Shentao Yang}{\texttt{shentao.yang@mccombs.utexas.edu}}
\icmlcorrespondingauthor{Mingyuan Zhou}{\texttt{mingyuan.zhou@mccombs.utexas.edu}}

\icmlkeywords{Text-to-Image Diffusion Model, Preference Alignment, Dense Reward for Direct Preference Optimization (DPO), Sequential Generation}

\vskip 0.3in
]



\printAffiliationsAndNotice{\icmlEqualContribution} % otherwise use the standard text.

\begin{abstract}
Aligning text-to-image diffusion model (T2I) with preference has been gaining increasing research attention.
While prior works exist on directly optimizing T2I by preference data, these methods are developed under the bandit assumption of a latent reward on the entire diffusion reverse chain, while ignoring the sequential nature of the generation process. 
This may harm the efficacy and efficiency of preference alignment.
In this paper, we take on a finer dense reward perspective and derive a tractable alignment objective that emphasizes the initial steps of the T2I reverse chain.
In particular, we introduce temporal discounting into DPO-style explicit-reward-free objectives, to break the temporal symmetry therein and suit the T2I generation hierarchy.
In experiments on single and multiple prompt generation, our method is competitive with strong relevant baselines, both quantitatively and qualitatively.
Further investigations are conducted to illustrate the insight of our approach.
Source code is available at \url{https://github.com/Shentao-YANG/Dense_Reward_T2I}\,.
\end{abstract}

%%%%%%%%%%%%%%%%%%%%%%%%%%%%%%%%%%%%%%%%%%%%%%%%%%%%%%%%%%%%%%%%%%%%%%%%%%%%%%%

\section{Introduction}\label{sec:intro}

% T2I have shown some success, 
% prior work has tried to aligning with preference by either optimize towards a trained reward function 
% or DPO style direct preference optimization by assuming a seq-level reward, without learning the reward model
Text-to-image diffusion model \citep[T2I,][]{ramesh2022hierarchical,saharia2022photorealistic}, trained by large-scale text-image pairs, has achieved remarkable success in image generation.
As an effort towards more helpful and less harmful generations, methods have been proposing to align T2I with preference, partially motivated by the progress of human/AI-feedback alignment for large language models (LLMs) \citep[][]{bai2022constitutional,gpt42023,llama22023}.
Prior works in this field typically optimize the T2I against an explicit reward function  trained in the first place \citep[][]{hpsv12023,imagereward2023,lee2023aligning}.
To remove the complexity in the modeling and computing of  an explicit reward function, recent work has   generalized direct preference optimization \citep[DPO,][]{dpo2023} from LLM into T2I's preference alignment \citep{wallace2023diffusion}, under a counterpart assumption of DPO that there is a latent reward function evaluating the entire diffusion reverse chain as a whole.

% while dpo style is good, seq-level reward has XXX issue (huge action space or the sparse reward issue), inducing a harder learning problem 

While DPO-style approaches have shown impressive potential,
from the reinforcement learning (RL) perspective,
these methods typically formulate the diffusion reverse chain as a contextual bandit, \ie, treating the entire generation trajectory as a single action; though the diffusion reverse chain is intrinsically a sequential generation process \citep{sohl2015deep,ho2020denoising}.
Since the reverse chain typically requires tens or even thousands of steps \citep{ddim2020,song2021scorebased}, such a bandit assumption, in particular, of a reward function on the whole chain/trajectory, can lead to a combinatorially large decision space over all timesteps. 
% By the sequential nature of diffusion generation,
This issue is twined with the well-known sparse reward (delayed feedback) issue in RL \citep{andrychowicz2017hindsight,liu2018competitive}, where an informative feedback is only provided after generating the entire trajectory. 
We hereafter use ``sparse reward'' to refer to this issue.
Without considering the sequential nature of the generation process, it is known from RL and LLM literature that this sparse reward issue, which often comes with high gradient variance and low sample efficiency \citep{sqltext2021}, can clearly hurt model training \citep{marbach2003approximate,takanobu2019guided}.

% in this paper, we continue dpo-style explicit-reward-free method
% we draw idea from recent development of the  latent preference generation model in NLP  and in roboticsto take on a finer grain dense reward perspective
% We break the symmetrics in the DPO-style alignment loss
% we introduce the (temporal) discount factor natural in RL into T2I alinemnt, and develop a tractable lower bound of the Bradley-Terry model of preferences 

% Our loss emphasize initial part of the reverse chain, which is more related to image skeleton and T2I alignment (cite)

In this paper, we contribute to the research on DPO-style explicit-reward-free alignment methods by taking on a finer-grain \textit{dense-reward} perspective, motivated by recent studies on the latent preference-generating reward function in NLP \citep[\eg,][]{yang2023preferencegrounded} and robotics \citep[\eg,][]{kim2023preference,hejna2023contrastive}. 
Instead of the hypothetical \textit{trajectory-level} reward function, we assume a latent reward function that can score \textit{each step} of the reverse chain, in hoping an easier learning problem from the RL viewpoint \citep[\eg,][]{laidlaw2023bridging}.
Inspired by studies on diffusion and T2I generation that the initial portion of the reverse chain sets up the image outline based on the given 
% \tc{see comment}
text conditional, and image's high-level attributes and aesthetic shapes \citep{ho2020denoising,wang2023diffusion},
% \tc{see comments}
we hypothesize that emphasizing those initial steps in T2I's preference alignment can help efficacy and efficiency, since those steps can be more directly related to the final de-noised image's being the preferred one.
Under this hypothesis, we break the temporal symmetry in the DPO-style alignment losses by introducing the temporal discounting factor, a key RL ingredient, into T2I's alignment.
Practically, we develop a lower bound of the resulting Bradley-Terry preference model  \citep{bradley1952rank}, which leads to a tractable loss to train a T2I for preference alignment in an  explicit-reward-free manner. 
% Not surprising, the temporal discounting in our objective emphasizes the initial steps of the reverse chain.
% Apart from the inspiration from \citet{wang2023diffusion}, from the RL viewpoint \citep[\eg,][]{laidlaw2023bridging}, our dense-reward perspective may lead to a easier learning problem % than sparse reward 
% and hence help alignment. 

% We test our method on XXX
% our method has good results
% We conduct further study to show the effectiveness of emphasizing the initial part, which to our best knowledge has not been widely studied in the T2I alignment literature

We test our method on the task of single prompt generation, which is easier for investigation; and the more challenging multiple prompt generation, where we align our T2I with the preference pertaining to one set of prompts and evaluate on another large-scale set of prompts.
On both tasks, our method exhibits competitive quantitative and qualitative performance against strong baselines.
We conduct further studies on the effectiveness of emphasizing the initial steps of the reverse chain in T2I's alignment, which to our best knowledge has not been well investigated in literature.

\section{Main Method}\label{sec:method}

\subsection{Notations and Assumptions} \label{sec:notation}

In this section, we state the notations and assumptions for deriving our method.
As discussed in \Secref{sec:intro}, our first and foremost assumption is a latent \textit{dense} reward.

\begin{restatable}{assumption}{AssumpDenseRew}\label{assump:dense_rew}
     There is a latent reward function $r(s_t, a_t)$ that can score each step $t$ of the T2I reverse chain.
\end{restatable}

% \myparagraph{Notations.}
We adopt the notations in prior works \citep[\eg,][]{fan2023dpok,ddpo2023} to formulate the diffusion reverse  process under the conditional generation setting as an Markov decision process (MDP), specified by $\gM = \br{\sS, \sA, \gP, r, \gamma, \rho}$. Specifically, let $\pi_\theta$ be the T2I with trainable parameters $\theta$, \ie, the policy network; $\cbr{\vx_t}_{t=T}^0$ be the  diffusion reverse chain of length $T$; and $\vc$ be the text conditional, \ie, the conditioning variable. We have, $\forall \, t$,
\begin{equation*}
\resizebox{0.49\textwidth}{!}{%
    $
    \begin{array}{ccc}
      s_t \triangleq (\vx_t, t, \vc) \,, &  \pi_\theta(a_t \given s_t) \triangleq p_\theta(\vx_{t-1} \given \vx_t, t, \vc) ,  \\
      a_t \triangleq \vx_{t-1} \,, &
      \rho(s_0) \triangleq (\gN(\vzero, \mI), \delta(T), \delta(\vc))  , 
       \\ \gP(s_{t+1}\given s_t, a_t) \triangleq \delta(\vx_{t-1}, t-1, \vc) \,, & r(s_t, a_t)\,,\; \gamma \in [0,1] ,
    \end{array}
$%
}
\end{equation*}
where $\delta(\cdot)$ is the delta measure
% , $p(\vc)$ is some distribution of the conditioning variable, 
and $\gP(\cdot \given s_t, a_t)$ is a deterministic transition.
We denote generically the reverse chain generated by a T2I under the text conditional $\vc$ as a trajectory $\tau$, \ie, $ \tau \triangleq (s_0, a_0, s_1, a_1, \ldots, s_T) \iff (\vx_T, \vx_{T-1}, \ldots, \vx_0) \given \vc \,$.
Note that for notation simplicity, $\vc$ is absorbed into the state part of $\tau$.

Similar to \citet{wallace2023diffusion}, we consider the setting where we are given two trajectories (reverse chains) with equal length $T$. 
For simplicity, assume that $\tau^1$ is the better one, \ie, $\tau^1 \succ \tau^2$. 
Let tuple $\ord\triangleq (1,2)$ and $\sigma(\cdot)$ denotes the sigmoid function, \ie, $\sigma(x) = 1/(1+\exp\br{-x})\,$.

As in standard RL settings \citep{rlintro2018,jointmatching2022}, the reward function $r$ in $\gM$ needs to be bounded. 
Without loss of generality, we assume $r(s,a) \in [0,1]$, and thus $r(s,a)$ may be interpreted as the probability of satisfying the preference when taking action $a$ at state $s$.

\begin{restatable}{assumption}{AssumpBddRew}\label{assump:bounded_rew}
    In $\gM$, $\forall\, (s,a) \in \sS \times \sA, \, 0\leq r(s,a) \leq 1$.
    % In the MDP $\gM$, assume that $\forall\, (s,a) \in \sS \times \sA, \, 0\leq r(s,a) \leq 1$.
\end{restatable}

The performance of a (generic) policy $\pi$ is typically evaluated by the expected cumulative discounted rewards \citep{rlintro2018}, which is defined as,  
\begin{equation}\label{eq:rl_objective_main}
\resizebox{0.43\textwidth}{!}{%
    $
    \eta(\pi) \triangleq \E\sbr{\sum_{t=0}^T \gamma^t\, r(s_t, a_t) \given s_0\sim \rho, a_t \sim \pi, s_{t+1} \sim \gP}\,.
$%
}
\end{equation}

\begin{restatable}{assumption}{AssumpEvalETau}
Based on \eqref{eq:rl_objective_main}, we assume that for a (generation) trajectory $\tau = (s_0, a_0, s_1, a_1, \ldots, s_T)$, its quality is evaluated by $e(\tau) \triangleq \sum_{t=0}^T \gamma^t\, r(s_t,a_t)$ .
\end{restatable}

\begin{remark}[Practical Rationality of $e(\tau)$] \label{remark:rationality_e_tau}
Since the final step of the reverse chain  depends on all previous steps, a score or (human) evaluation on the final de-noised image should indeed evaluate the whole corresponding de-noising chain, \ie, the entire generation trajectory $\tau$. 
This notion is particularly intuitive when the de-noising process is deterministic, \eg, DDIM \citep{ddim2020}. 
Though motivated from the RL viewpoint (\eqref{eq:rl_objective_main}), 
in evaluating a T2I's generation, typically humans first check its \textit{conceptual} shapes and matching with the text prompt; and if that's OK, then look at finer details in the image.
Thus, the initial steps of the reverse chain, which set up image \textit{outlines} (\Secref{sec:intro}), can play a more important role in an image's being preferred.
This insight is distilled into $e(\tau)$ by using $\gamma < 1$, which emphasizes the contribution from the initial steps.
% for $\gamma < 1$, $e(\tau)$ naturally emphasizes the initial portion of the reverse chain, which is our desiderata as discussed in \Secref{sec:intro}. 
% \st{TODO: @ALL help me check \& improve this discussion!!!}
\end{remark}

\subsection{Method Derivation} \label{sec:method:derive}

The derivation of our method is inspired by the RL literature \citep[\eg,][]{kakade2002approximately,trpo2015,awr2019} and DPO \citep{dpo2023}.
Due to the space limit, in this section we only present the key steps.
A step-by-step derivation is deferred to \cref{sec:detailed_method_derive_proof}.

Directly optimizing $\eta(\pi)$ in \eqref{eq:rl_objective_main} requires constantly sampling from the current learning policy, which can be less practical for T2I's preference alignment.
We are therefore motivated by the cited literature to consider an approximate off-policy objective.
Specifically, we employ the initial pre-trained T2I, denoted as $\pi_I$; and generate the off-policy trajectories by some ``old'' policy $\pi_O$, where $\pi_O$ may be chosen as $\pi_I$ or some saved policy checkpoint not far from $\pi_I$.
We denote $d_{\pi_O}(s)$ as the stationary distribution of $\pi_O$ (detailed in \cref{sec:exp_e_tau}).
To avoid generating unnatural images,  we impose a KL regularization  towards $\pi_I$ on the learning policy $\pi$. Together, we arrive at the following regularized policy optimization problem
\begin{equation}\label{eq:optim_reg_main}
\resizebox{0.43\textwidth}{!}{%
    $
    \begin{aligned}
        \arg\max_\pi \quad & \E_{s\sim d_{\pi_O}(s)}\E_{a\sim \pi(a\given s)}\sbr{r(s,a)} \\
         & \qquad- C \cdot \E_{s\sim d_{\pi_O}(s)}\sbr{\KL\br{\pi(\cdot \given s) \,\|\, \pi_I(\cdot \given s)}}  \\
        \mathrm{s.t.} \quad
        & \int_\sA \pi(a\given s) \intd a = 1 ,\quad \forall \, s\in \sS \,,
    \end{aligned}
$%
}    
\end{equation}
where $C$ is a tuning regularization/KL coefficient.

By solving the first-order condition of the Lagrange form of \eqref{eq:optim_reg_main}, we can get the optimal (regularized) policy $\pi^*$ as 
\begin{equation}  \label{eq:optimal_policy_main}
    \begin{aligned}
        \pi^*(a\given s) &= \exp\br{ r(s,a) /{C}} \pi_I(a\given s)\left/{Z(s)} \right.\,,
    \end{aligned}
\end{equation}
where $Z(s)$ denotes the partition function, taking the form
\begin{equation*}
    Z(s) = \int_\sA \exp\br{ r(s,a) /{C}}  \pi_I(a\given s)  \intd a \,.
\end{equation*}
We also have the relation between $\pi^*$ and $r$, as
\begin{equation}\label{eq:formula_r_main}
    \begin{split}
    r(s,a) 
        &= C\log\sbr{{\pi^*(a\given s)}/{\pi_I(a\given s)}} + C\log Z(s) \,.
    \end{split}
\end{equation}
For a \textit{given} trajectory $\tau = (s_0, a_0, s_1, a_1, \ldots, s_T)$, after plugging in \eqref{eq:formula_r_main}, $e(\tau)$ can be expressed by $\pi^*$ as 
\begin{equation} \label{eq:e_tau_pi_star_main}\textstyle
% \resizebox{0.43\textwidth}{!}{%
% $
        e(\tau) 
        % &= C\sum_{t=0}^T \sbr{ \gamma^t \log\frac{\pi^*(a_t\given s_t)}{\pi_I(a_t\given s_t)}} + C \sum_{t=0}^T\gamma^t \log Z(s_t) \\
        = C\sum_{t=0}^T \sbr{ \gamma^t \log\frac{\pi^*(a_t\given s_t)}{\pi_I(a_t\given s_t)} } + C \log Z(\tau) \,.
% $%
% }  
\end{equation}
where we denote $\log Z(\tau) \triangleq \sum_{t=0}^T\gamma^t \log Z(s_t)$ for notation simplicity 
since the discounted sum is over all $s_t \in \tau$.

Under the Bradley-Terry (BT) model, by plugging in \eqref{eq:e_tau_pi_star_main}, the probability of $\ord$ under $\{e(\tau^k)\}_{k=1}^2$ and hence $\pi^*$ is
\begin{equation}\label{eq:pl_original_main}
\resizebox{0.43\textwidth}{!}{%
    $
    \begin{aligned}
        \Pr &\br{\ord \given  \pi^*, \{e\br{\tau^k}\}_{k=1}^2} = \sigma\br{e\br{\tau^1} - e\br{\tau^2}} \\
        &\quad=  \frac{\exp\br{C\sum_{t=0}^T\gamma^t \log \frac{\pi^*\br{a^1_t \given s^1_t}}{\pi_I\br{a^1_t \given s^1_t}}}Z\br{\tau^{\textcolor{burgundy}{1}}}^{C}}{\sum_{i=1}^2 \exp\br{C\sum_{t=0}^T\gamma^t \log \frac{\pi^*\br{a^i_t \given s^i_t}}{\pi_I\br{a^i_t \given s^i_t}}}Z\br{\tau^{\textcolor{burgundy}{i}}}^{C}}
        \,.
    \end{aligned}
$%
}    
\end{equation}
\eqref{eq:pl_original_main}, however, contains the intractable partition functions $Z(\tau^1)$ and $Z(\tau^2)$.
We will provide a tractable lower bound of \eqref{eq:pl_original_main} by arguing that $Z(\tau^1)\geq Z(\tau^2)$.
Our argument is based on the reward-shaping technique \citep{rewardshaping1999}.

\begin{restatable}[Reward Shaping]{definition}{DefRewShape}
    A shaping-reward function $\Phi$ is a real-valued function on the state space, $\Phi: \sS \rightarrow \R$. It induces a new MDP $\gM' = \br{\sS, \sA, \gP, r', \gamma, \rho}$ where $r'(s,a) \triangleq r(s,a) + \Phi(s)$.
\end{restatable}

\begin{restatable}[Invariance of Optimal Policy under Reward Shaping]{lemma}{LemmaInvOptPol}\label{lemma:pi_star_inv}
    The optimal (regularized) policy \eqref{eq:optimal_policy_main} under the reward-shaped MDP $\gM'$ is the same as that in the original MDP $\gM$\,.
\end{restatable}
The proof is deferred to \eqref{eq:proof_inv_opt_pol} in \cref{sec:mle_obj_derive}.
Note that $\gM'$ and $\gM$ share the same state and action space.
Thus, it makes sense to consider the invariance of the optimal policy, where invariance means at each state taking the same action with the same probability.

% \begin{restatable}[Equivalence class of $r(s,a)$]{definition}{DefEquivClassR}\label{def:equiv_class_r}
\begin{restatable}{definition}{DefEquivClassR}\label{def:equiv_class_r}
The equivalence class $[r]$ of the reward function $r$ is the set of all reward functions that can be obtained from $r$ by reward shaping, \ie, $\forall\, r' \in [r], \exists\, \Phi: \sS \rightarrow \R, \,s.t.\; r'(s,a) - r(s,a) = \Phi(s), \forall\, s\in \sS\,, a\in \sA$\,.
\end{restatable}

\begin{restatable}{remark}{RemarkInvEquivClass}
    By Lemma~\ref{lemma:pi_star_inv}, all reward functions in $[r]$ share the same optimal (regularized) policy  as $r$, \ie, \eqref{eq:optimal_policy_main}.
\end{restatable}

\vspace{-.2em}
We are now able to justify our argument: $Z(\tau^1)\geq Z(\tau^2)$. 
% We consider a slightly general setting of $\geq 2$ trajectories.

\begin{restatable}{theorem}{ThmOrderPartitionAppendix}\label{theorem:order_partition}
Under Assumption~\ref{assump:bounded_rew}, and a sufficiently large regularization coefficient $C$,
    for any finite number $K \geq 2$ of trajectories $\{\tau^k\}_{k=1}^K$ where $\tau^1 \succ \tau^2 \succ \cdots \succ \tau^K$, $\exists\, r' \in [r], \,s.t.,\, Z(\tau^1)\geq Z(\tau^2)\geq \cdots \geq Z(\tau^K)$ under $r'$.
\end{restatable} % use as: \ThmOrderPartitionAppendix*
We defer the proof of Theorem~\ref{theorem:order_partition} to \cref{sec:proof_order_partitions}.

\begin{restatable}{remark}{RemarkValueC}\label{remark:value_C}
% The requirement on the value of $C$ depends on $\max r(s,a)$ on $\sS \times \sA$.
For the value of $C$, as we will see in the proof, we technically require that $\forall (s,a) \in \sS \times \sA, r(s,a) / C \leq \mathrm{const} \approx 1.79$.
Under Assumption~\ref{assump:bounded_rew}, $C\geq 0.56$ will suffice.
We note that this technical requirement helps reducing the search space of the hyperparameter $C$ in practice.
\end{restatable}

\begin{restatable}{remark}{RemarkCallRprimtAsR}
By Lemma~\ref{lemma:pi_star_inv}, $r'$ in Theorem~\ref{theorem:order_partition} and the original $r$ lead to the same optimal policy $\pi^*$, which is our ultimate target. 
Due to this invariance, for notation simplicity, we hereafter refer to $r'$ as $r$, though we \textit{may} actually work in the ``equivalent'' MDP $\gM\textcolor{burgundy}{'} = \br{\sS, \sA, \gP, r\textcolor{burgundy}{'}, \gamma, \rho}$. 
% with the shaped reward $r'$. 
\end{restatable}

With Theorem~\ref{theorem:order_partition}, we can provide a simpler lower bound to $\Pr(\ord \given \pi^*, \{e\br{\tau^k}\}_{k=1}^2)$ in \eqref{eq:pl_original_main},
\begin{equation}\label{eq:pl_lower_bound_main}
% \vspace{-1em}
    \resizebox{0.43\textwidth}{!}{%
    $
        \Pr\br{\ord \given \pi^*, \cbr{e\br{\tau^k}}_{k=1}^2}
        \geq  \frac{\exp\br{C\sum_{t=0}^T\gamma^t \log \frac{\pi^*\br{a^1_t \given s^1_t}}{\pi_I\br{a^1_t \given s^1_t}}}}{\sum_{i=1}^2 \exp\br{C\sum_{t=0}^T\gamma^t \log \frac{\pi^*\br{a^i_t \given s^i_t}}{\pi_I\br{a^i_t \given s^i_t}}}} \,.
$%
}  
% \vspace{-.5em}
\end{equation}
Recall that $e(\tau)$ evaluates a trajectory $\tau$'s quality, and thus a better trajectory  comes with a higher $e(\tau)$.
Hence $\Pr(\ord \given \pi^*, \{e(\tau^k)\}_{k=1}^2) = \max \Pr(\cdot \given \pi^*, \{e(\tau^k)\}_{k=1}^2)$, \ie, under $\gM$ with $\pi_I$ and conditioning on $\pi^*$, $\ord$ should be the most probable ordering under the BT model shown in \eqref{eq:pl_original_main}.
Thus, in order to approximate $\pi^*$, we train $\pi_\theta$ by maximizing the lower bound \eqref{eq:pl_lower_bound_main} of the \textit{corresponding} BT likelihood of $\ord$, which leads to the negative-log-likelihood \textit{loss} function for training $\pi_\theta$ as % negative-log-likelihood objective for an minimization objective for training $\pi_\theta$ as
\begin{equation}\label{eq:pk_nll_main}
\resizebox{.91\linewidth}{!}{%
$
\begin{aligned}
    &\gL_\gamma(\theta \given \ord, \{e(\tau^k)\}_{k=1}^2) \\
    =& -\log\sigma\br{ C\E_{t \sim \mathrm{Cat}\br{\cbr{\gamma^t}}} \sbr{\log\frac{\pi_\theta\br{a^1_t \given s^1_t}}{\pi_I\br{a^1_t \given s^1_t}} - \log\frac{\pi_\theta\br{a^2_t \given s^2_t}}{\pi_I\br{a^2_t \given s^2_t}}}} \,,
\end{aligned}    
$%
}    
\end{equation}
where $\mathrm{Cat}(\{\gamma^t\})$ denotes the categorical distribution on $\{0,\ldots,T\}$ with the probability vector $\{\gamma^t/ \sum_{t'} \gamma^{t'}\}_{t=0}^T$ and $C$ is overloaded to absorb the normalization constant.

\myparagraph{Interpretation.} \label{sec:interpretation}
To see what $\gL_\gamma(\theta \given \ord, \{e(\tau^k)\}_{k=1}^2)$, our loss  in \eqref{eq:pk_nll_main}, is doing, let's calculate its gradient.

Since \eqref{eq:pk_nll_main} is an objective for minimization problem, the gradient update direction is
$-\nabla_\theta \gL_\gamma$.
% The gradient can be derived by chain rule as follows.
For notation simplicity, we denote $\widetilde e(\tau^k) \triangleq  C\sum_{t=0}^T\gamma^t \log \frac{\pi_\theta(a^k_t \given s^k_t)}{\pi_I(a^k_t \given s^k_t)}$.
We have
\begin{equation}\label{eq:derivative_l_gamma}
\vspace{-.01em}
    \resizebox{.91\linewidth}{!}{%
$
\begin{aligned}
        \frac{\partial\, (-\gL_\gamma (\theta \given \ord, \{e(\tau^k)\}_{k=1}^2))}{\partial \, \theta} 
        &= \underbrace{\frac{\exp\br{\widetilde e\br{\tau^2} - \widetilde e\br{\tau^1}}}{1+\exp\br{\widetilde e\br{\tau^2} - \widetilde e\br{\tau^1}}}}_{\textcolor{blue}{\textbf{(*)}}} \times\, C \\
        \times \sum_{t=0}^T\gamma^t     \big(
        \pi_I(a^1_t \given s^1_t)  \nabla_\theta \log & \pi_\theta(a^1_t \given s^1_t)   - \pi_I(a^2_t \given s^2_t) \nabla_\theta \log \pi_\theta(a^2_t \given s^2_t)
        \big) \,.
    \end{aligned}    
$%
}
% \vspace{-.5em}
\end{equation}
Detailed derivation is in \cref{sec:derive_loss_gradient}.
The term \textcolor{blue}{\textbf{(*)}} is high when $\widetilde e(\tau^2) > \widetilde e(\tau^1)$, \ie, in the unwanted case where the discounted (relative) likelihood of the inferior trajectory $\tau^2$ is higher.
In that case, we increase the likelihood of $(s_t,a_t) \in \tau^1$ and decrease $(s_t,a_t) \in \tau^2$.
Note that this mechanism is weighted by $\gamma^t$, with which we emphasize the earlier steps in the reverse chain. As discussed in \Secref{sec:intro}, this could be more effective in getting desirable final images.

Additionally, for $(s_t, a_t)$, if $\pi_I(a_t \given s_t)$ is small, our changes (increase or decrease likelihood) can be small too.
This may be interpreted as those $(s_t, a_t)$ are 
% either at the edge of the sampling distribution leading to unsafe policy improvement, if $\pi_O\br{a_t \given s_t}$ is small; or
at the edge of the initial distribution threatening the generation of realistic images.
Meanwhile, if $\pi_I(a_t \given s_t)$ is high, our changes can be also high,
since we now have more ``room'' for improving and our gradient utilizes this to achieve safe and effective training.

\subsection{Practical Implementation} \label{sec:method_implement} % \myparagraph{Implementation.}

In practice, we assume that $\pi_\theta$ is optimized over a given prompt distribution $p(\vc)$, where $p(\vc)=\delta(\vc)$ if we fine-tune $\pi_\theta$ on a single prompt $\vc$, and $p(\vc) = \mathrm{Unif}(\gD(\vc))$ for a dataset $\gD(\vc)$ of prompts if tuning $\pi_\theta$ on multiple prompts.

We implement our algorithm as an online off-policy learning routine.
Similar to prior RL works \citep[\eg,][]{dqn2013,ddpg2016}, we iterate between \textbf{(1)} using the current $\pi_\theta$ to sample $N_{\mathrm{traj}}$ trajectories for each of the $N_\mathrm{pr}$ prompts sampled from $p(\vc)$; and \textbf{(2)} training $\pi_\theta$ via \eqref{eq:pk_nll_main} on mini-batches of trajectories sampled from all stored.
To mimic the classical RLHF settings \citep[\eg,][]{ziegler2019fine,instructgpt2022}, we set $N_{\mathrm{traj}} = 5\geq 2$ for resource efficiency.
In calculating the loss \eqref{eq:pk_nll_main}, we sample $N_\mathrm{step}$ timesteps from $\mathrm{Cat}(\{\gamma^t\})$ to estimate the expectation inside $\sigma(\cdot)$.
Algo.~\ref{algo:main_method} outlines the key steps of our method.

\setlength{\textfloatsep}{0.1cm}
\setlength{\floatsep}{0.1cm}
\begin{algorithm}[tb]
% \captionsetup{font=small}
\caption{Outline of Our Off-policy Learning Routine.}
\label{algo:main_method}
\begin{algorithmic}
\STATE \textbf{Input:} Prompt distribution $p(\vc)$, T2I $\pi_\theta$,  training steps $M_\mathrm{tr}$, trajectory collect period $M_\mathrm{col}$, $\#$ prompts to collect trajectories $N_\mathrm{pr}$, $\#$ trajectories for each prompt $N_{\mathrm{traj}}$\,.
% \STATE
\STATE \textbf{Initialization:} Sample $N_\mathrm{pr}$ prompts $\cbr{\vc}\sim p(\vc)$, get $N_{\mathrm{traj}}$ trajectories for each $\vc$\,.
\FOR{$\mathrm{iter} \in \{1, \ldots,M_\mathrm{tr}\}$}
\STATE Sample a mini-batch $\gB \triangleq \{(\tau_i^1, \tau_i^2)_{\vc_i}\}_i$ from storage.
\STATE Optimize $\pi_\theta$ via \eqref{eq:pk_nll_main} using $\gB$.
\IF{$\mathrm{iter}~\%~M_\mathrm{col} == 0$}
\STATE Re-sample $N_\mathrm{pr}$ prompts $\cbr{\vc} \sim p(\vc)$, get $N_{\mathrm{traj}}$ trajectories for each  $\vc$, and update the storage.
\ENDIF
\ENDFOR
\end{algorithmic}
\end{algorithm}
\setlength{\textfloatsep}{0.2cm}
\setlength{\floatsep}{0.1cm}

\subsection{Connection with the DPO Objective.} \label{sec:connect_with_dpo}
The original DPO loss (Eq. (7) in \citet{dpo2023}) can be obtained as a variant of \eqref{eq:pk_nll_main} when setting $\gamma=1$, after factorizing out the probability at each step $t$. % auto-regressively factorizing  out
Using $\gamma=1$ in our formulation is equivalent to the DPO-style trajectory-level bandit setting since $\gamma=1$ makes the contribution of each timestep $t$ to $e(\tau)$ \textit{symmetric}, \ie, each timestep is equally important, and therefore each timestep $t$ is  \textit{symmetric} in the loss  as well. %, making the final loss symmetric over $t$.
Likewise, in DPO-style trajectory-level bandit setting, since the trajectory as a \textit{whole} receives a single reward, this reward/evaluation does not distinguish each step $t$ within the trajectory either, making each timestep $t$ \textit{symmetric} again in the training loss, same as our variant with $\gamma=1$.
% Since $\gamma=1$ makes the contribution of each timestep $t$ \textit{symmetric} and indistinguishable, $e(\tau)$ is effectively an ``overall reward'' of the entire generation trajectory $\tau$.
Due to this connection in the loss, we refer to this variant as ``trajectory-level reward,'' indistinguishable to whether it actually comes from a trajectory-level bandit setting or our formulation but with $\gamma=1$.
As a reminder, in our formulation, if we set $\gamma < 1$, then the contribution of each timestep $t$ to $e(\tau)$ will not be symmetric, since earlier steps will be emphasized. 
This leads to the desirable \textit{asymmetry} of \mbox{timestep $t$ in the loss, as shown in \eqref{eq:pk_nll_main}.}

\section{Related Work} \label{sec:related_work}

% <----------------------------- > %
\myparagraph{T2I's Alignment with Preference.}\label{sec:t2ialignment}%
There have been growing interests in aligning T2I's, or more broadly diffusion models', generations to (human) preferences.
Efforts have been putting on tuning the models on curated data % images/datasets
\citep{podell2023sdxl,dai2023emu} or re-captioning existing image datasets \citep{dalle32023,segalis2023picture}, to bias T2I generation towards better text fidelity and aesthetics.
These data enhancement efforts may complement our method. % are considered complementary

% Methods requires explicit reward functions 
To more directly optimize the feedback, methods have been proposing to fine-tune T2I with respect to (\wrt) reward models pre-trained on large-scale human preference datasets \citep{imagereward2023,hpsv22023,pickapic2023}.
% with Supervised Learning: 
\citet{Lee2023AligningTM} and \citet{hpsv12023} adapt the classical supervised training by fine-tuning T2I via reward-weighted likelihood or discarding low-reward images, with online versions extended by \citet{reft2023}.      % which is extended into  online version by \citep{reft2023}.
% with Reinforcement Learning
By formulating the denoising process as an MDP, policy gradient methods are adopted to fine-tune T2I for specific rewards \citep{sftpg2023,fan2023dpok,ddpo2023} or polishing the input prompts \citep{Hao2022OptimizingPF}. 
%  downstream rewards, polishing the input prompts
% backpropagate reward function gradient through the reverse chain to directly aligns generated images with the downstream reward
% DOODL (uses backpropagation through sampling to improve image generation with respect to differentiable objectives),
% With a further assumption on reward function's differentiablity
Further assuming a \textit{differentiable} reward function, a more direct alignment/feedback-optimization can be achieved by backpropagating the reward function's gradient through the reverse chain \citep[\eg,][]{draft2023,alignprop2023,doodl2023}.     % simply backpropagating
Although optimizing \wrt\mbox{} explicit rewards have shown efficacy and efficiency, it requires a stronger assumption than our method on having an explicit scalar reward function, while assuming analytic gradients of the reward function is even stronger.
By contrast, our method only requires  binary comparison between generated images/trajectories, which is among the simplest in T2I's preference alignment.  

% ----------------------------------------------------
% ImageReward
\begin{figure*}[tb]
     % \vspace{+2mm}
     \centering
     \begin{subfigure}[b]{0.22\textwidth}
         \centering
         \includegraphics[width=\textwidth]{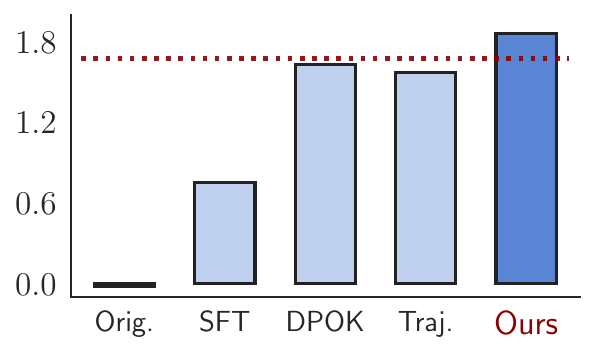}
         \captionsetup{font=footnotesize}
         \vspace{-6mm}
         \caption{\footnotesize{Color}} % Sequence Num.
         \label{fig:single_imagerew_Color}
     \end{subfigure}
     \hfill
     % \hspace{1em}
     \begin{subfigure}[b]{0.22\textwidth}
         \centering
         \includegraphics[width=\textwidth]{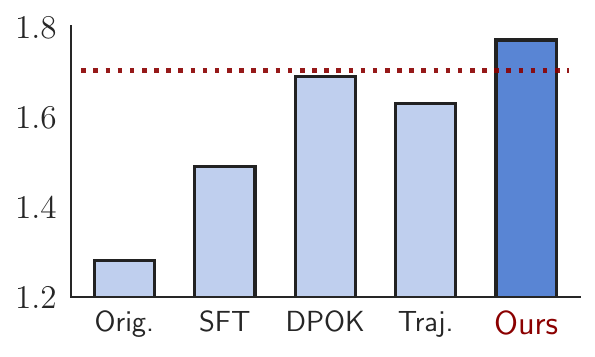}
         \captionsetup{font=footnotesize}
         \vspace{-6mm}
         \caption{\footnotesize{Count}} % Sequence Num. 
         \label{fig:single_imagerew_Count}
     \end{subfigure}
     \hfill
     % \hspace{1em}
    %  \rulesep
    \begin{subfigure}[b]{0.22\textwidth}
         \centering
         \includegraphics[width=\textwidth]{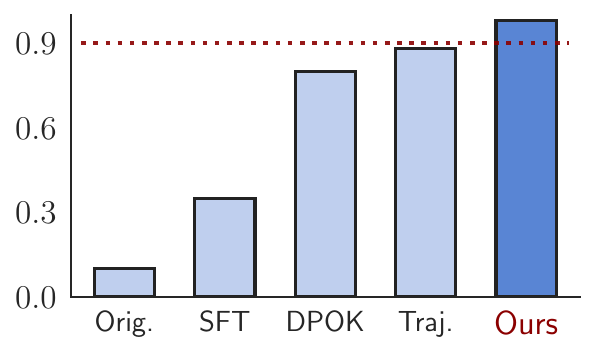}
         \captionsetup{font=footnotesize}
         \vspace{-6mm}
         \caption{\footnotesize{Composition}}    % Sequence Num. 
         \label{fig:single_imagerew_Composition}
     \end{subfigure}
     \hfill
     % \hspace{1em}
     \begin{subfigure}[b]{0.22\textwidth}
         \centering
         \includegraphics[width=\textwidth]{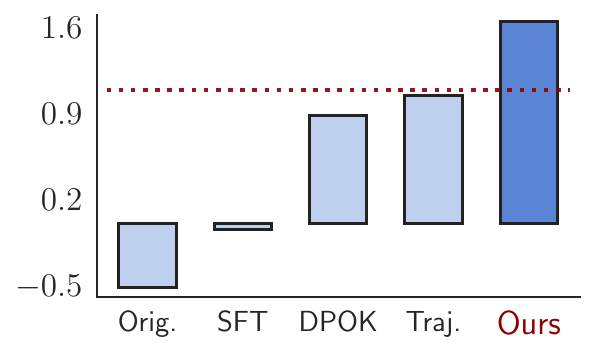}
         \captionsetup{font=footnotesize}
         \vspace{-6mm}
         \caption{\footnotesize{Location}} % Sequence Num. 
         \label{fig:single_imagerew_Location}
     \end{subfigure}
     \vspace{-.8em}
     \captionsetup{font=small}
        \caption{ 
        \small
        ImageReward scores for the seen prompts in the single prompt experiments.
        ``Orig.'' denotes the original SD1.5.
        ``SFT'' is the supervised fine-tuned model.
        ``Traj.'' denotes the classical DPO-style objective discussed in \Secref{sec:connect_with_dpo}, \ie, assuming trajectory-level reward. 
        All our produced results are the average over $100$ samples.
        Horizontal line indicates the best baseline result.
        }
        \label{fig:single_imagerew}
        \vspace{-1em}
\end{figure*}
% ----------------------------------------------------
% ----------------------------------------------------
% Aestheric
\begin{figure*}[tb]
     % \vspace{+2mm}
     \centering
     \begin{subfigure}[b]{0.22\textwidth}
         \centering
         \includegraphics[width=\textwidth]{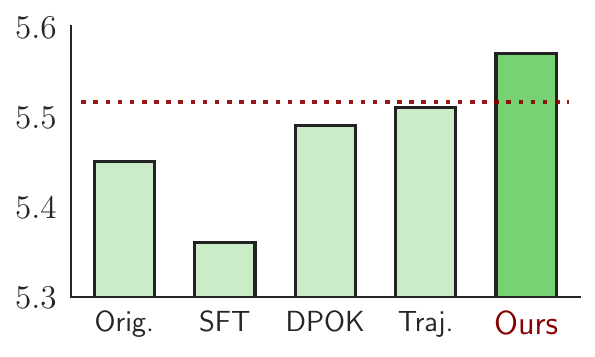}
         \captionsetup{font=footnotesize}
         \vspace{-6mm}
         \caption{\footnotesize{Color}} % Sequence Num.
         \label{fig:single_aes_Color}
     \end{subfigure}
     \hfill
     % \hspace{1em}
     \begin{subfigure}[b]{0.22\textwidth}
         \centering
         \includegraphics[width=\textwidth]{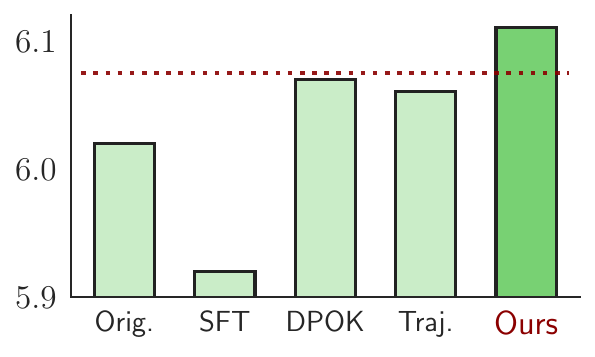}
         \captionsetup{font=footnotesize}
         \vspace{-6mm}
         \caption{\footnotesize{Count}} % Sequence Num. 
         \label{fig:single_aes_Count}
     \end{subfigure}
     \hfill
     % \hspace{1em}
    %  \rulesep
    \begin{subfigure}[b]{0.22\textwidth}
         \centering
         \includegraphics[width=\textwidth]{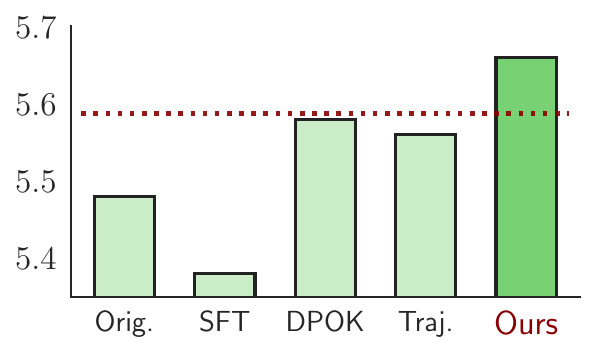}
         \captionsetup{font=footnotesize}
         \vspace{-6mm}
         \caption{\footnotesize{Composition}}    % Sequence Num. 
         \label{fig:single_aes_Composition}
     \end{subfigure}
     \hfill
     % \hspace{1em}
     \begin{subfigure}[b]{0.22\textwidth}
         \centering
         \includegraphics[width=\textwidth]{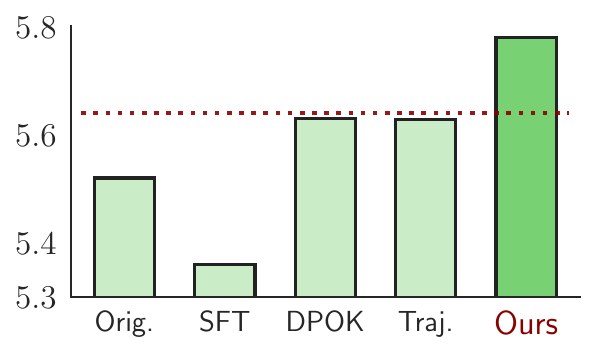}
         \captionsetup{font=footnotesize}
         \vspace{-6mm}
         \caption{\footnotesize{Location}} % Sequence Num. 
         \label{fig:single_aes_Location}
     \end{subfigure}
     \vspace{-.8em}
     \captionsetup{font=small}
        \caption{ 
        \small
        Aesthetic scores for the seen prompts in the single prompt experiments.
        Number reporting and abbreviations follow Fig.\ref{fig:single_imagerew}. 
        }
        \label{fig:single_aes}
        \vspace{-1.1em}
\end{figure*}
% ----------------------------------------------------

% ----------------------------------------------------
\begin{figure*}[tb]
     % \vspace{-1.0mm}
     \centering
\includegraphics[width=0.9\textwidth]{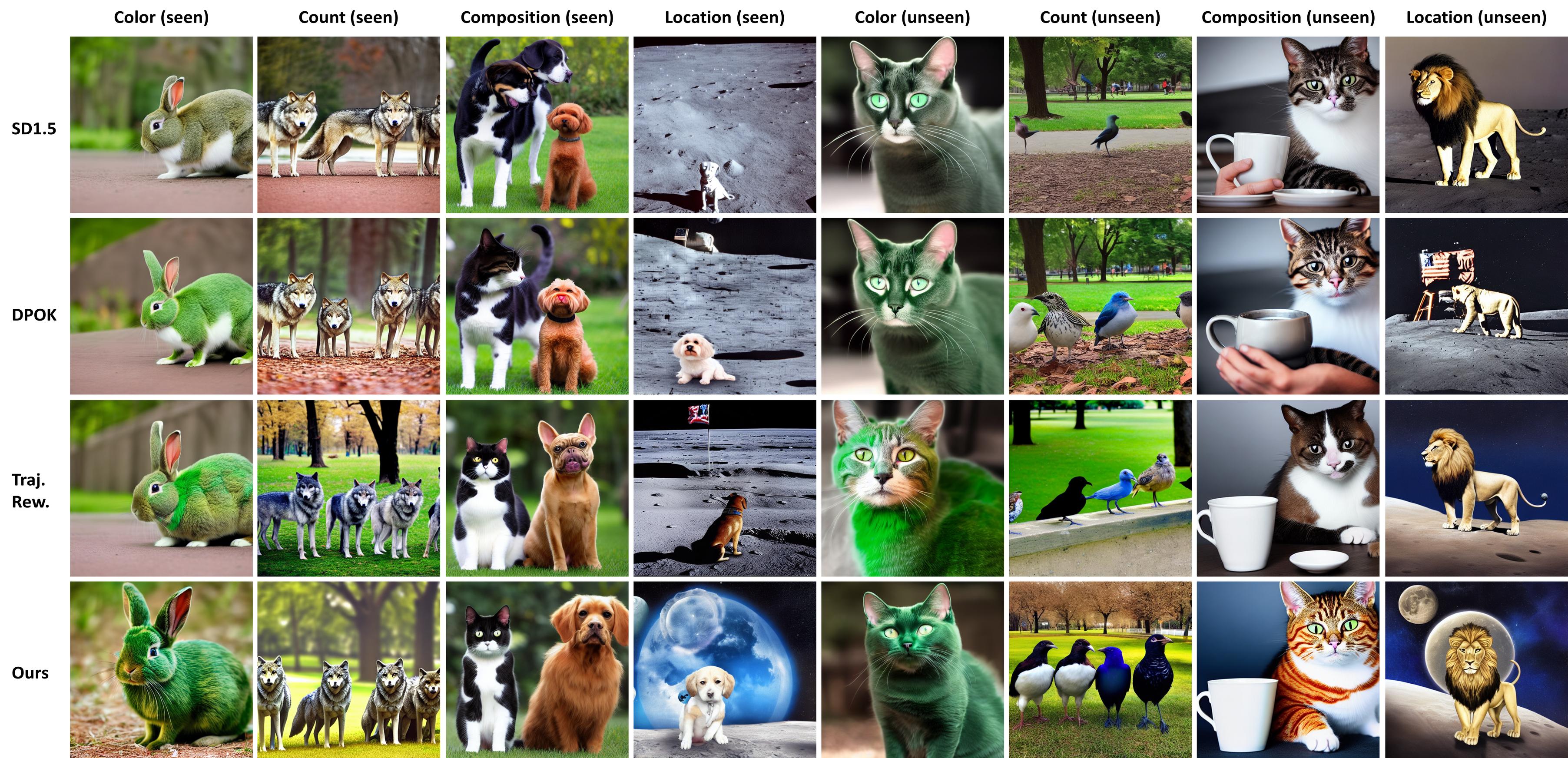}
     \vspace{-2mm}
     \captionsetup{font=small}
        \caption{ 
        \small
        Generated images in the single prompt experiment for both seen and unseen prompts (Table~\ref{table:four_single_prompts}).
        Each comparison is generated from the same random seed.
        ``Traj. Rew.'' denotes the classical DPO-style objective of assuming trajectory-level reward (\Secref{sec:connect_with_dpo}).
        }
        \label{fig:single_prompt_generated_images}
        \vspace{-1.1em}
\end{figure*}
% ----------------------------------------------------

% Diffusion-DPO: most similar to our work;
Most close to our work, Diffusion-DPO \citep{wallace2023diffusion} also considers an explicit-reward-free T2I alignment method.
% a different assumption and thus use forward to approximate true backward
It is nevertheless developed under a different setting where the generation latents are discarded,
and thus it needs to approximate the reverse process with the forward.
However, given the relatively-small scale of the preference alignment stage, storing the reverse chains can be both feasible and straightforward.  
% our backward is exact
We thereby eschew such an approximation and use the exact generation latents.
% its derivation assumes trajectory-level reward
More importantly, as in DPO \citep{dpo2023}, Diffusion-DPO is derived by assuming reward on the \textit{whole} chain/trajectory, obtainable as a variant of our method (\Secref{sec:connect_with_dpo}) and distinct from our dense reward perspective.
% In experiment,  compare with traj-level reward and show that our dense reward is better
In experiments, we validate the efficacy of our perspective by comparing with this approach of ``trajectory-level reward.''
% <----------------------------- > %

\cref{sec:more_related_work} reviews literature on \textbf{(1)} dense \textit{v.s.} sparse training guidance for sequential generative models, \textbf{(2)} characterizing the (latent) preference generation distribution, and \textbf{(3)} learning-from-preference in related fields.

% \st{
% @Tianqi
% Locations of the images
% \begin{itemize}
%     \item Raw SD1.5: Image generated at step "0"
%     \item DPOK: "/home/styang/rldiffusion/dpok/online\_model/img\_reward\_0/pre\_train/single\_prompt"
%     \item Seq-level: "/data/shentao/rldiffusion/results/seq\_level" (For "Location", use images generated at Step 8K)
%     \item Dense Rew (Ours): "/data/shentao/rldiffusion/results/dense\_rew"
% \end{itemize}
% We need the comparison of both seen and unseen prompts.
% }

\section{Experiments}\label{sec:exp}

To straightforwardly evaluate our method's ability to satisfy preference, motivated by 
recent papers in directly tuning T2I \wrt\mbox{} pre-trained rewards (\Secref{sec:related_work}) and relevant papers in NLP \citep[\eg,][]{caspi2021,fantasticrewards2022,yang2023preferencegrounded}, in our experiments, we use the following logics:
We obtain preference among multiple trajectories by some 
open-source scorer trained on data of human preference over T2I's generations; and test our method's ability in increasing the score, as an indication of the model's improved alignment with (human) preference.
The scorer factors in text fidelity.
For preference simulation, given a prompt and $N_{\mathrm{traj}}$ corresponding images, the higher the score, the more preferable the image is.

For computational efficiency, our policy $\pi_\theta$ is implemented as LoRA \citep{hu2021lora} added on the U-net \citep{unet2015} module of a frozen pre-trained Stable Diffusion v1.5 \citep[SD1.5,][]{stablediffusion2022}, and we only train the LoRA parameters.
% , which is technically the trainable policy parameter $\theta$. 
With SD1.5, the generated images are of resolution $512 \times 512$.
For all our main results, we set the discount factor $\gamma$ to be $\gamma=0.9$. 
We perform ablation study on the $\gamma$ value in \Secref{sec:exp_abla} \textbf{(b)}.
As in prior works \citep[\eg,][]{fan2023dpok,ddpo2023}, in both sampling trajectories and generating evaluation images, we use DDPM sampler with $50$ inference steps and classifier-free guidance \citep{ho2022classifier}. 
We use the default guidance scale of $7.5$.
Source code is publicly \href{https://github.com/Shentao-YANG/Dense_Reward_T2I}{released}.

\subsection{Single Prompt}\label{sec:exp_single}

\begin{table}[t]
\captionsetup{font=small}
\caption{
\small
Seen and unseen prompts in \Secref{sec:exp_single} for each domain.
} 
\label{table:four_single_prompts}
\centering 
\vspace{-1em}
% \scriptsize
\resizebox{.48\textwidth}{!}{
\begin{tabular}{@{}cll@{}}
\toprule
Domain     & \multicolumn{1}{c}{Seen}                     & \multicolumn{1}{c}{Unseen}                  \\ \midrule
Color       & A green colored rabbit.  & A green colored cat.    \\
Count       & Four wolves in the park. & Four birds in the park. \\
Composition & A cat and a dog.         & A cat and a cup.        \\
Location    & A dog on the moon.       & A lion on the moon.     \\ \bottomrule
\end{tabular}
}
% \vspace{8mm}
\end{table}

\myparagraph{Settings.}
To facilitate investigation, we first test our method on the single-text-prompt setting in DPOK \citep{fan2023dpok},  \ie, using one prompt during  LoRA fine-tuning.
As in DPOK, the goal is to test our method on training the policy T2I to achieve % the ability of 
generating objects with specified colors, counts, or locations, or generating composition of two objects.
We borrow the seen (training) and unseen prompts from DPOK, which are tabulated in \cref{table:four_single_prompts}.
In all single prompt experiments, we use the explicit reward model in DPOK, ImageReward \citep{imagereward2023}, to generate preference.
We report both ImageReward and (Laion) Aesthetic score \citep{schuhmann2022laion}, averaged over $100$ generated images.

\myparagraph{Implementation.}
For a fair comparison,  we collect the same total amount of $20000$ images/trajectories as DPOK.
Rather than its fully-online image-collection strategy, we are motivated by recent RLHF works \citep[\eg,][]{ziegler2019fine,stiennon2020learning,bai2022training} to more practically divide our trajectory collection into four stages, where each stage collects $5000$ trajectories and discards the previously collected ones.
As in DPOK, we use LoRA with rank $4$ and train the model for a total of $M_\mathrm{tr}=10000$  steps, and hence $M_\mathrm{col}=2500$ steps, $N_\mathrm{pr}=1000$.
We set the KL coefficient $C=10$ and $N_\mathrm{step} = 3$.
\Secref{sec:exp_abla} \textbf{(c)} ablates the value of $C$.
More details on hyperparameters are in \cref{sec:exp_details_single_prompt}. 

\begin{table*}[t]
\captionsetup{font=small}
\caption{
\small
HPSv2 and Aesthetic score for the multiple prompt experiment.
Shown here are results for selected relevant and/or strong baselines, with
full set of results in Table~\ref{table:multiple_promtps_all} of \cref{sec:add_results}.
The first four result columns are the four styles in HPSv2 test set and ``Average'' is the overall average.
Best result in each metric is bold.
Note that HPSv2 paper and Github repository do not report Aesthetics score. 
} 
\label{table:multiple_promtps_small}
\centering 
\vspace{-.6em}
% \setlength\tabcolsep{12pt}
% \renewcommand{\arraystretch}{1.1}
% \resizebox{.85\textwidth}{!}
% {
% \footnotesize
\begin{tabular}{@{}lcccccc@{}}
\toprule
Model                   & Animation & Concept-art & Painting & Photo    & Average & Aesthetic     \\ \midrule
DALL·E 2                & 27.34           & 26.54             & 26.68          & 27.24          & 26.95          & -             \\
Stable Diffusion v1.5   & 27.43           & 26.71             & 26.73          & 27.62          & 27.12          & 5.62          \\
Stable Diffusion v2.0   & 27.48           & 26.89             & 26.86          & 27.46          & 27.17          & -             \\
SDXL Refiner 0.9        & 28.45           & 27.66             & 27.67          & 27.46          & 27.80           & -             \\
Dreamlike Photoreal 2.0 & 28.24           & 27.60              & 27.59          & 27.99          & 27.86          & -             \\ \midrule
Trajectory-level Reward & 29.37           & 28.81             & 28.83          & 29.16          & 29.04          & 5.94          \\
\textbf{Ours}                    & \textbf{30.46}  & \textbf{29.95}    & \textbf{30.01} & \textbf{29.93} & \textbf{30.09} & \textbf{6.31} \\ \bottomrule
\end{tabular}
% }
% \vspace{1.em}
\end{table*}

\myparagraph{Results.}
% \st{
% TODO: Bar plots for seen prompts on both ImageReward and Aesthetic
% }
We compare our method with the the original SD1.5 (``Orig.''), supervised fine-tuned model (``SFT''), DPOK, and the classical DPO-style objective, \ie, the approach of assuming trajectory-level reward, which is abbreviated as ``Traj.''.    % in the figures.
As discussed in \Secref{sec:connect_with_dpo}, ``Traj.'' can be obtained by setting $\gamma=1$ in our loss \eqref{eq:pk_nll_main}.
We follow the DPOK paper to plot the ImageReward in Fig.~\ref{fig:single_imagerew} and Aesthetic score in   Fig.~\ref{fig:single_aes} for the seen prompts, where the results for ``Orig.'', ``SFT'', and DPOK are directly from the DPOK paper.
Fig.~\ref{fig:single_prompt_generated_images} shows examples of the generated images from both our method and the baselines.
% , in which we drop ``SFT'' for conciseness since it underperforms DPOK.
More image comparisons are deferred to \cref{sec:single_prompt_more_images}.

As shown in Fig.~\ref{fig:single_imagerew} and Fig.~\ref{fig:single_aes},
our method can improve both ImageReward, the preference generating metric, and the unseen Aesthetic score.
The higher scores of our method over DPOK  on both metrics validate the efficacy of our method for T2I's preference alignment.
Comparing with ``Traj.'', our method improves more over the original SD1.5, which we attribute to our dense reward perspective, implemented by introducing temporal discounting to emphasize the initial steps of the diffusion reverse chain.
From Fig.~\ref{fig:single_prompt_generated_images}, it is clear that, on both seen and unseen text prompts, our method generates images that are not only faithfully matched with the prompts, but also of higher aesthetic quality, \eg, having more colorful details and/or backgrounds.
\Secref{sec:exp_abla} \textbf{(a)} compares the generation trajectories of our method and the baselines. 
Indeed, our method  generates the desired shapes earlier, which explains why it produces better final images.
% the gain of our approach.
% helps to explain our method's better final images.

% \st{
% @Tianqi
% Locations of the images and prompts 
% \begin{itemize}
%     \item Raw SD1.5: "/data/shentao/rldiffusion/results/multi\_prompt/sd15/"
%     \item Dreamlike Photoreal 2.0: "/data/shentao/rldiffusion/results/multi\_prompt/Dreamlike" (I think the image filenames are the indices in enumerating the prompts in the image folder)
%     \item Seq-level: "/data/shentao/rldiffusion/results/multi\_prompt/seq\_level/iter32000/"
%     \item Dense Rew (Ours): "/data/shentao/rldiffusion/results/multi\_prompt/ours/exp110"
% \end{itemize}
% }

% \st{
% TODO: Both tabular results (HPSv2 \& Aesthetic)
% }

% ----------------------------------------------------
\begin{figure*}[tb]
     % \vspace{-1.0mm}
     \centering
\includegraphics[width=0.9\textwidth]{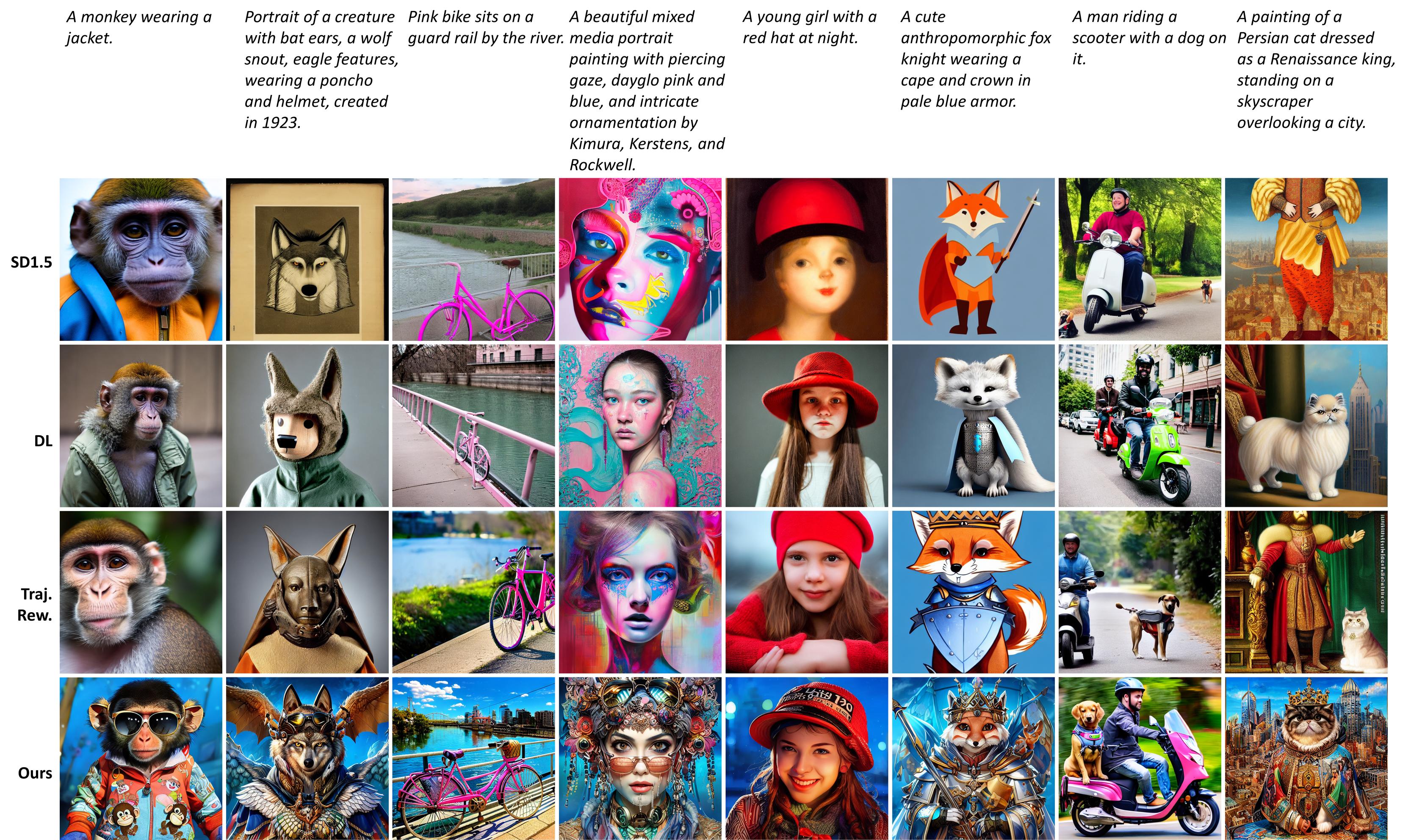}
     \vspace{-2mm}
     \captionsetup{font=small}
        \caption{ 
        \small
        Generated images in the multiple prompt experiment from our method and baselines, with prompts.
        ``DL'' denotes Dreamlike Photoreal 2.0, the best baseline from HPSv2 paper.
        ``Traj. Rew.'' is the classical DPO-style objective of assuming trajectory-level reward.
        % , (discussed in \Secref{sec:connect_with_dpo}).
        % More generation examples are in \cref{sec:multiple_prompt_more_images}.
        }
        \label{fig:multiple_prompt_generated_images}
        \vspace{-1.em}
\end{figure*}
% ----------------------------------------------------

\subsection{Multiple Prompts}\label{sec:exp_multiple_prompts}

\myparagraph{Settings.}
We consider a more challenging setting where we apply our method to train a T2I on the HPSv2 \citep{hpsv22023} train prompts and evaluate on the HPSv2 test prompts, which have no intersection with the train prompts.  % the trained T2I
We obtain preference by HPSv2 and report the average of both HPSv2 and Aesthetic score over all HPSv2 test prompts.
Due to  the large test-set size (3200 prompts), we follow the HPSv2 paper to generate one image per prompt for evaluation.
% large testset size

\myparagraph{Implementation.}
We use the same trajectory-collection strategy as in the single prompt experiments (\Secref{sec:exp_single}).
Due to the task complexity and the large size of the HPSv2 train set ($> 100,000$ prompts), we collect a total of $100,000$ trajectories, divided into ten collection stages.
Each stage collects $10,000$ trajectories and discards the previously collected ones.
We use LoRA with rank $32$ and train the model for a total of $M_\mathrm{tr}=40,000$ steps, and hence $M_\mathrm{col}=4000$ steps, $N_\mathrm{pr}=2000$.
We set the  KL coefficient $C=12.5$ and ablates the value of $C$ in \Secref{sec:exp_abla} \textbf{(c)}.
We use $N_\mathrm{step} = 1$ based on compute constraints such as GPU memory.
\cref{sec:exp_details_multiple_prompt} provides more hyperparameter settings.

\myparagraph{Results.}
Table~\ref{table:multiple_promtps_small} shows the HPSv2 and Aesthetic score for our method and selected relevant and/or strong baselines from the HPSv2 paper, with the
full set of baselines deferred to Table~\ref{table:multiple_promtps_all} of \cref{sec:add_results}.
All baselines available in  \href{https://github.com/tgxs002/HPSv2/tree/3ab15c150044de4c3f714493e9902c4ca3d44257}{HPSv2 Github Repository} are directly cited.
As in \Secref{sec:exp_single}, we further compare with the  classical DPO-style objective of assuming trajectory-level reward (\Secref{sec:connect_with_dpo}).
Fig.~\ref{fig:multiple_prompt_generated_images} shows examples of generated images from our method and baselines, with more image comparisons in \cref{sec:multiple_prompt_more_images}.

As seen in Table~\ref{table:multiple_promtps_small}, our method is able to improve the preference generating metric, HPSv2, and the unseen Aesthetic score.
The improvement from our method is larger than the variant of assuming trajectory-level reward, validating our insight of emphasizing the initial part of the T2I generation process, a product of our distinct dense reward perspective.
In Fig.~\ref{fig:multiple_prompt_generated_images}, we see that our method generates images well matched with the text prompts, in some cases better than the baselines, \eg, on the prompts of ``a girl at night,'' ``fox knight,'' and ``scooter with a dog on.''
From both short and the more challenging long prompts, our method is able to generate vivid images, often with sophisticated aesthetic shapes.
Together with the image examples in \cref{sec:multiple_prompt_more_images}, Fig.~\ref{fig:multiple_prompt_generated_images} qualitatively validates the efficacy of our method.

\subsection{Further Study}\label{sec:exp_abla}
This section considers the following four research questions to better understand our method.

\textbf{(a):} \textit{Does the T2I trained by our method indeed generate the desired shapes earlier in the diffusion reverse chain?}

% ----------------------------------------------------
% \begin{figure*}[t]
%      % \vspace{-1.0mm}
%      \centering
% \includegraphics[width=0.8\textwidth]{icml2024/Tex/fig/fig_single_prompt_traj.jpg}
%      \vspace{-2mm}
%      \captionsetup{font=small}
%         \caption{ 
%         \small
%         Generation trajectories from our method and baselines for the single prompt ``A green colored rabbit.'' XXXXX
%         }
%         \label{fig:single_prompt_generation_trajs}
%         % \vspace{-5mm}
% \end{figure*}

\begin{figure*}[tb]
     % \vspace{-1.0mm}
     \centering
\includegraphics[width=0.9\textwidth]{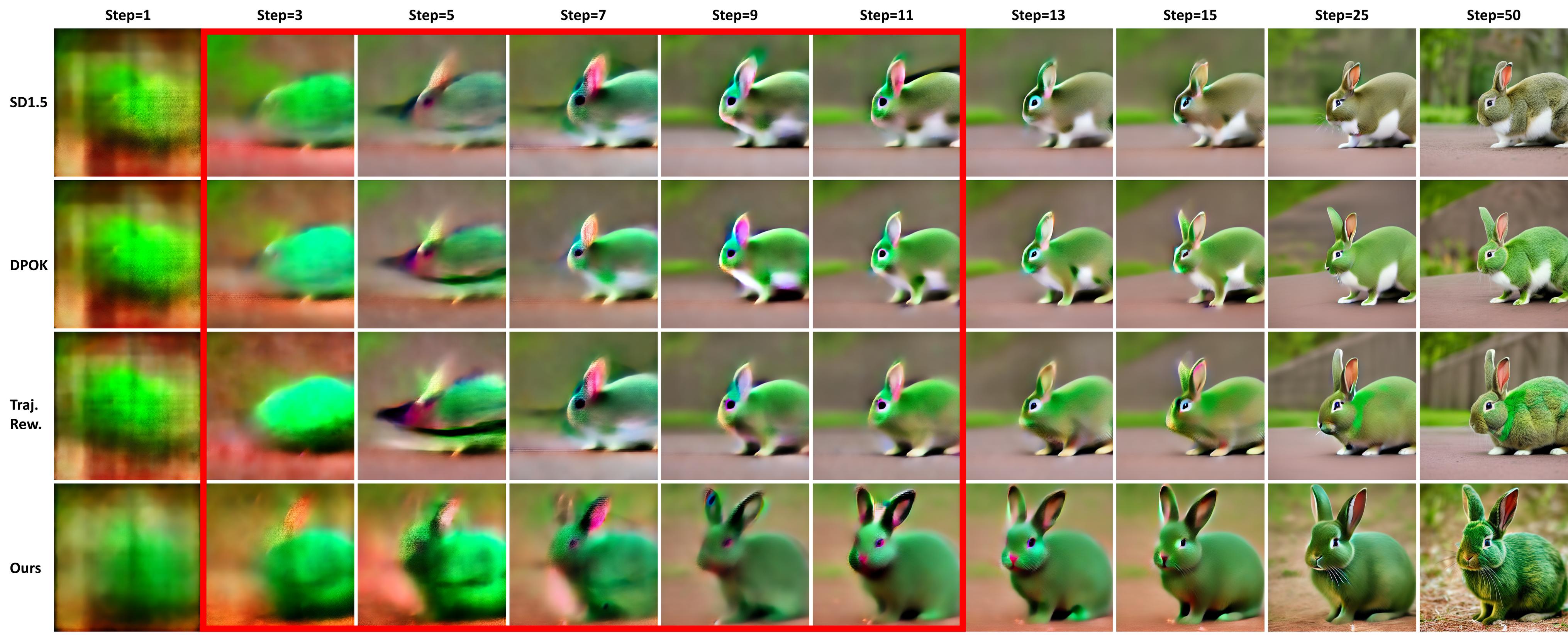}
     \vspace{-2mm}
     \captionsetup{font=small}
        \caption{ 
        \small
        Generation trajectories from our method and the baselines on the prompt ``A green colored rabbit.'' in the single prompt experiment, correspond to the images in Fig.~\ref{fig:single_prompt_generated_images}. 
        Shown are the  $\hat \vx_0$ predicted from the latents at the specified steps of the reverse chain.
        }
        \label{fig:single_prompt_generation_trajs}
        \vspace{-1em}
\end{figure*}
% ----------------------------------------------------

% ----------------------------------------------------
% Aestheric
\begin{figure*}[tb]
    \centering
    \begin{minipage}{0.49\textwidth}
        \centering
\includegraphics[width=.704\textwidth]{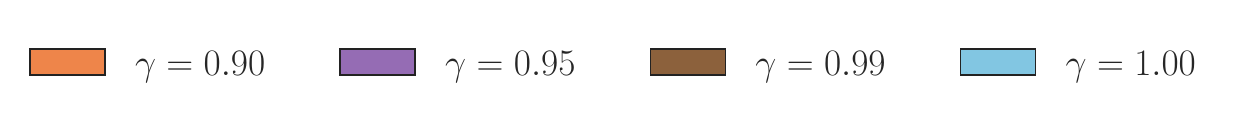}
\vspace{-.5em}
\\
     \begin{subfigure}[b]{0.48\textwidth}
         \centering
         \includegraphics[width=\textwidth]{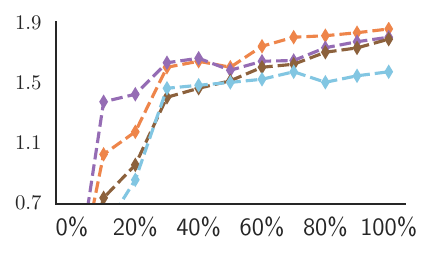}
         \captionsetup{font=footnotesize}
         \vspace{-6mm}
         \caption{\footnotesize{Single (ImageReward)}} % Sequence Num.
         \label{fig:single_color_gamma}
     \end{subfigure}
     \hfill
     \begin{subfigure}[b]{0.48\textwidth}
         \centering
         \includegraphics[width=\textwidth]{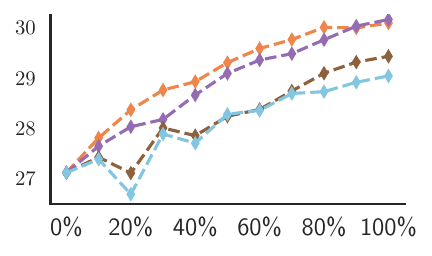}
         \captionsetup{font=footnotesize}
         \vspace{-6mm}
         \caption{\footnotesize{Multiple (HPSv2)}} % Sequence Num. 
         \label{fig:multiple_gamma}
     \end{subfigure}
     \vspace{-3mm}
     \captionsetup{font=small}
        \caption{ 
        \small
        Preference generating metrics over the training process, for the single and multiple prompt experiments under various discount factor $\gamma$. 
        $x$-axis represents $t\%$ of the training process.
        In \textbf{(a)} all lines start from $-0.02$ at $0\%$, the value of ``Orig.'' in Fig.~\ref{fig:single_imagerew_Color}.
        % in Table~\ref{table:four_single_prompts}.
        }
        \label{fig:line_varying_gamma}
    \end{minipage}
    % \rulesep
    \hfill
    \begin{minipage}{0.49\textwidth}
        \centering
\includegraphics[width=.88\textwidth]{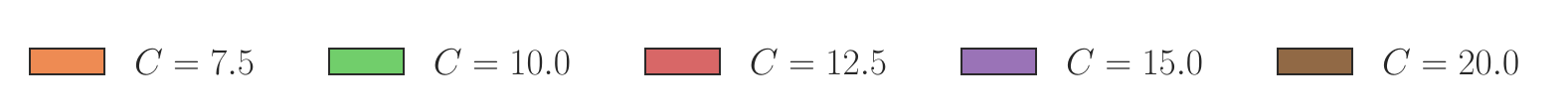}
\vspace{-.5em}
\\
     \begin{subfigure}[b]{0.48\textwidth}
         \centering
         \includegraphics[width=\textwidth]{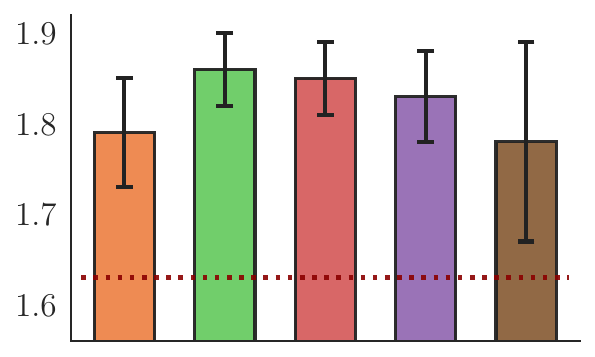}
         \captionsetup{font=small}
         \vspace{-6mm}
         \caption{\footnotesize{Single (ImageReward)}} % Sequence Num.
         \label{fig:single_kl}
     \end{subfigure}
     \hfill
     \begin{subfigure}[b]{0.48\textwidth}
         \centering
         \includegraphics[width=\textwidth]{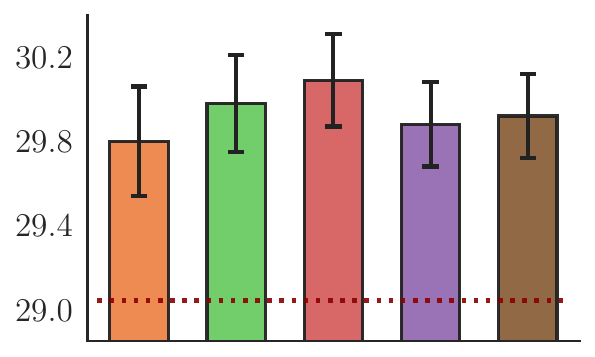}
         \captionsetup{font=small}
         \vspace{-6mm}
         \caption{\footnotesize{Multiple (HPSv2)}}
         \label{fig:multiple_kl}
     \end{subfigure}
     \vspace{-3mm}
     \captionsetup{font=small}
        \caption{ 
       \small
       Preference generating metrics with error bars showing one standard deviation, for the single and multiple prompt experiments under various KL coefficient $C$. 
       Horizontal line indicates the best baseline result from Fig.~\ref{fig:single_imagerew_Color} and Table~\ref{table:multiple_promtps_small}.
        }
        \label{fig:bar_varying_kl}
    \end{minipage}
    \vspace{-1.2em}
\end{figure*}
% ----------------------------------------------------

% \st{
% @Tianqi
% Single prompt: compare our method and the baselines on the predicted $\hat x_0$ at each step of the reverse chain, showing that our method generates the required shape early and better
% }

As discussed in \Secref{sec:intro}, we hypothesize that emphasizing the initial steps of the T2I generation trajectory can help the effectiveness and efficiency of preference alignment.
As a verification, Fig.~\ref{fig:single_prompt_generation_trajs} digs into the generated images of the prompt ``A green colored rabbit.'' in the single prompt experiment, by showing the generation trajectories corresponding to the images in Fig.~\ref{fig:single_prompt_generated_images}.
Specifically, we compare our method and the baselines on the images $\hat \vx_0$ predicted from the latents at the specified timesteps of the reverse chain.
More trajectory comparisions are in \cref{sec:more_traj_single_prompt}.

As shown in Fig.~\ref{fig:single_prompt_generation_trajs}, and in particular the steps circled out by the red rectangle therein, our method can generate identifiable shapes of a rabbit as early as at Steps $3$ and $5$, while the baselines are still largely unrecognizable, \eg, similar to a mouse.
At step $11$, our method is able to produce a relatively complete image to the given prompt, while the baselines are much cruder.
This comparison confirms that, with the incorporation of $\gamma < 1$, our method can match the given prompt earlier in the reverse chain, 
and thereby more steps later in the chain can be allocated to polish pictorial details and aesthetics, leading to  better/preferable final images.

\textbf{(b): } \textit{What will happen if we change the value of $\gamma$?}
% \myparagraph{ The impact of the discount factor }

% \st{
% TODO: for "Color" and "multiprompts", add line plots of the training metric over the training process, 
% }

To investigate the impact of temporal discount factor $\gamma$ on training T2I for preference alignment,
we consider more values of $\gamma$ between $\gamma=0.9$ used in our main results, and $\gamma=1$ in the classical approach of trajectory-level reward.
% , which leads to the  trajectory-level reward in the classical DPO-style objective.
Fig.~\ref{fig:line_varying_gamma} plots the preference generating metrics over the training process, under $\gamma \in \{0.9, 0.95, 0.99, 1.0\}$, for the single prompt (``A green colored rabbit.'') and multiple prompt experiments.
We use the same evaluation protocols as in the main results.
For HPSv2, we plot the average over the test set.
Patterns on other single prompts are similar.

As shown in Fig.~\ref{fig:line_varying_gamma}, using a smaller temporal discount factor, such as $\gamma=0.9$ or $\gamma=0.95$, trains T2I faster and better, compared to larger $\gamma$ values, especially the classical DPO-style loss of $\gamma=1$.
Recall from \Secref{sec:intro} that a smaller $\gamma$ emphasizes more on the initial part of the reverse chain, while a sparse trajectory-level reward, equivalent to $\gamma=1$, can incur training instability.
In Fig.~\ref{fig:line_varying_gamma}, on both experiments, $\gamma=0.9$ or $\gamma=0.95$ generally leads to larger improvement at the beginning of the training process.
This validates our intuition and prior study that stressing the earlier steps of the reverse chain could improve the training efficiency of aligning T2I with preference.
From Fig.~\ref{fig:line_varying_gamma}, even using $\gamma=0.99$, a small break on the temporal
symmetry in the DPO-style losses, can improve training efficiency and stability over the classical setting of $\gamma=1$.
This further corroborates the efficacy of our dense reward perspective on T2I's alignment.
\cref{sec:smaller_gamma_not_better} further discusses the effect of $\gamma$ on training T2I.

\textbf{(c):} \textit{Is our method robust to the choice of KL coefficient $C$?}
% \myparagraph{The choice of the temperature/}

% \st{
% TODO: for "Color" and "multiprompts", add bar plots of the training metric with error bars,
% }

To study the sensitivity of our method to the KL coefficient $C$ in our loss \eqref{eq:pk_nll_main}, we vary the value of $C$ from the values set in Sections~\ref{sec:exp_single}~and~\ref{sec:exp_multiple_prompts}.
Fig.~\ref{fig:bar_varying_kl} plots the scores of the preference generating metrics for experiments in the single prompt (``A green colored rabbit.'') and multiple prompts.
Other single prompts show similar patterns.
For HPSv2, we again plot the average over the test set, with Aesthetic and breakdown scores for each style in Table~\ref{table:multiple_varying_c} at \cref{sec:add_results}.

From Fig.~\ref{fig:bar_varying_kl}, we see that our method is generally robust across a range of KL coefficient $C$.
A small value of $C$ may be prone to overfitting while a large value may distract/slow the training process, both of which deteriorate the results.

\textbf{(d):} \textit{Are the images from our method preferred by humans?}
% \textbf{(d):} \textit{Human evaluation on the generated images?}

\begin{table}[tb]
% \vspace{-1em}
% \captionsetup{font=small}
\caption{
\small
Human evaluation on the multiple prompt experiment.
Shown are our ``win rate'' against the baselines specified in Fig.~\ref{fig:multiple_prompt_generated_images}, \ie, the percentage of times our method is preferred in binary comparisons.
Detailed description on the setup is in \cref{sec:human_eval_setup}.
} 
\label{table:multiple_prompt_human}
\centering 
\vspace{-.7em}
 % \setlength\tabcolsep{9pt}
% \renewcommand{\arraystretch}{1.1}
% \resizebox{\textwidth}{!}{
% \footnotesize
\begin{tabular}{cccc}
\toprule
  Opponent                   & SD1.5 & Dreamlike & Traj. Rew. \\ \midrule
Win Rate &   $76.8\%$    &      $68.3\%$      &     $65.1\%$               
\\ \bottomrule
\end{tabular}
% }
% \vspace{-1.5em}
\end{table}

To further verify our method, we collect human evaluations on the generated images in the multiple prompt experiment, where binary comparisons between two images from two models are conducted.
Table~\ref{table:multiple_prompt_human} shows the ``win rate'' of our method over each of the baselines in Fig.~\ref{fig:multiple_prompt_generated_images}.
Detailed setups of the human evaluation are provided in \cref{sec:human_eval_setup}.

The preference for our method over each baseline is evident in Table~\ref{table:multiple_prompt_human}.
Recall that the preference source, HPSv2 scorer, is trained on human preference data.
The gain of our method over raw SD1.5 verifies the efficacy of our method in aligning T2I with preference.
Further, images from our method are more often preferred over the corresponding images from the classical trajectory-level reward approach. 
This again validates our dense reward perspective that introduces temporal discounting into T2I's preference alignment.
% , which emphasizes the more important initial steps of the reverse train.

\section{Conclusion}
To suit the explicit-reward-free preference-alignment loss to the sequential generation nature of T2I and improve on the classical trajectory-level reward assumption,
in this paper, we take on a dense reward perspective and introduce temporal discounting into the alignment objective, motivated by both an easier learning task in RL and the generation hierarchy of T2I reverse chain.
By experiments and further studies, we validate the efficacy of our method and reveal its key insight.
Future work may involve extending our method to noisy preference labels and applying it to broader applications, such as text-to-video or image-to-image generation.
% Future work may extend our method onto noisy preference labels, and apply our method to broader applications, such as text-to-video or image-to-image generation. 

\section*{Impact Statement}

Our paper contributes to the ongoing research on increasing helpfulness and decreasing harmfulness of generative models, by proposing a method that seeks to improve the efficacy and efficiency of aligning T2I with preference.
Of a special note, our method does not require training an explicit reward model, which can potentially save some compute and resources.
On the other hand, as prior preference alignment methods,
it is possible that our method will be misused to train malicious T2I by aligning with some unethical or ill-intended preference.
This potential negative impact may be alleviated by a more closer monitoring on the datasets and preference sources to which our method is applied.

\section*{Limitations}

As with classical off-policy RL and RLHF methods, our method’s iteration between model training and data collection incurs additional complexity and costs, compared to the pure offline approach of gathering data only once prior to policy training. On the other hand, it is known that off-policy methods can reduce the mismatch between learning policy’s generation distribution and the data distribution, and generally lead to more stable training and better results than pure offline methods.
Another limitation of our method is that our method requires storing the generation reverse chains.
Though this is feasible and straightforward given the
relatively-small scale of the preference alignment stage, our approach does raise extra CPU-memory and/or storage requirements, compared to only storing the final images and discarding all generation latents.
As an example, in our experiments with SD1.5, storing the generation latents requires about two times more CPU memory (\textit{not} GPU memory), calculated as $50 \times 4 \times 64^2 / (512^2 \times 3) \times (16 / 8) \approx 2.08$, where the last multiplier comes from the fact that our generation latents are stored in \texttt{bfloat16} format and the final images are in \texttt{uint8}.
This limitation may be further alleviated by using a more advanced diffusion/T2I sampler.

\section*{Acknowledgements}
The authors acknowledge the support of NSF-IIS 2212418, NIH-R37 CA271186, McCombs REG, and TACC. S. Yang acknowledges the support of the University of Texas Graduate Continuing Fellowship.

% \textbf{Do not} include acknowledgements in the initial version of
% the paper submitted for blind review.

% If a paper is accepted, the final camera-ready version can (and
% probably should) include acknowledgements. In this case, please
% place such acknowledgements in an unnumbered section at the
% end of the paper. Typically, this will include thanks to reviewers
% who gave useful comments, to colleagues who contributed to the ideas,
% and to funding agencies and corporate sponsors that provided financial

% \clearpage

\bibliographystyle{icml2024}
\bibliography{ref}

%%%%%%%%%%%%%%%%%%%%%%%%%%%%%%%%%%%%%%%%%%%%%%%%%%%%%%%%%%%%%%%%%%%%%%%%%%%%%%%
%%%%%%%%%%%%%%%%%%%%%%%%%%%%%%%%%%%%%%%%%%%%%%%%%%%%%%%%%%%%%%%%%%%%%%%%%%%%%%%
% APPENDIX
%%%%%%%%%%%%%%%%%%%%%%%%%%%%%%%%%%%%%%%%%%%%%%%%%%%%%%%%%%%%%%%%%%%%%%%%%%%%%%%
%%%%%%%%%%%%%%%%%%%%%%%%%%%%%%%%%%%%%%%%%%%%%%%%%%%%%%%%%%%%%%%%%%%%%%%%%%%%%%%
\newpage
\appendix
\onecolumn

\begin{center}
\Large
\textbf{Appendix}
\end{center}

\setcounter{tocdepth}{1}
\tableofcontents
\newpage

% \section{Additional Experiment Results}
\section{Tabular Results}
\label{sec:add_results}

\begin{table}[H]
\captionsetup{font=small}
\caption{
\small
HPSv2 and Aesthetic score for the multiple prompt experiment in \Secref{sec:exp_multiple_prompts}. 
The first four result columns are the four styles in the HPSv2 test set and ``Average'' is the overall average.
``Trajectory-level Reward'' is the classical DPO-style objective discussed in \Secref{sec:connect_with_dpo}, which assumes a latent trajectory-level reward function evaluating the entire T2I reverse chain as a whole.
All baselines benchmarked in the HPSv2 paper are directly cited from the official \href{https://github.com/tgxs002/HPSv2/tree/3ab15c150044de4c3f714493e9902c4ca3d44257}{Github Repository}.
Our produced results follow the testing principle in the HPSv2 paper and GitHub Repository.
We bold the best result in each metric.
Note that the HPSv2 paper and Github Repository do not report the Aesthetic score. 
} 
\label{table:multiple_promtps_all}
\centering 
% \vspace{-.5em}
% \setlength\tabcolsep{12pt}
% \renewcommand{\arraystretch}{1.1}
% \resizebox{1.\textwidth}{!}
% {
% \footnotesize
\begin{tabular}{@{}lccccccc@{}}
\toprule
Model                   & Animation & Concept-art & Painting & Photo    & Average & Aesthetic     \\ \midrule
GLIDE                   & 23.34           & 23.08             & 23.27          & 24.50           & 23.55          & -             \\
LAFITE                  & 24.63           & 24.38             & 24.43          & 25.81          & 24.81          & -             \\
VQ-Diffusion            & 24.97           & 24.70              & 25.01          & 25.71          & 25.10           & -             \\
FuseDream               & 25.26           & 25.15             & 25.13          & 25.57          & 25.28          & -             \\
Latent Diffusion        & 25.73           & 25.15             & 25.25          & 26.97          & 25.78          & -             \\
DALL·E mini             & 26.10            & 25.56             & 25.56          & 26.12          & 25.83          & -             \\
VQGAN + CLIP            & 26.44           & 26.53             & 26.47          & 26.12          & 26.39          & -             \\
CogView2                & 26.50            & 26.59             & 26.33          & 26.44          & 26.47          & -             \\
Versatile Diffusion     & 26.59           & 26.28             & 26.43          & 27.05          & 26.59          & -             \\
DALL·E 2                & 27.34           & 26.54             & 26.68          & 27.24          & 26.95          & -             \\
Stable Diffusion v1.4   & 27.26           & 26.61             & 26.66          & 27.27          & 26.95          & -             \\
Stable Diffusion v1.5   & 27.43           & 26.71             & 26.73          & 27.62          & 27.12          & 5.62          \\
Stable Diffusion v2.0   & 27.48           & 26.89             & 26.86          & 27.46          & 27.17          & -             \\
Epic Diffusion          & 27.57           & 26.96             & 27.03          & 27.49          & 27.26          & -             \\
DeepFloyd-XL            & 27.64           & 26.83             & 26.86          & 27.75          & 27.27          & -             \\
Openjourney             & 27.85           & 27.18             & 27.25          & 27.53          & 27.45          & -             \\
MajicMix Realistic      & 27.88           & 27.19             & 27.22          & 27.64          & 27.48          & -             \\
ChilloutMix             & 27.92           & 27.29             & 27.32          & 27.61          & 27.54          & -             \\
Deliberate              & 28.13           & 27.46             & 27.45          & 27.62          & 27.67          & -             \\
SDXL Base 0.9           & 28.42           & 27.63             & 27.60           & 27.29          & 27.73          & -             \\
Realistic Vision        & 28.22           & 27.53             & 27.56          & 27.75          & 27.77          & -             \\
SDXL Refiner 0.9        & 28.45           & 27.66             & 27.67          & 27.46          & 27.80           & -             \\
Dreamlike Photoreal 2.0 & 28.24           & 27.60              & 27.59          & 27.99          & 27.86          & -             \\ \midrule
Trajectory-level Reward & 29.37           & 28.81             & 28.83          & 29.16          & 29.04          & 5.94          \\
\textbf{Ours}                    & \textbf{30.46}  & \textbf{29.95}    & \textbf{30.01} & \textbf{29.93} & \textbf{30.09} & \textbf{6.31} \\ \bottomrule
\end{tabular}
% }
\end{table}

\begin{table}[H]
\captionsetup{font=small}
\caption{
\small
HPSv2 and Aesthetic score for the ablation study on KL coefficient $C$ in \Secref{sec:exp_abla} \textbf{(c)}.
Shown here are breakdown scores of our main method ($\gamma=0.9$) in the \textit{multiple} prompt experiment under various value of $C$, together with
the best baseline in Table~\ref{table:multiple_promtps_all} of \cref{sec:add_results}.
The first four result columns are the four styles in HPSv2 test set and ``Average'' is the overall average.
Within subscript is one standard deviation, as plotted in Fig.~\ref{fig:multiple_kl}, calculated by the principle described in the HPSv2 paper and GitHub Repository.
} 
\label{table:multiple_varying_c}
\centering 
% \vspace{-.8em}
% \setlength\tabcolsep{12pt}
% \renewcommand{\arraystretch}{1.1}
% \resizebox{1.\textwidth}{!}
% {
% \footnotesize
\begin{tabular}{@{}lcccccc@{}}
\toprule
\multirow{2}{*}{Model} & \multicolumn{5}{c}{HPSv2}                             &           \\ \cmidrule(lr){2-6}
                       & Animation & Concept-art & Painting & Photo & Averaged & Aesthetic \\ \midrule
Baseline               & 29.37     & 28.81       & 28.83    & 29.16 & 29.04    & 5.94      \\
$C=7.5$                  & 30.16     & 29.59       & 29.64    & 29.76 & 29.79 {\scriptsize (0.26) }    & 6.24      \\
$C=10.0$                 & 30.36     & 29.87       & 29.91    & 29.80  & 29.99 {\scriptsize (0.23) }    & 6.29      \\
$C=12.5$                 & 30.46     & 29.95       & 30.01    & 29.93 & 30.09 {\scriptsize (0.22) }    & 6.31      \\
$C=15.0$                 & 30.30     & 29.72       & 29.74    & 29.77 & 29.88 {\scriptsize (0.20) }    & 6.23      \\
$C=20.0$                 & 30.28     & 29.73       & 29.76    & 29.95 & 29.93 {\scriptsize (0.20) }    & 6.17      \\ \bottomrule
\end{tabular}
% }
\end{table}

% \input{./Tex/tab/clip_score_single_prompt}
% <---------------------------------------------------------------> %
\section{Detailed Method Derivation and Proofs}\label{sec:detailed_method_derive_proof}

In this section, we provide a detailed step-by-step derivation of our method. 
For completeness and better readability, some materials in \cref{sec:method} will be restated.

\subsection{Notation and Assumptions}

This section restates the notations and assumptions in \cref{sec:notation} for convenience.

\AssumpDenseRew*

We adopt the notations in prior works \citep[\eg,][]{fan2023dpok,ddpo2023} to formulate the diffusion reverse  process under the conditional generation setting as an Markov decision process (MDP), specified by $\gM = \br{\sS, \sA, \gP, r, \gamma, \rho}$. Specifically, let $\pi_\theta$ be the T2I with trainable parameters $\theta$, \ie, the policy network; $\cbr{\vx_t}_{t=T}^0$ be the  diffusion reverse chain of length $T$; and $\vc$ be the conditioning variable, \ie, the text conditional in our setting. We have, $\forall \, t$,

\begin{equation*}
    \begin{array}{ccc}
      s_t \triangleq (\vx_t, t, \vc) \,, & a_t \triangleq \vx_{t-1}  \,,  &  \pi_\theta(a_t \given s_t) \triangleq p_\theta(\vx_{t-1} \given \vx_t, t, \vc) \,, \\
       \gP(s_{t+1}\given s_t, a_t) \triangleq \delta(\vx_{t-1}, t-1, \vc) \,, & \rho(s_0) \triangleq (\gN(\vzero, \mI), \delta(T), \delta(\vc)) \,, & r(s_t, a_t)\,,\; \gamma \in [0,1]\,,
    \end{array}
\end{equation*}
where $\delta(\cdot)$ is the delta measure
% , $p(\vc)$ is some distribution of the conditioning variable, 
and $\gP(s_{t+1}\given s_t, a_t)$ is a deterministic transition.
We denote the reverse chain generated by a (generic) T2I under the text conditional $\vc$ as a trajectory $\tau$, \ie, $ \tau \triangleq (s_0, a_0, s_1, a_1, \ldots, s_T) \iff (\vx_T, \vx_{T-1}, \ldots, \vx_0) \given \vc $.
Note that for notation simplicity, $\vc$ is absorbed into the state part of trajectory $\tau$.

Similar to \citet{wallace2023diffusion}, in the method derivation, we consider the setting where we are given two  diffusion reverse chains (trajectories) with equal length $T$. For presentation simplicity, assume that $\tau^1$ is the better trajectory, \ie, $\tau^1 \succ \tau^2$. Let tuple $\ord\triangleq (1,2)$ and $\sigma(\cdot)$ denotes the sigmoid function, \ie, $\sigma(x) = \frac{1}{1+\exp\br{-x}}$.

Since in practice the state space of the T2I reverse chain is the continuous embedding space, it is self-evident to assume that any two trajectories do not cross with each other, as follows.

\begin{assumption}[No Crossing Trajectories] % \label{assump:no_cross}
     $\forall\,\tau^i \ne \tau^j, s^i_t \ne s^j_t,\, \forall\, t \in \cbr{0,\ldots, T}$.
\end{assumption}

Furthermore, as in the standard RL setting \citep{rlintro2018,jointmatching2022}, the reward function $r$ in $\gM$ needs to be bounded. Without loss of generality, we assume $r(s,a) \in [0,1]$, and thus $r(s,a)$ may be interpreted as the probability of satisfying the preference when taking action $a$ at state $s$.

\AssumpBddRew*

In RL problems, the performance of a (generic) policy $\pi$ is typically evaluated by the expected cumulative discounted rewards \citep{rlintro2018}, which is defined as,  
\begin{equation}\label{eq:rl_objective}
    \eta(\pi) \triangleq \E\sbr{\sum_{t=0}^T \gamma^t\, r(s_t, a_t) \given s_0\sim \rho(\cdot), a_t \sim \pi(\cdot \given s_t), s_{t+1} \sim \gP(\cdot \given s_t, a_t), \forall\, t\geq 0}\,.
\end{equation}
Note that \eqref{eq:rl_objective} above is an extended version of \eqref{eq:rl_objective_main} in \cref{sec:notation}.

\AssumpEvalETau*

\cref{remark:rationality_e_tau} in \cref{sec:notation} provides a discussion on the practical rationality of $e(\tau)$ in T2I's preference alignment.

\subsection{Step-by-step Derivation of Our Method}\label{sec:step_by_step_derivation}

\subsubsection{Expression of $e(\tau)$} \label{sec:exp_e_tau}

We can express $\eta(\pi)$ in \eqref{eq:rl_objective} by the (discounted) stationary distribution $d_\pi(s)$ of the policy $\pi$, defined as $d_\pi(s) \propto \sum_{t=0}^T \gamma^t \Pr(s_t = s \given \pi, \gP)$, up to a (positive) normalizing constant \citep{sdmgan2022,wmbrl2022}.
We have
\begin{equation*}
    \begin{split}
\eta(\pi) &= \E_{a_t \sim \pi, s_{t+1}\sim \gP}\sbr{\sum_{t=0}^T \gamma^t\, r(s_t, a_t)} = \sum_{t=0}^T \int_\sS \Pr(s_t = s \given \pi, \gP) \int_{\sA} \pi(a\given s)\, \gamma^t \, r(s,a) \intd a \intd s  \\
&=  \int_\sS \sum_{t=0}^T \gamma^t \Pr(s_t = s \given \pi, \gP) \int_{\sA} \pi(a\given s) \, r(s,a) \intd a \intd s \\ 
&\propto  \int_\sS d_\pi(s) \int_{\sA} \pi(a\given s) \, r(s,a) \intd a \intd s 
=  \E_{s\sim d_\pi(s)}\E_{a\sim \pi(a\given s)}\sbr{r(s,a)} \,.
    \end{split}
\end{equation*}

The goal of RL is to maximize the expected cumulative discounted rewards $\eta(\pi)$,  which is unfortunately difficult due to the complicate relationship between $d_\pi(s)$ and $\pi$. 
We therefore optimize an off-policy approximation of $\eta(\pi)$ by employing an approximation approach common in prior RL works \citep[\eg,][]{kakade2002approximately,peters2010relative,trpo2015,abdolmaleki2018maximum,awr2019}. 
Specifically, we change $d_\pi(s)$ to $d_{\pi_O}(s)$ for some ``old'' policy $\pi_O$, from which we generate the off-policy trajectories/data. 
We further add a KL regularization on $\pi$ towards the initial pre-trained model $\pi_I$ to avoid generating unnatural images. 
In sum, we arrive at the following constrained policy search problem
\begin{equation}\label{eq:optim_strict}
    \begin{split}
        \arg\max_\pi \quad & \E_{s\sim d_{\pi_O}(s)}\E_{a\sim \pi(a\given s)}\sbr{r(s,a)} \\
        \mathrm{s.t.} \quad 
        & \KL\br{\pi(\cdot \given s) \,\|\, \pi_I(\cdot \given s)} \leq \epsilon \,,\quad \forall \, s\in \sS \\
        & \int_\sA \pi(a\given s) \intd a = 1 ,\quad \forall \, s\in \sS,
    \end{split}
\end{equation}
where $\pi_O$ may be chosen as $\pi_I$ or some saved policy checkpoint not far away from $\pi_I$.

Enforcing the pointwise KL-regularization in \eqref{eq:optim_strict} is difficult, as in AWR \citep{awr2019}, we change the pointwise KL-regularization into enforcing the regularization only in expectation $\E_{s\sim d_{\pi_O}}\sbr{\cdots}$  and change \eqref{eq:optim_strict} into a regularized maximization problem  
\begin{equation}\label{eq:optim_reg}
    \begin{split}
        \arg\max_\pi \quad & \E_{s\sim d_{\pi_O}(s)}\E_{a\sim \pi(a\given s)}\sbr{r(s,a)} - C \cdot \E_{s\sim d_{\pi_O}(s)}\sbr{\KL\br{\pi(\cdot \given s) \,\|\, \pi_I(\cdot \given s)}}  \\
        \mathrm{s.t.} \quad
        & \int_\sA \pi(a\given s) \intd a = 1 ,\quad \forall \, s\in \sS.
    \end{split}
\end{equation}
The Lagrange form of the maximization problem \eqref{eq:optim_reg} is
\begin{equation}
    \begin{split}
          \gL(\pi) \triangleq \E_{s\sim d_{\pi_O}(s)}\E_{a\sim \pi(a\given s)}\sbr{r(s,a)} -  C \cdot \E_{s\sim d_{\pi_O}(s)}\sbr{\KL\br{\pi(\cdot \given s) \,\|\, \pi_I(\cdot \given s)}}  + \int_\sS \alpha_s \br{1-\int_\sA \pi(a\given s) \intd a} \intd s \,
           .
    \end{split}
\end{equation}
$\forall\, s\in\sS, a \in \sA$, the optimal policy under $\gL(\pi)$ can be obtained by setting the derivatives \wrt\mbox{} $\pi(a\given s)$ equal to $0$. 
We have
\begin{equation}\label{eq:formula_r}
    \begin{split}
         \frac{\partial\, \gL(\pi)}{\partial\, \pi(a\given s)} &= d_{\pi_O}(s)r(s,a) - C d_{\pi_O}(s)\log\pi(a\given s)- C d_{\pi_O}(s)  +C d_{\pi_O}(s)\log \pi_I(a\given s) - \alpha_s = 0 \\
        \implies\;\,  r(s,a) 
        &= C\log\frac{\pi^*(a\given s)}{\pi_I(a\given s)} +C + \frac{ \alpha_s}{d_{\pi_O}(s)}
    \end{split}
\end{equation}
where $\pi^*$ is the optimal policy under $r$.

From \eqref{eq:formula_r}, we can also get the formula for the optimal policy $\pi^*$ as 
\begin{equation}\label{eq:optimal_policy}
    \begin{split}
        \pi^*(a\given s) &= \exp\br{\frac{1}{C} r(s,a)} \pi_I(a\given s) \exp\br{-1-\frac{\alpha_s}{C  d_{\pi_O}(s)}}  \triangleq \exp\br{\frac{1}{C} r(s,a)} \pi_I(a\given s) \frac{1}{Z(s)},
    \end{split}
\end{equation}
where $Z(s)$ denotes the partition function, taking the form
\begin{equation*}
    Z(s) = \int_\sA \exp\br{\frac{1}{C} r(s,a)}  \pi_I(a\given s)  \intd a.
\end{equation*}

For a \textit{given} trajectory $\tau = (s_0, a_0, s_1, a_1, \ldots, s_T)$, the quality evaluation $e(\tau)$ can be expressed by $\pi^*$ as 
\begin{equation}
    \begin{split}
        e(\tau) &\triangleq  \sum_{t=0}^T \gamma^t\, r(s_t,a_t) 
        = C\sum_{t=0}^T \gamma^t \log\frac{\pi^*(a_t\given s_t)}{\pi_I(a_t\given s_t)} + C \sum_{t=0}^T\gamma^t \br{1+\frac{ \alpha_{s_t}}{C d_{\pi_O}(s_t)}} \\
        &= C\sum_{t=0}^T \gamma^t \log\frac{\pi^*(a_t\given s_t)}{\pi_I(a_t\given s_t)} + C \sum_{t=0}^T\gamma^t \log Z(s_t).
    \end{split}
\end{equation}
Since the trajectory $\tau$ and hence all $s_t$'s are given, $Z(s_t)$'s are constant and the summation over $\log Z(s_t)$ is a ``property'' of the trajectory $\tau$, we thus denote $\log Z(\tau) \triangleq \sum_{t=0}^T\gamma^t \log Z(s_t)$ for notation simplicity. Then the formula for $e(\tau)$ becomes 
\begin{equation}\label{eq:e_tau_with_log_z}
    e(\tau) = C\sum_{t=0}^T \sbr{ \gamma^t \log\frac{\pi^*(a_t\given s_t)}{\pi_I(a_t\given s_t)} } + C \log Z(\tau) \,.
\end{equation}

\subsubsection{Loss Function for T2I/Policy Training}\label{sec:mle_obj_derive}
Recall that we are given two  diffusion reverse chains (trajectories) $\cbr{\tau^1, \tau^2}$ with equal length $T$. 
Also recall the notation that $\tau^1$ is the better trajectory, \ie, $\tau^1 \succ \tau^2$, the tuple $\ord\triangleq (1,2)$ and $\sigma(\cdot)$ denotes the sigmoid function. 
Under the Bradley-Terry model of pairwise preference, the probability of $\ord$ under $\cbr{e\br{\tau^k}}_{k=1}^2$ and hence $\pi^*$ is
\begin{equation}\label{eq:pl_original}
    \Pr\br{\ord \given \pi^*, \cbr{e\br{\tau^k}}_{k=1}^2} = \frac{\exp\br{e\br{\tau^1}}}{\exp\br{e\br{\tau^1}} + \exp\br{e\br{\tau^2}}} = \sigma\br{e\br{\tau^1} - e\br{\tau^2}}\,,
\end{equation}
where we explicitly put $\pi^*$ into the conditioning variables for better readability.

From \eqref{eq:e_tau_with_log_z}, $e\br{\tau^1}$ and $e\br{\tau^2}$ respectively contains the ``partition functions'' $Z(\tau^1)$ and $Z(\tau^2)$, both of which are intractable.
We argue that $Z(\tau^1)\geq Z(\tau^2)$, which will be critical for providing a tractable lower bound of \eqref{eq:pl_original} that cancels out these partition functions. 
Our argument is based on the reward-shaping technique \citep{rewardshaping1999}, as follows.

\DefRewShape*

\LemmaInvOptPol*

\begin{remark}
    The only difference between the MDPs $\gM$ and $\gM'$ is the reward function ($r$ \textit{v.s.} $r'$). 
    In particular, they share the same state and action space.
    Therefore, it make sense to consider the invariance of the optimal policy in these two MDPs.
    Invariance means that, in these two MDPs, at each state, the optimal policies take the same action with the same probability.
\end{remark}

\begin{proof}[Proof of \cref{lemma:pi_star_inv}]
Denote the optimal policy under the MDP $\gM'$ as $\pi^{*'}$, we have
\begin{equation}\label{eq:proof_inv_opt_pol}
    \begin{split}
        \pi^{*'}(a\given s) &= \frac{\exp\br{\frac{1}{C} \br{r(s,a) + \Phi(s)}} \pi_I(a\given s) }{\int_\sA \exp\br{\frac{1}{C} \br{r(s,a) + \Phi(s)}}\pi_I(a\given s)  \intd a} =  \frac{\exp\br{\frac{1}{C}\Phi(s)}\exp\br{\frac{1}{ C} r(s,a)} \pi_I(a\given s) }{\exp\br{\frac{1}{C}\Phi(s)} \int_\sA \exp\br{\frac{1}{C} r(s,a) } \pi_I(a\given s)  \intd a}\\
        &= \pi^*(a \given s),
    \end{split}
\end{equation}    
since $\exp\br{\frac{1}{C}\Phi(s)}$ is independent of the integration dummy-variable $a$ in the denominator.
\end{proof}

\DefEquivClassR*
\RemarkInvEquivClass*

With the reshaping technique, we can justify our previous argument that $Z(\tau^1)\geq Z(\tau^2)$ as follows. 

\ThmOrderPartitionAppendix*

We defer the proof of Theorem~\ref{theorem:order_partition} to \Secref{sec:proofs}.

\RemarkValueC*

\RemarkCallRprimtAsR*

% \subsubsection{Loss Function for Policy Training}
\label{sec:loss_function_derivation}
With Theorem~\ref{theorem:order_partition}, we can lower bound $\Pr\br{\ord \given \pi^*, \cbr{e\br{\tau^k}}_{k=1}^2}$ in \eqref{eq:pl_original} by a simpler formula.
After plugging the expression of $e(\tau)$ \wrt\mbox{} the optimal policy $\pi^*$ in \eqref{eq:e_tau_with_log_z}, we have,
\begin{equation}\label{eq:pl_lower_bound}
    \begin{split}
        \Pr\br{\ord \given \pi^*, \cbr{e\br{\tau^k}}_{k=1}^2} &=  \frac{\exp\br{C\sum_{t=0}^T\gamma^t \log \frac{\pi^*\br{a^1_t \given s^1_t}}{\pi_I\br{a^1_t \given s^1_t}}}Z\br{\tau^1}^{C}}{\sum_{i=1}^2 \exp\br{C\sum_{t=0}^T\gamma^t \log \frac{\pi^*\br{a^i_t \given s^i_t}}{\pi_I\br{a^i_t \given s^i_t}}}Z\br{\tau^{\textcolor{burgundy}{i}}}^{C}} \\
        &\geq  \frac{\exp\br{C\sum_{t=0}^T\gamma^t \log \frac{\pi^*\br{a^1_t \given s^1_t}}{\pi_I\br{a^1_t \given s^1_t}}}Z\br{\tau^1}^{C}}{\sum_{i=1}^2 \exp\br{C\sum_{t=0}^T\gamma^t \log \frac{\pi^*\br{a^i_t \given s^i_t}}{\pi_I\br{a^i_t \given s^i_t}}}Z\br{\tau^{\textcolor{burgundy}{1}}}^{C}} \\
        &=  \frac{\exp\br{C\sum_{t=0}^T\gamma^t \log \frac{\pi^*\br{a^1_t \given s^1_t}}{\pi_I\br{a^1_t \given s^1_t}}}}{\sum_{i=1}^2 \exp\br{C\sum_{t=0}^T\gamma^t \log \frac{\pi^*\br{a^i_t \given s^i_t}}{\pi_I\br{a^i_t \given s^i_t}}}} \;.
    \end{split}
\end{equation}
By our definition on the quality evaluation $e(\tau)$, a better trajectory $\tau$ comes with a higher $e(\tau)$.
Hence $\exp\br{e\br{\tau^1}} / \br{\sum_{i=1}^2 \exp\br{e\br{\tau^i}}} \geq \exp\br{e\br{\tau^2}} / \br{\sum_{i=1}^2 \exp\br{e\br{\tau^i}}}$.
In other words, among $\cbr{\tau^1, \tau^2}$, $\tau^1$ should have the highest chance of being ranked top under the Bradley-Terry preference model \eqref{eq:pl_original} induced by the true reward $r(s,a)$.
Thus we conclude that $\Pr\br{\ord \given \pi^*, \cbr{e\br{\tau^k}}_{k=1}^2} = \max \Pr\br{\cdot \given \pi^*, \cbr{e\br{\tau^k}}_{k=1}^2}$, \ie, in the MDP $\gM$ (or $\gM'$) with the addition of $(\pi_I, C)$ and conditioning on $\pi^*$, $\ord$ should be the most probable ordering under the Bradley-Terry model \eqref{eq:pl_original}.
Thus, in order to approximate $\pi^*$, our parametrized policy $\pi_\theta$ ought to maximize the likelihood of $\ord$ under the \textit{corresponding Bradley-Terry model}  constructed by substituting $\pi^*$ with $\pi_\theta$.
Based on this intuition, we train $\pi_\theta$ by maximizing the lower bound of the Bradley-Terry likelihood of $\ord$ in \eqref{eq:pl_lower_bound}, which leads to the negative-log-likelihood objective for an minimization problem for training $\pi_\theta$ as
\begin{equation}\label{eq:pk_nll}
% \resizebox{.93\linewidth}{!}{%
% $
\begin{aligned}
    \gL_\gamma\br{\theta \given \ord, \cbr{e\br{\tau^k}}_{k=1}^2} 
    =& -\log\sigma\br{ C\sum_{t=0}^T\gamma^t  \sbr{\log\frac{\pi_\theta\br{a^1_t \given s^1_t}}{\pi_I\br{a^1_t \given s^1_t}} - \log\frac{\pi_\theta\br{a^2_t \given s^2_t}}{\pi_I\br{a^2_t \given s^2_t}}}} \\
    % =& -\log\sigma\br{ C\,\E_{t \sim \mathrm{Cat}\br{\cbr{\gamma^t}}} \sbr{\log\frac{\pi_\theta\br{a^1_t \given s^1_t}}{\pi_I\br{a^1_t \given s^1_t}} - \log\frac{\pi_\theta\br{a^2_t \given s^2_t}}{\pi_I\br{a^2_t \given s^2_t}}}} \\
    =& -\log \sigma  \br{ \textcolor{burgundy}{C \times \frac{1 - \gamma^{T+1}}{1-\gamma}} \E_{t \sim \mathrm{Cat}\br{\{\gamma^t\}}} \sbr{\log\frac{\pi_\theta(a^1_t \mid s^1_t)}{\pi_I(a^1_t \mid s^1_t)} - \log\frac{\pi_\theta(a^2_t \mid s^2_t)}{\pi_I(a^2_t \mid s^2_t)}} }\\
    =& -\log \sigma  \br{\textcolor{burgundy}{ \bigg(C \times \frac{1 - \gamma^{T+1}}{1-\gamma}\bigg)} \E_{t \sim \mathrm{Cat}\br{\{\gamma^t\}}} \sbr{\log\frac{\pi_\theta(a^1_t \mid s^1_t)}{\pi_I(a^1_t \mid s^1_t)} - \log\frac{\pi_\theta(a^2_t \mid s^2_t)}{\pi_I(a^2_t \mid s^2_t)}} } \\
    =& -\log \sigma  \br{ \textcolor{burgundy}{C}\, \E_{t \sim \mathrm{Cat}\br{\{\gamma^t\}}} \sbr{\log\frac{\pi_\theta(a^1_t \mid s^1_t)}{\pi_I(a^1_t \mid s^1_t)} - \log\frac{\pi_\theta(a^2_t \mid s^2_t)}{\pi_I(a^2_t \mid s^2_t)}} } \\
    \text{with} \quad \textcolor{burgundy}{C} \leftarrow & \; C \times \frac{1 - \gamma^{T+1}}{1-\gamma} \;,
\end{aligned}    
% $%
% }    
\end{equation}
where $\mathrm{Cat}\br{\cbr{\gamma^t}}$ denotes the categorical distribution on $\cbr{0,\ldots,T}$ with the probability vector $\{\gamma^t/ \sum_{t'} \gamma^{t'}\}_{t=0}^T$; and $C$ is overloaded to absorb the normalization constant, which is legitimated given that $C$ itself is a hyperparameter and so does $C$ times the normalization constant.

\subsection{Proofs}\label{sec:proofs}

\subsubsection{Derivation of the Gradient in \eqref{eq:derivative_l_gamma}} \label{sec:derive_loss_gradient}

Here we derive the gradient of $\gL_\gamma\br{\theta \given \ord, \cbr{e\br{\tau^k}}_{k=1}^2}$ in \eqref{eq:pk_nll} with respect to $\theta$, which is presented in \Secref{sec:interpretation}.

Since \eqref{eq:pk_nll} is an objective for a minimization problem, the gradient update direction is
$
-\nabla_\theta \gL_\gamma\br{\theta \given \ord, \cbr{e\br{\tau^k}}_{k=1}^2} = \nabla_\theta \br{-\gL_\gamma\br{\theta \given \ord, \cbr{e\br{\tau^k}}_{k=1}^2}}
$.
The gradient can be derived by chain rule as follows.
For notation simplicity, we denote $\widetilde e(\tau^k) \triangleq  C\sum_{t=0}^T\gamma^t \log \frac{\pi_\theta\br{a^k_t \given s^k_t}}{\pi_I\br{a^k_t \given s^k_t}}$.
We have
\begin{equation*}
    \begin{split}
        -\gL_\gamma\br{\theta \given \ord, \cbr{e\br{\tau^k}}_{k=1}^2} &= -\log\br{1+\exp\br{\widetilde e\br{\tau^2}-\widetilde e\br{\tau^1}}} 
         \end{split}
\end{equation*}
\begin{equation*}
    \begin{split}
        \frac{\partial \br{-\gL_\gamma\br{\theta \given \ord, \cbr{e\br{\tau^k}}_{k=1}^2}}}{\partial \br{\widetilde e\br{\tau^2}-\widetilde e\br{\tau^1}}} &= -\frac{\exp\br{\widetilde e\br{\tau^2}}}{\exp\br{\widetilde e\br{\tau^1}} + \exp\br{\widetilde e\br{\tau^2}}} 
            \end{split}
\end{equation*}
\begin{equation*}
    \begin{split}
        \forall\, k = 1,2,\quad \frac{\partial\, \widetilde e\br{\tau^k}}{\partial\, \theta } &= C\sum_{t=0}^T\gamma^t \nabla_\theta \log \frac{\pi_\theta\br{a^k_t \given s^k_t}}{\pi_I\br{a^k_t \given s^k_t}} = C\sum_{t=0}^T\gamma^t   \frac{\pi_I\br{a^k_t \given s^k_t}}{\pi_\theta\br{a^k_t \given s^k_t}} \nabla_\theta \pi_\theta\br{a^k_t \given s^k_t}\\
        &= C\sum_{t=0}^T\gamma^t   \pi_I\br{a^k_t \given s^k_t} \nabla_\theta \log \pi_\theta\br{a^k_t \given s^k_t}     
        \end{split}
\end{equation*}
\begin{equation*}
%     \resizebox{.93\linewidth}{!}{%
% $
\begin{aligned}
        \frac{\partial \br{-\gL_\gamma\br{\theta \given \ord, \cbr{e\br{\tau^k}}_{k=1}^2}}}{\partial\, \theta} &= \frac{\partial \br{-\gL_\gamma\br{\theta \given \ord, \cbr{e\br{\tau^k}}_{k=1}^2}}}{\partial \br{\widetilde e\br{\tau^2}-\widetilde e\br{\tau^1}}} \br{\frac{\partial\, \widetilde e\br{\tau^2}}{\partial\, \theta } - \frac{\partial\, \widetilde e\br{\tau^1}}{\partial \,\theta }} \\
        & = -\frac{\exp\br{\widetilde e\br{\tau^2}}}{\exp\br{\widetilde e\br{\tau^1}} + \exp\br{\widetilde e\br{\tau^2}}} \br{\frac{\partial \,\widetilde e\br{\tau^2}}{\partial\, \theta } - \frac{\partial\, \widetilde e\br{\tau^1}}{\partial\, \theta }}\\
        & = \frac{\exp\br{\widetilde e\br{\tau^2}}}{\exp\br{\widetilde e\br{\tau^1}} + \exp\br{\widetilde e\br{\tau^2}}} \br{\frac{\partial \,\widetilde e\br{\tau^1}}{\partial\, \theta } - \frac{\partial\, \widetilde e\br{\tau^2}}{\partial \,\theta }} \\
        &= \frac{\exp\br{\widetilde e\br{\tau^2} - \widetilde e\br{\tau^1}}}{1+\exp\br{\widetilde e\br{\tau^2} - \widetilde e\br{\tau^1}}} \times C \times \sum_{t=0}^T\gamma^t     \bigg(
        \pi_I\br{a^1_t \given s^1_t} \nabla_\theta \log \pi_\theta\br{a^1_t \given s^1_t}  \\
        &\qquad\qquad\qquad\qquad\qquad\qquad\qquad\qquad\qquad\quad  - \pi_I\br{a^2_t \given s^2_t} \nabla_\theta \log \pi_\theta\br{a^2_t \given s^2_t}
        \bigg).
    \end{aligned}    
% $%
% }
\end{equation*}

\subsubsection{Proof of Theorem~\ref{theorem:order_partition}}   \label{sec:proof_order_partitions}

As a reminder, in \cref{theorem:order_partition} we consider a more general case where we are given a finite number $K$ of trajectories whose preference ordering is assume to be $\tau^1 \succ \tau^2 \succ \cdots \succ \tau^K$.
Each trajectory $\tau^k$ takes the form $\tau^k = (s^k_0, a^k_0, s^k_1, a^k_1, \ldots, s^k_T)$.

\myparagraph{A Simplified Case without Reward Shaping.}\label{sec:proof_simplified}

To gain some intuitions, we first present a simplified setting where the distribution $\pi_I$ is deterministic on the given samples, \ie, $\pi_I\br{a^i_t\given s^i_t} = \delta\br{a^i_t \given s^i_t}$.
In this scenario, Theorem~\ref{theorem:order_partition} can be proved without using the reward-shaping argument.

We now state and proof this special case of Theorem~\ref{theorem:order_partition}.

\begin{theorem}[A special case of Theorem~\ref{theorem:order_partition}]
If the sampling distribution $\pi_I\br{a^i_t\given s^i_t} = \delta\br{a^i_t \given s^i_t}$, then
     the original reward function $r(s,a)$ satisfies $Z(\tau^1)\geq Z(\tau^2)\geq \cdots \geq Z(\tau^K)$\,.
\end{theorem}

\begin{proof}
    Our target is $\forall\, k \in \cbr{1,\ldots, K}, i \in \cbr{k,\ldots, K}$,
    \begin{equation*}
        \begin{split}
            Z(\tau^k) \geq Z(\tau^i) &\iff \log Z(\tau^k)\geq \log Z(\tau^i) \iff \sum_{t=0}^T \gamma^t \log Z\br{s^k_t} \geq \sum_{t=0}^T \gamma^t \log Z\br{s^i_t} \\
            & \iff \sum_{t=0}^T \gamma^t \br{\log Z\br{s^k_t} -\log Z\br{s^i_t}} \geq 0.
        \end{split}
    \end{equation*}

In the special case of $\pi_I\br{a^i_t\given s^i_t} = \delta\br{a^i_t \given s^i_t}$, with the \textit{original} reward function $r(s,a)$, we have
\begin{equation*}
    \begin{split}
        Z\br{s^i_t} = \int_\sA \exp\br{\frac{1}{C} r\br{s^i_t, a}} \pi_I\br{a\given s^i_t} \intd a &= \exp\br{\frac{1}{C} r\br{s^i_t, a^i_t}} 
        \implies \log Z\br{s^i_t} = \frac{1}{C} r\br{s^i_t, a^i_t} \\
       \implies \log Z\br{s^k_t} - \log Z\br{s^i_t}  &= \frac{1}{C} \br{r\br{s^k_t, a^k_t} - r\br{s^i_t, a^i_t}} \\
       \implies \sum_{t=0}^T \gamma^t \br{\log Z\br{s^k_t} - \log Z\br{s^i_t}}  &= \frac{1}{C} \sum_{t=0}^T \gamma^t \br{r\br{s^k_t, a^k_t} - r\br{s^i_t, a^i_t}} \,.
    \end{split}
\end{equation*}
Since $\tau^k \succ \tau^i \iff e\br{\tau^k} > e\br{\tau^i}$, plugging in the definition of $e\br{\tau}$, we get,
\begin{equation*}
    \begin{split}
      e\br{\tau^k} > e\br{\tau^i} \iff   \sum_{t=0}^T \gamma^t r(s^k_t,a^k_t) \geq \sum_{t=0}^T \gamma^t r(s^i_t,a^i_t) &\iff \sum_{t=0}^T \gamma^t \br{r(s^k_t,a^k_t)  - r(s^i_t,a^i_t)} \geq 0 \iff \\
         \sum_{t=0}^T \gamma^t \log Z\br{s^k_t} \geq \sum_{t=0}^T \gamma^t \log Z\br{s^i_t} &\iff \log Z\br{\tau^k} \geq \log Z\br{\tau^i} \iff Z\br{\tau^k} \geq  Z\br{\tau^i} \,.
    \end{split}
\end{equation*}
Hence the original reward function $r(s,a)$, \textit{without shaping}, satisfies the ordering $Z\br{\tau^k} \geq  Z\br{\tau^i}$.
Notice that all the above steps are ``$\iff$'' and recall our assumption that $\tau^1 \succ \tau^2 \succ \cdots \succ \tau^K \implies \tau^k \succ \tau^i \iff k \leq i$.
It is clear that such an ordering is transitive, in a sense that, if $\tau^k \succ \tau^i \succ \tau^j$, then
\begin{equation*}
\left.\begin{aligned}
    \tau^k \succ \tau^i \implies k\leq i \\
    \tau^i \succ \tau^j \implies i\leq j \\
\end{aligned}\right\} % \hspace{0.5cm}
\implies k \leq j \implies Z\br{\tau^k} \geq  Z\br{\tau^j} \,.
\end{equation*}
Since $k$ is arbitrary, we conclude that $Z(\tau^1) \geq Z(\tau^2) \geq \cdots \geq Z(\tau^K)$, as desired.
\end{proof}

\myparagraph{The General Case}

We repead Theorem~\ref{theorem:order_partition} here for better readability.

\ThmOrderPartitionAppendix*

As discussed in Remark~\ref{remark:value_C}, for the value of $C$, we technically requires that $\forall\, (s,a) \in \sS \times \sA, r(s,a) / C \leq \mathrm{const} \approx 1.79$.
Hence, under Assumption~\ref{assump:bounded_rew}, $C\geq 0.56$ will suffice.
This provides some information on the setting of hyperparameter $C$.

\begin{proof}

Under the shaped reward $r'(s,a) = r(s,a) + \Phi(s)$, $Z\br{s^k_t}$ takes the form
\begin{equation*}
    \begin{split}
        Z\br{s^k_t} &= \int_\sA \exp\br{\frac{1}{C} r\br{s^k_t,a} + \frac{1}{C} \Phi\br{s^k_t}} \pi_I\br{a \given s^k_t}  \intd a \\
        &= \exp\br{\frac{1}{C} \Phi\br{s^k_t}}  \int_\sA \exp\br{\frac{1}{C} r\br{s^k_t,a} } \pi_I\br{a \given s^k_t}  \intd a \\
        &= \exp\br{\frac{1}{C}  \Phi\br{s^k_t}} \E_{a\sim \pi_I\br{\cdot \given s^k_t}}\sbr{
        \exp\br{\frac{1}{C} r\br{s^k_t,a} } 
}.
    \end{split}
\end{equation*}
Taking $\log$ on both sides of the equation, by Jensen's inequality, we have
\begin{equation*}
    \begin{split}
        \log Z\br{s^k_t} &= \frac{1}{C}  \Phi\br{s^k_t} + \log \E_{a\sim \pi_I\br{\cdot \given s^k_t}}\sbr{
        \exp\br{\frac{1}{C} r\br{s^k_t,a} }}  \geq \frac{1}{C}  \Phi\br{s^k_t} +  \E_{a\sim \pi_I \br{\cdot \given s^k_t}}\sbr{
        \frac{1}{C} r\br{s^k_t,a} }.
    \end{split}
\end{equation*}
On the other hand, we also have
\begin{equation*}
    \begin{split}
        Z\br{s^i_t} &= \exp\br{\frac{1}{C} \Phi\br{s^i_t}}  \int_\sA \exp\br{\frac{1}{C} r\br{s^i_t,a} } \pi_I\br{a \given s^i_t}  \intd a = \exp\br{\frac{1}{C} \Phi\br{s^i_t}} \E_{a\sim \pi_I\br{\cdot \given s^i_t}} \sbr{\exp\br{\frac{1}{C} r\br{s^i_t,a} }}   ,
    \end{split}
\end{equation*}
Taking $\log$ again on both sides of the equations, we have
\begin{equation*}
    \begin{split}
        \log Z\br{s^i_t} & = \frac{1}{C} \Phi\br{s^i_t} + \log\br{\E_{a\sim \pi_I\br{\cdot \given s^i_t}}\sbr{\exp\br{\frac{1}{C} r\br{s^i_t,a} }}}  \\
        &\leq \frac{1}{C} \Phi\br{s^i_t} + \E_{a\sim \pi_I\br{\cdot \given s^i_t}}\sbr{\exp\br{\frac{1}{C} r\br{s^i_t,a} }} - 1 \\
        &\leq \frac{1}{C} \Phi\br{s^i_t} + \E_{a\sim \pi_I\br{\cdot \given s^i_t}}\sbr{1+\frac{1}{C} r\br{s^i_t,a} + \frac{1}{C^2} r^2\br{s^i_t,a} } - 1 \\
        &\leq \frac{1}{C} \Phi\br{s^i_t} + \E_{a\sim \pi_I\br{\cdot \given s^i_t}}\sbr{\frac{1}{C} r\br{s^i_t,a} + \frac{1}{C^2} r\br{s^i_t,a} } \\
        &=  \frac{1}{C} \Phi\br{s^i_t} + \frac{C+1}{C^2} \E_{a\sim \pi_I\br{\cdot \given s^i_t}}\sbr{ r\br{s^i_t,a}  }
    \end{split}
\end{equation*}
where the first inequality is because $\forall\, x > 0, \log x \leq x - 1$; the second inequality is because $e^x \leq 1 + x + x^2, \forall\, x < \mathrm{const} \approx 1.79$; the third inequality is because Assumption~\ref{assump:bounded_rew}, \ie, $\forall\, (s,a) \in \sS \times \sA, \, 0\leq r(s,a) \leq 1$.

Combining the above analysis, we have
\begin{equation*}
    \begin{split}
        \log Z\br{s^k_t} - \log Z\br{s^i_t} & \geq \frac{1}{C}  \Phi\br{s^k_t} +  \E_{a\sim \pi_I\br{\cdot \given s^k_t}}\sbr{
        \frac{1}{C} r\br{s^k_t,a} } - \frac{1}{C} \Phi\br{s^i_t} - \frac{C+1}{C^2} \E_{a\sim \pi_I\br{\cdot \given s^i_t}}\sbr{ r\br{s^i_t,a}  } \\
        & \geq \frac{1}{C}  \br{\Phi\br{s^k_t}   -  \Phi\br{s^i_t}} - \frac{C+1}{C^2} \\
    \end{split}
\end{equation*}
where the second inequality is again due to Assumption~\ref{assump:bounded_rew}, \ie, $0\leq r(s,a) \leq 1$.

Summing over $t$, we have
\begin{equation*}
    \begin{split}
        \sum_{t=0}^T \gamma^t \br{\log Z\br{s^k_t} - \log Z\br{s^i_t}} \geq \frac{1}{C} \sum_{t=0}^T \gamma^t
        \br{ \Phi\br{s^k_t} - \Phi\br{s^i_t}}  -  \frac{C+1}{C^2} \frac{1-\gamma^{T+1}}{1-\gamma}\, {\geq}_{?}\, 0 \,,
    \end{split}
\end{equation*}
where the desired final inequality of $\,\geq 0\,$ holds if 
\begin{equation*}
    \sum_{t=0}^T \gamma^t \br{\Phi\br{s^k_t} - \Phi\br{s^i_t}} \geq \frac{C+1}{C} \frac{1-\gamma^{T+1}}{1-\gamma} \,,
\end{equation*}
where $\frac{C+1}{C} \frac{1-\gamma^{T+1}}{1-\gamma} < \infty$ is finite.
Therefore, there exists a finite shaping function $\Phi(s)$ satisfying the above constraint, which can restore the order of $Z(\tau^k)$ and $Z(\tau^i)$ to be $Z(\tau^k) \geq Z(\tau^i)$.

Furthermore, this restoration is transitive in the sense that, for $\tau^k \succ \tau^i \succ \tau^j$ and the corresponding $Z(\tau^k)$, $Z(\tau^i)$, and $Z(\tau^j)$, if
\begin{equation*}
    \begin{split}\textstyle
        \sum_{t=0}^T \gamma^t \br{\Phi\br{s^k_t} - \Phi\br{s^i_t}} &\geq \frac{C+1}{C} \frac{1-\gamma^{T+1}}{1-\gamma}   \implies Z(\tau^k) \geq Z(\tau^i) \\
        \textstyle
        \sum_{t=0}^T \gamma^t \br{\Phi\br{s^i_t} - \Phi\br{s^j_t}} &\geq \frac{C+1}{C} \frac{1-\gamma^{T+1}}{1-\gamma}  \implies Z(\tau^i) \geq Z(\tau^j),
    \end{split}
\end{equation*}
then $\log Z(\tau^k) - \log Z(\tau^j) \geq 0 \iff Z(\tau^k) \geq Z(\tau^j)$, because,
\begin{equation*}
    \begin{split}
        \log Z(\tau^k) - \log Z(\tau^j) &= \sum_{t=0}^T \gamma^t \br{\log Z\br{s^k_t} - \log Z\br{s^j_t}} \\
        &= \sum_{t=0}^T \gamma^t \br{\log Z\br{s^k_t} - \log Z\br{s^i_t} + \log Z\br{s^i_t} - \log Z\br{s^j_t}} \\
        &= \sum_{t=0}^T \gamma^t \br{\log Z\br{s^k_t} - \log Z\br{s^i_t}} + \sum_{t=0}^T \gamma^t\br{\log Z\br{s^i_t} - \log Z\br{s^j_t}} \\
        &\geq \frac{1}{C} \br{\sum_{t=0}^T \gamma^t \br{ \Phi\br{s^k_t} - \Phi\br{s^i_t}}  -  \frac{C+1}{C} \frac{1-\gamma^{T+1}}{1-\gamma}} \\
        & \quad\quad + \frac{1}{C} \br{\sum_{t=0}^T \gamma^t \br{ \Phi\br{s^i_t} - \Phi\br{s^j_t}}  -  \frac{C+1}{C} \frac{1-\gamma^{T+1}}{1-\gamma}} \geq 0\,,
    \end{split}
\end{equation*}
since by  Assumption~\ref{assump:bounded_rew}, $C\geq 0.56$ is positive. 

It follows that for $K$ trajectories $\tau^1 \succ \tau^2 \succ \cdots \succ \tau^K$, we can restore the order of $Z\br{\tau^k}$'s by \textbf{at most} $(K-1)$ requirements on the reward-shaping function $\Phi(s)$, taking the form,
\begin{equation*}
    \begin{split}
        \sum_{t=0}^T \gamma^t \br{\Phi\br{s^1_t} - \Phi\br{s^2_t}} &\geq  \frac{C+1}{C} \frac{1-\gamma^{T+1}}{1-\gamma}, \\[-0.5ex]
                        & \vdotswithin{ \dots } \\[-0.5ex] % change the 120 mu value for appropriate position
  \sum_{t=0}^T \gamma^t \br{\Phi\br{s^{K-1}_t} - \Phi\br{s^K_t}} &\geq \frac{C+1}{C} \frac{1-\gamma^{T+1}}{1-\gamma} \,.
    \end{split}
\end{equation*}
Since each of these $(K-1)$ requirements only specify a finite lower bound on the discounted sum of the difference of the reward-shaping function $\Phi(\cdot)$ on two trajectories, it is clear that there exists a finite reward-shaping function $\Phi(s)$ satisfying these   requirements.
Hence, there exists a shaped reward function $r' \in [r], r'(s,a) = r(s,a) + \Phi(s)$, such that $Z(\tau^1) \geq \cdots \geq Z(\tau^K)$ under $r'$.
\end{proof}

% <---------------------------------------------------------------> %
\section{The Smaller $\gamma$, The Better?}\label{sec:smaller_gamma_not_better}

Though it would be great if ``the smaller $\gamma$, the better result'', this is unfortunately not true. 
In the multiple prompt experiment, as shown in Fig.~\ref{fig:multiple_gamma}, $\gamma=0.95$ is slightly better than $\gamma=0.9$ towards the end of training. 

% ----------------------------------------------------
% Aestheric
\begin{figure*}[tb]
    \centering
    % \begin{minipage}{0.49\textwidth}
        \centering
\includegraphics[width=0.4\textwidth]{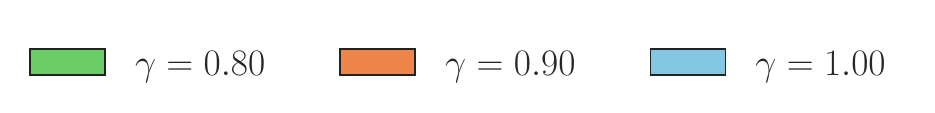}
\vspace{-.5em}
\\
     \begin{subfigure}[b]{0.5\textwidth}
         \centering
         \includegraphics[width=\textwidth]{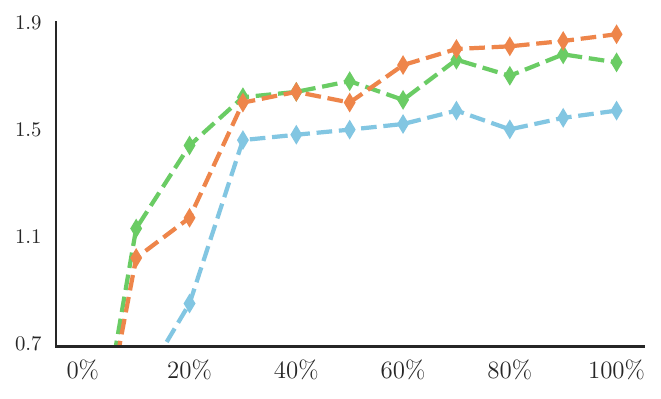}
         \captionsetup{font=footnotesize}
         \vspace{-6mm}
         % \caption{\footnotesize{Single (ImageReward)}} % Sequence Num.
         % \label{fig:line_varying_gamma_single_with_08}
     \end{subfigure}
     \vspace{-2mm}
     \captionsetup{font=small}
        \caption{ 
        \small
        ImageReward over the training process for the single prompt experiment on the color-domain prompt ``A green colored rabbit.'', with $\gamma \in \{0.8, 0.9, 1.0\}$.
        As in Fig.~\ref{fig:single_color_gamma}, $x$-axis represents $t\%$ of the training process and all lines start from $-0.02$ at $0\%$, the value of ``Orig.'' in Fig.~\ref{fig:single_imagerew_Color}.
        }
        \label{fig:line_varying_gamma_single_with_08}
    % \end{minipage}
    % \rulesep
\end{figure*}

As another verification, we re-run our single prompt experiment (“A green colored rabbit.”) under $\gamma=0.8$. Fig.~\ref{fig:line_varying_gamma_single_with_08} compares its performance with $\gamma\in\{0.9, 1.0\}$ at each decile of the training process.
From Fig.~\ref{fig:line_varying_gamma_single_with_08}, we see that $\gamma=0.8$ is again better than the classical setting of $\gamma=1.0$ and indeed trains faster than $\gamma=0.9$ in the first $20\%$ of the training process. However, in the second half of training, $\gamma=0.8$ is less stable and its performance is inferior to $\gamma=0.9$. 

Recall that during training, the smaller $\gamma$, the more emphasis is on the initial steps of the reverse chain. As shown in Fig.~\ref{fig:line_varying_gamma_single_with_08}, a too-small $\gamma$ may thus have a stronger tendency of overfitting, leading to a more varying training process and inferior final result. Further, during training, a too-small $\gamma$ may pay too-few attention to the later steps of the reverse chain that generate image details, resulting in less preferable image generations. 

From Figs.~\ref{fig:multiple_gamma} and \ref{fig:line_varying_gamma_single_with_08}, we conclude that while a \textit{sensible} incorporation of $\gamma < 1$ can outperform the classical setting of $\gamma = 1$, final performance is not monotone with $\gamma$. The optimal $\gamma$ value can be task specific. In our experiments, we find that $\gamma=0.9$ or $0.95$ can be a good starting point.

% <---------------------------------------------------------------> %
\section{Discussion on Our Method's Applicability to Real Human Preference}

In the experiments (\cref{sec:exp}), we use human-preference scorers for quantitatively verifying our method’s ability to satisfy (human) preferences, which also facilitates reproducibility. 
Human-preference scorers are also essential for further studies of our proposed method in \cref{sec:exp_abla} \textbf{(b)} and \textbf{(c)}.
Apart from the numeric scores, we present image samples and conduct a human evaluation (\cref{sec:exp_abla} \textbf{(d)}) to verify our method’s ability in generating (human) preferable images.

As presented in \cref{sec:method}, our method does not make assumptions about the preference source. Thus, a reward function is not an intrinsic requirement of our method. Being agnostic to the preference source, our method is readily applicable to (real) human preferences as well. 

Adapting the classical off-policy RLHF paradigm in the literature \citep[\eg,][]{ziegler2019fine,stiennon2020learning,menick2022teaching,bai2022training}, a simple workflow of applying our method to real human preferences iterate on:
\begin{enumerate}[topsep=0em,itemsep=0em]
    \item Generate trajectories from the latest policy, gather human preferences on the corresponding images, and store the quantities required in \cref{sec:method}; %
    \item Continue training the T2I by our proposed loss for a chosen number of steps, utilizing the newly collected human data. %
\end{enumerate}
Given its similarity with the cited RLHF literature, we believe that this workflow is indeed practical for human-in-the-loop.

% <---------------------------------------------------------------> %
\section{More Related Works} \label{sec:more_related_work}

\myparagraph{Dense \textit{v.s.} Sparse Training Guidance for Sequential Generative Models.}
% Sequential Generative Models: text generation models, dialog system, diffusion models
By its sequential generation nature, T2Is are instances of generative models with \textit{sequential} nature, which further includes, \eg, text generation models, \citep{devlin2018bert,bart2019,radford2019language} and dialog systems \citep{chen2017survey,kwan2022survey}.
% Classsically use sparse reward, potentially with per-step KL regularization that are un-related to the task metrics -> bad
Similar to T2I's alignment (\cref{sec:t2ialignment}), a classical guiding signal for training  sequential generative models is the native trajectory-level feedback such as the downstream test metric \citep[\eg,][]{ryang2012framework,ranzato2015sequence,rennie2017self,paulus2017deep,shu2021reward,quark2022,snell2022offline}. 
As discussed in \Secref{sec:intro}, ignoring the sequential-generation nature can incur optimization difficulty and training instability due to the sparse reward issue \citep{sqltext2021,snell2022offline}.
In RL-based methods for training text generation models, in particular, it has become popular to incorporate into the training objective a per-step KL penalty towards the uniform distribution \citep{sqltext2021,rlprompt2022}, the initial pre-trained model \citep{ziegler2019fine,nlpo2022}, the supervised fine-tuned model \citep{offlinerldialog2019,stiennon2020learning,jaques2020human,instructgpt2022}, or some base momentum model \citep{castricato2022robust}, to ``densify'' the sparse reward.  
Although a per-step KL penalty does help the RL-based training, it can be less task-tailored should one 
regularizes the generative models towards those generic distributions, especially regarding the ultimate training goal --- optimizing the desired trajectory-level feedback.
As discussed in \citet{yang2023preferencegrounded} (Appendix F), when combined with the sparse reward issue, such a KL regularization can in fact distract the training of text generation models from improving the received feedback, 
especially for the initial steps of the generation process, which unfortunately will affect all subsequent generation steps.

% Early work: Relatively ideal setting: inverse RL
In some relatively restricted settings, task-specific \textit{dense} rewards have been explored for training text generation models.
With the assumption of abundant expert data for supervised (pre-)training,
\citet{shi2018toward} use inverse RL \citep{russell1998learning} to infer a per-step reward;  \citet{leakgan2018} propose a hierarchical approach; \citet{yang2018unsupervised} learn LM discriminators; while \citet{adversarialranking2017} and \citet{yu2017seqgan} first learn a trajectory-level adversarial reward function similar to a GAN discriminator, before applying the expensive and high-variance Monte-Carlo rollout to simulate per-step rewards.
In the code generation domain, \citet{le2022coderl} use some heuristic values related to the trajectory-level evaluation, without explicitly learning per-step rewards.

% Recently, Grounding Preference: ToD system and our NIPS 2023
Inspired by preference learning in robotics \citep[\eg,][]{christiano2017deep}, methods have  been recently developed to learn a \textit{dense} per-step reward function whose trajectory-level aggregation aligns with the preference ordering among multiple alternative generations. 
These methods have been applied to both sufficient-data and low-data regime, in applications of training task-oriented dialog systems \citep[\eg,][]{caspi2021,fantasticrewards2022} and fine-tuning text-sequence generation models \citep{yang2023preferencegrounded}.

Motivated by this promising direction in prior work and an easier learning problem in RL,
in this paper, we continue the research on \textit{dense} training guidance for sequential generative models, by assuming that the trajectory-level preferences are generated by a latent \textit{dense} reward function.
Through incorporating the key RL ingredient of temporal discounting factor $\gamma$, we break the temporal symmetry in the DPO-style explicit-reward-free alignment loss.
Our training objective naturally suits the T2I generation hierarchy by emphasizing the initial steps of the T2I generation process, which benefits all subsequent generation steps and thereby improves both effectiveness and efficiency of training, as shown in our experiments (\Secref{sec:exp}).
% <----------------------------- > %

% <----------------------------- > %
\myparagraph{Characterizing the (Latent) Preference Generation Distribution.} 
Since preference comparisons are typically performed only among the fully-generated trajectories, % under the RL framework, 
aligning trajectory generation with preference mostly requires characterizing how preference is originated from per-step rewards, as part of the preference model's assumptions.
% classical IL: sum
In the imitation learning literature, preference model is classical chosen to be the Boltzmann  distribution over the \textit{undiscounted} sum of per-step rewards \citep{christiano2017deep,trex2019,drex2020}.
% Pieter's ICLR 2023
Several advances have been made on the characterization of the preference model, especially for accommodating the specific nature of concrete tasks. 
In robotics, \citet{kim2023preference} proposes to model the (negative) potentials of the Boltzmann  distribution by \textit{learning} a \textit{weighted-sum} to aggregate the per-step rewards over the entire trajectory. 
% Regret/advantage style method: Knox XXX, cpl
Motivated by the simulated robotics benchmark of location/goal reaching, an alternative formulation has been developed that models the potentials of the preference Boltzmann distribution  by the optimal advantage function or regret \citep{Knox2022ModelsOH,Knox2023LearningOA,hejna2023contrastive}.
Of a special note, though the objective in CPL (Eq. (5) in \citet{hejna2023contrastive}) looks similar to our \eqref{eq:pk_nll_main}, \textbf{in experiments, CPL actually sets $\gamma=1$} (Page 29 Table 6 of \citet{hejna2023contrastive}), making \textbf{their actual loss indeed being the ``trajectory-level reward'' variant} discussed in \Secref{sec:connect_with_dpo}.
% Our NIPS 2023
Apart from robotic tasks, in text-sequence generation, \citet{yang2023preferencegrounded}  take into account the variable-length nature of the tasks, \eg, text summarization, and propose to incorporate inductive bias into modelling the potentials of the preference Boltzmann  distribution, through a \textit{task-specific selection} on how the per-step rewards should be aggregated over the trajectory.
In this paper, we are among the earliest works to consider the characterization of the
preference model in T2I's alignment.
By incorporating temporal discounting ($\gamma < 1$) into the preference Boltzmann  distribution, we cater for the generation hierarchy of the diffusion and T2I reverse chain \citep{ho2020denoising,wang2023diffusion}.
Through experiment results and further study (\Secref{sec:exp}, especially \Secref{sec:exp_abla} \textbf{(a)} \& \textbf{(b)}), we demonstrate that temporal discounting can be useful for effective and efficient T2I preference alignment.
% <----------------------------- > %

% <----------------------------- > %
\myparagraph{Learning-from-preference in Related Fields.} 
% robotics: XXX
As discussed before, learning-from-preference has been a longstanding problem in robotics/control tasks \citep{akrour2011preference,akrour2012april,furnkranz2012preference} and has recently been scaled up to train deep-neural-network-based policies \citep{christiano2017deep,ibarz2018reward,biyik2019green,trex2019,drex2020,Lee2021PEBBLEFI,shin2021offline,hejna2023few,hejna2023inverse}.
These methods typically start by learning a reward function from data of pairwise comparisons or rankings, before using RL algorithms for policy optimization.
Motivated by its success in robotics, learning-from-preference is adopted in the field of natural language generation to improve text summarization \citep{ziegler2019fine,stiennon2020learning} and has become a \textit{de-facto} ingredient in the recent trend of LLMs and conversational agent \citep[\eg,][]{instructgpt2022,bai2022training,menick2022teaching,gpt42023}.
% Not just in fine-tuning stage, but also on pre-training stage
Apart from the fine-tuning stage, learning-from-preference has also been applied to the pre-training stage, though only use the  sparse trajectory-level evaluation \citep{korbak2023pretraining}.
To alleviate the modelling and compute complexity of an explicit reward model, 
following the  maximum-entropy principle in control and RL \citep{ziebart2008maximum,ziebart2010modeling,Finn2016ACB},
DPO-style objectives \citep[\eg,][]{dpo2023, tunstall2023zephyr, azar2023general, yuan2023rrhf, zhao2023slic,ethayarajh2023halos} directly train the LLMs to align with the preference data, without explicitly learning a deep neural network for the reward function.
In this paper, we are among the earliest to study the extension of learning-from-preference into T2I's preference alignment.
By taking a dense-reward perspective, we contribute to the DPO-style explicit-reward-free methods by developing a novel objective that emphasizes the initial part of the sequential generation process, which better accommodates the generation hierarchy of diffusion models and T2Is \citep{ho2020denoising,wang2023diffusion}. 
We validate our perspective through experiments in \Secref{sec:exp}.
% <----------------------------- > %

% <---------------------------------------------------------------> %
\section{Experiment Details} \label{sec:exp_details}

We note that in mini-batch training of \eqref{eq:pk_nll_main}, for the the sampled mini-batch $\gB \triangleq \{(\tau_i^1, \tau_i^2)_{\vc_i}\}_i$, each trajectories in the trajectory tuple $(\tau_i^1, \tau_i^2)$  corresponds to the same text prompt $\vc_i$, which makes the preference comparison between trajectories valid.
Different trajectory tuples may correspond to different text prompts in the multiple prompt experiments. 

We implement our method based on the \href{https://github.com/google-research/google-research/tree/master/dpok}{source code of DPOK} \citep{fan2023dpok}, and inherit as many of their designs and hyperparameter settings as possible, \eg, the specific U-net layers to add LoRA.
In the notation of \Secref{sec:method}, the LoRA parameters are our trainable policy parameter $\theta$. 
To further save GPU memory, the entire training process is conducted under \texttt{bfloat16} precision.
For training stability, we are motivated by PPO \citep{schulman2017proximal} and DPOK to clip all log density ratios $\log\frac{\pi_\theta\br{a_t \given s_t}}{\pi_I\br{a_t \given s_t}}$ to be within $[-\epsilon, \epsilon]$, since $\log\br{1\pm \epsilon} \approx \pm \epsilon$.
Without further tuning, we set $\epsilon=\mathrm{1e-4}$ in single prompt experiments as in DPOK, and $\epsilon=5e-4$ in multiple prompt experiments.

Below we discuss the hyperparameter settings specific to our single and multiple prompt experiments.
\begin{multicols}{2}% 2-column layout
  \begin{minipage}{0.45\textwidth}
    \begin{table}[H]
\captionsetup{font=small}
\caption{
\small Key hyperparameters for T2I (policy) training in the single prompt experiments.
} 
\label{table:exp_single_hyperparam} 
\centering 
\vspace{-.8em}
% \resizebox{\textwidth}{!}{
\begin{tabular}{@{}lll@{}}
\toprule
Hyperparameter              & & Value                                       \\ \midrule
$M_\mathrm{tr}$ & & 10000 \\
$M_\mathrm{col}$ & & 2500 \\
$N_\mathrm{pr}$ & & 1000 \\
$N_{\mathrm{traj}}$ & & 5 \\
$N_\mathrm{step}$ & & 3 \\
$C$ & & 10.0 \\
$\gamma$ & & 0.9 \\
Batch Size & & 4 \\
LoRA Rank & & 4 \\
Optimizer & & AdamW   \\
Learning Rate & & 3e-5 \\
Weight Decay & & 2e-3 \\ 
Gradient Norm Clipping & & 1.0 \\
Learning Rate Scheduler & & Constant \\
Preference Source & & ImageReward \\
\bottomrule
\end{tabular}
% }
\end{table}
  \end{minipage}
  \hfill
\begin{minipage}{0.45\textwidth}
\begin{table}[H]
\captionsetup{font=small}
\caption{
\small Key hyperparameters for T2I (policy) training in the multiple prompt experiments.
} 
\label{table:exp_multi_hyperparam} 
\centering 
\vspace{-.8em}
% \resizebox{.4\textwidth}{!}{
\begin{tabular}{@{}llll@{}}
\toprule
Hyperparameter        &   &     & Value                                       \\ \midrule
$M_\mathrm{tr}$ & & & 40000 \\
$M_\mathrm{col}$ & & & 4000 \\
$N_\mathrm{pr}$ &  & & 2000 \\
$N_{\mathrm{traj}}$ & & & 5 \\
$N_\mathrm{step}$ & & & 1 \\
$C$ & & & 12.5 \\
$\gamma$ & & & 0.9 \\
Batch Size & & & 32 \\
LoRA Rank & & & 32 \\
Optimizer & & & AdamW   \\
Learning Rate & & & 2e-5 \\
Weight Decay & & & 1.5e-3 \\ 
Gradient Norm Clipping & & & 0.05 \\
Learning Rate Scheduler & & & Constant \\
Preference Source & & & HPSv2 \\
\bottomrule
\end{tabular}
% }
\end{table}
\end{minipage}
\end{multicols}

\subsection{Single Prompt Experiments} \label{sec:exp_details_single_prompt}

\cref{table:exp_single_hyperparam} tabulates the key training hyperparameters, where we use the Adam optimizer with decoupled weight decay \citep[AdamW,][]{adamw2017}.

\subsection{Multiple Prompt Experiments} \label{sec:exp_details_multiple_prompt}

We note that we obtained the HPSv2 train prompts by email correspondence with HPSv2's authors.
We produce all results by following the testing principle in the HPSv2 paper and the official GitHub Repository.
\cref{table:exp_multi_hyperparam} tabulates the key training hyperparameters, where we again use the AdamW optimizer.
In the qualitative comparisons (Fig.~\ref{fig:multiple_prompt_generated_images} and \cref{sec:multiple_prompt_more_images}), image samples for the baseline ``Dreamlike Photoreal 2.0'' are directly from the officially released \href{https://huggingface.co/datasets/zhwang/HPDv2/tree/main/benchmark/benchmark_imgs}{HPSv2 benchmark images}.

\subsection{Setups of the Human Evaluation} \label{sec:human_eval_setup}

In our human evaluation (\Secref{sec:exp_abla} \textbf{(d)}), we generally adopt the principle in prior work \citep[\eg,][]{hpsv22023,imagereward2023,wallace2023diffusion} to evaluate the generated images' fidelity  to the text prompt, as well as their overall quality.
We use the same set of baseline methods as in Fig.~\ref{fig:multiple_prompt_generated_images}, since we view this set as both representative and minimal.
In conducting this evaluation, we \textit{randomly sampled} $200$ prompts from the HPSv2 test set.
Note that though we use the same set of baseline methods, the sampled text prompts are \textit{not} necessarily the same as those shown in Fig.~\ref{fig:multiple_prompt_generated_images} and \cref{sec:multiple_prompt_more_images}.
We asked $20$ qualified evaluators for binary comparisons between two images, each from a different model, based on the provided corresponding text prompt.
The method names were anonymized.
The evaluators were asked to read the text prompt and select which one of the two images is better, in terms of both text fidelity and image quality.
To reduce randomness and bias in human judgement, we ensured that all binary comparisons would be evaluated multiple times by the same or a different evaluator.
In Table~\ref{table:multiple_prompt_human}, we report the ``win rate'' of our method, \ie, the percentage of binary comparisons with the stated opponent where the image from our method is preferred.
Note that the ``win rate'' is averaged over all comparisons between the specified two parties.
We leave as future work a more comprehensive and larger scale human evaluation for our method.

% <---------------------------------------------------------------> %
\section{More Generated Images}\label{sec:more_images}

\subsection{More Images from the Single Prompt Experiment} \label{sec:single_prompt_more_images}

% \begin{figure*}[h!] 
% \centering
% % \subfigure[``A green colored rabbit"]
% % {\includegraphics[width=0.85\textwidth]{figures/sup_kl_rabbit.pdf}
% % % \label{fig}
% % } 
% \end{figure*}

\begin{figure}[H]
     % \vspace{-1.0mm}
     \centering
\includegraphics[width=1.\textwidth]{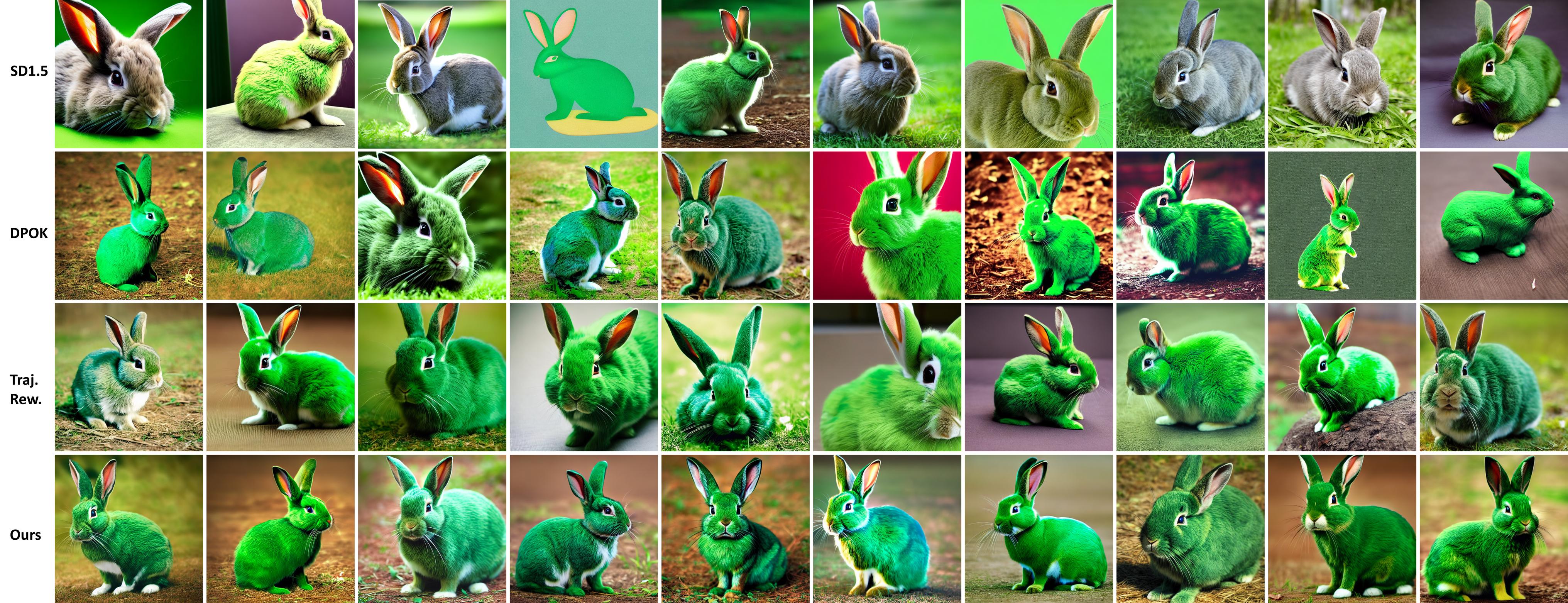}
     \vspace{-1.5em}
     \captionsetup{font=small}
        \caption{ 
        \small
        Single prompt experiment: \textit{randomly sampled} generated images for the prompt ``A green colored rabbit.'', from our method and the baselines in Fig.~\ref{fig:single_prompt_generated_images}.
        ``Traj. Rew.'' denotes the classical DPO-style objective of assuming trajectory-level reward (\Secref{sec:connect_with_dpo}).
        }
        % \label{fig:single_prompt_generated_images}    % no need for label in appendix for now
        % \vspace{-.5em}
\end{figure}

\begin{figure}[H]
     % \vspace{-1.0mm}
     \centering
\includegraphics[width=1.\textwidth]{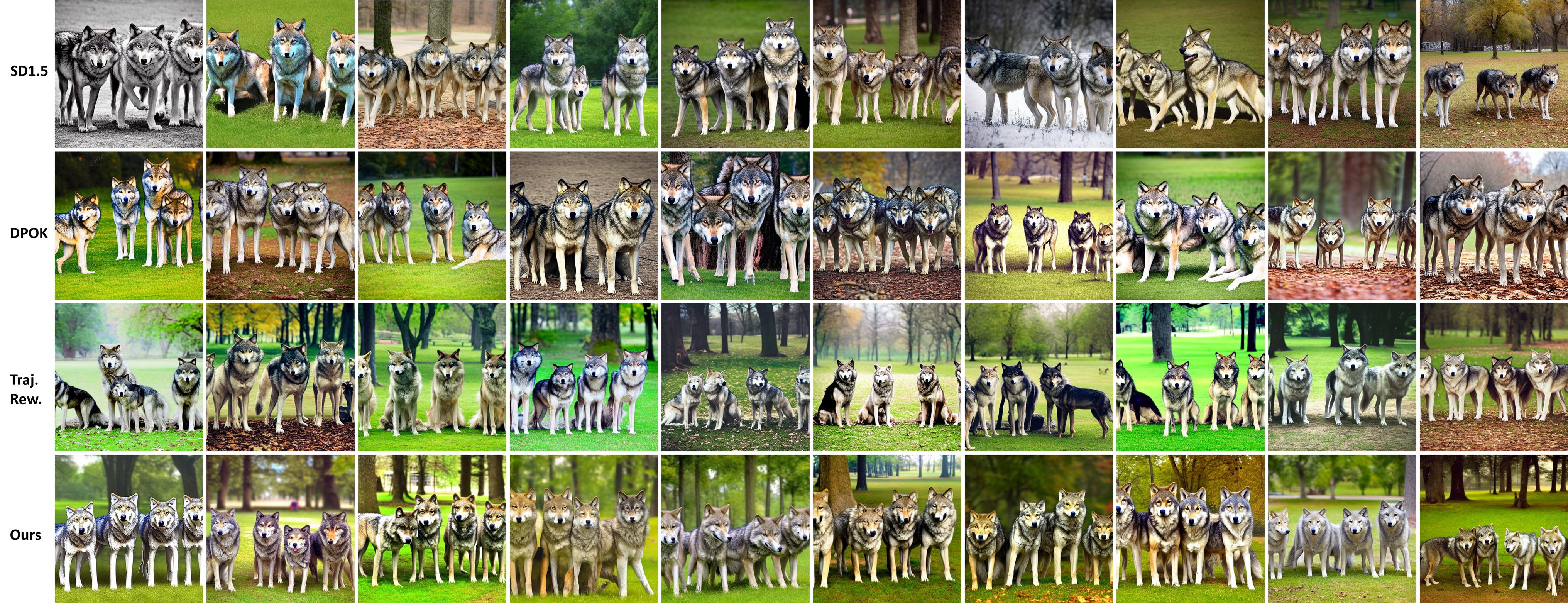}
     \vspace{-1.5em}
     \captionsetup{font=small}
        \caption{ 
        \small
        Single prompt experiment: \textit{randomly sampled} generated images for the prompt ``Four wolves in the park.'', from our method and the baselines in Fig.~\ref{fig:single_prompt_generated_images}.
        ``Traj. Rew.'' denotes the classical DPO-style objective of assuming trajectory-level reward (\Secref{sec:connect_with_dpo}).
        }
        % \label{fig:single_prompt_generated_images}    % no need for label in appendix for now
        % \vspace{-.5em}
\end{figure}
\mbox{}

\begin{figure}[H]
     % \vspace{-1.0mm}
     \centering
\includegraphics[width=1.\textwidth]{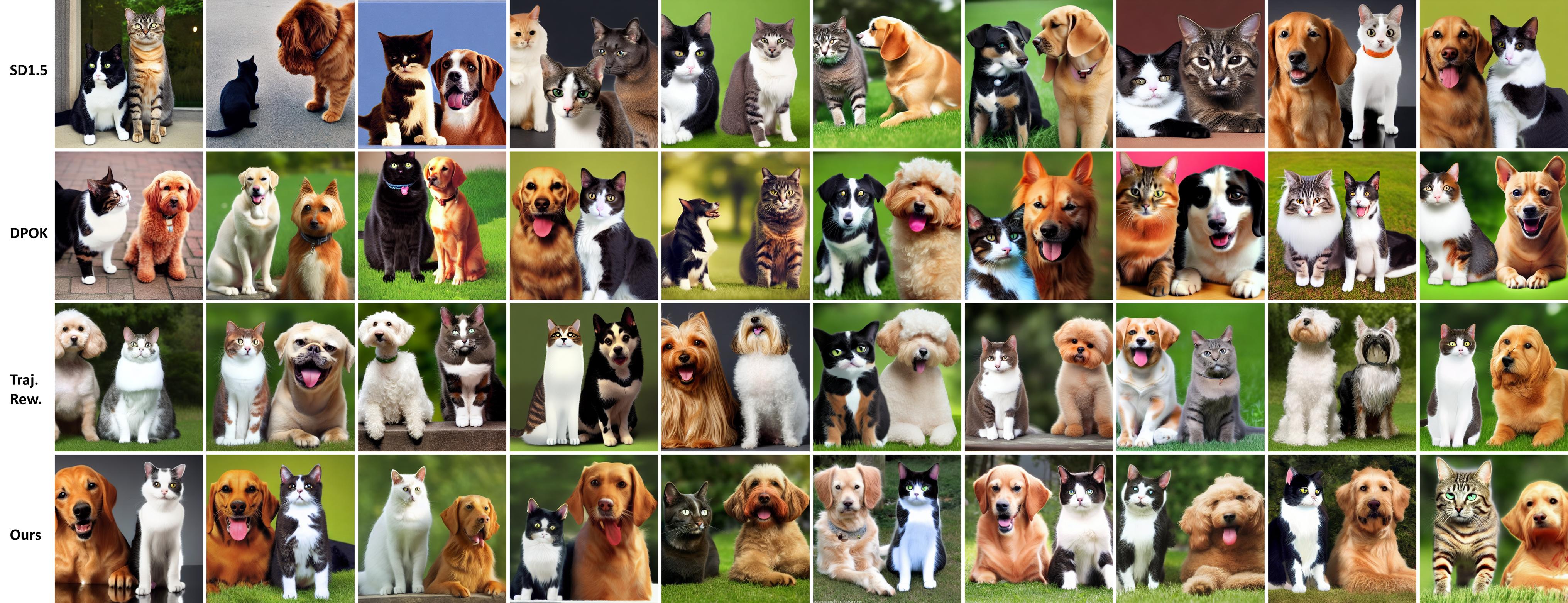}
     \vspace{-1.5em}
     \captionsetup{font=small}
        \caption{ 
        \small
        Single prompt experiment: \textit{randomly sampled} generated images for the prompt ``A cat and a dog.'', from our method and the baselines in Fig.~\ref{fig:single_prompt_generated_images}.
        ``Traj. Rew.'' denotes the classical DPO-style objective of assuming trajectory-level reward (\Secref{sec:connect_with_dpo}).
        }
        % \label{fig:single_prompt_generated_images}    % no need for label in appendix for now
        % \vspace{-.5em}
\end{figure}

\begin{figure}[H]
     % \vspace{-1.0mm}
     \centering
\includegraphics[width=1.\textwidth]{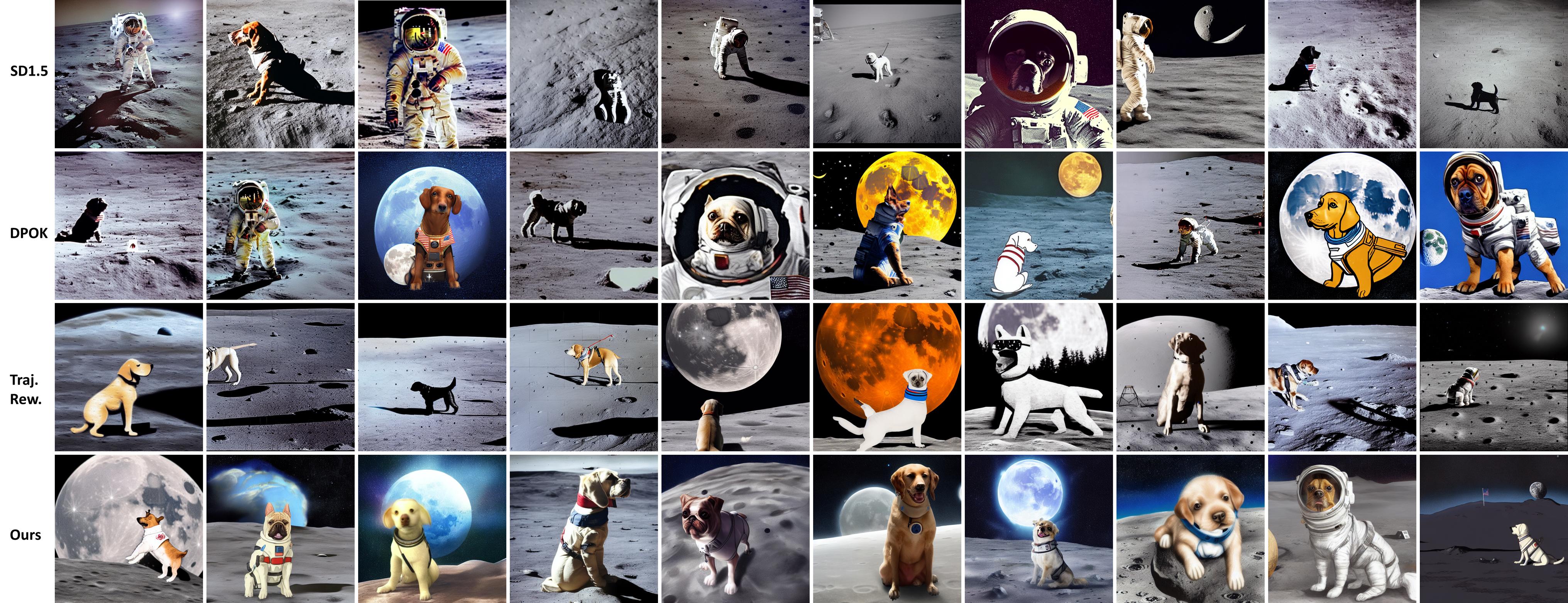}
     \vspace{-1.5em}
     \captionsetup{font=small}
        \caption{ 
        \small
        Single prompt experiment: \textit{randomly sampled} generated images for the prompt ``A dog on the moon.'', from our method and the baselines in Fig.~\ref{fig:single_prompt_generated_images}.
        ``Traj. Rew.'' denotes the classical DPO-style objective of assuming trajectory-level reward (\Secref{sec:connect_with_dpo}).
        }
        % \label{fig:single_prompt_generated_images}    % no need for label in appendix for now
        % \vspace{-.5em}
\end{figure}
\mbox{}

\begin{figure}[H]
     % \vspace{-1.0mm}
     \centering
\includegraphics[width=1.\textwidth]{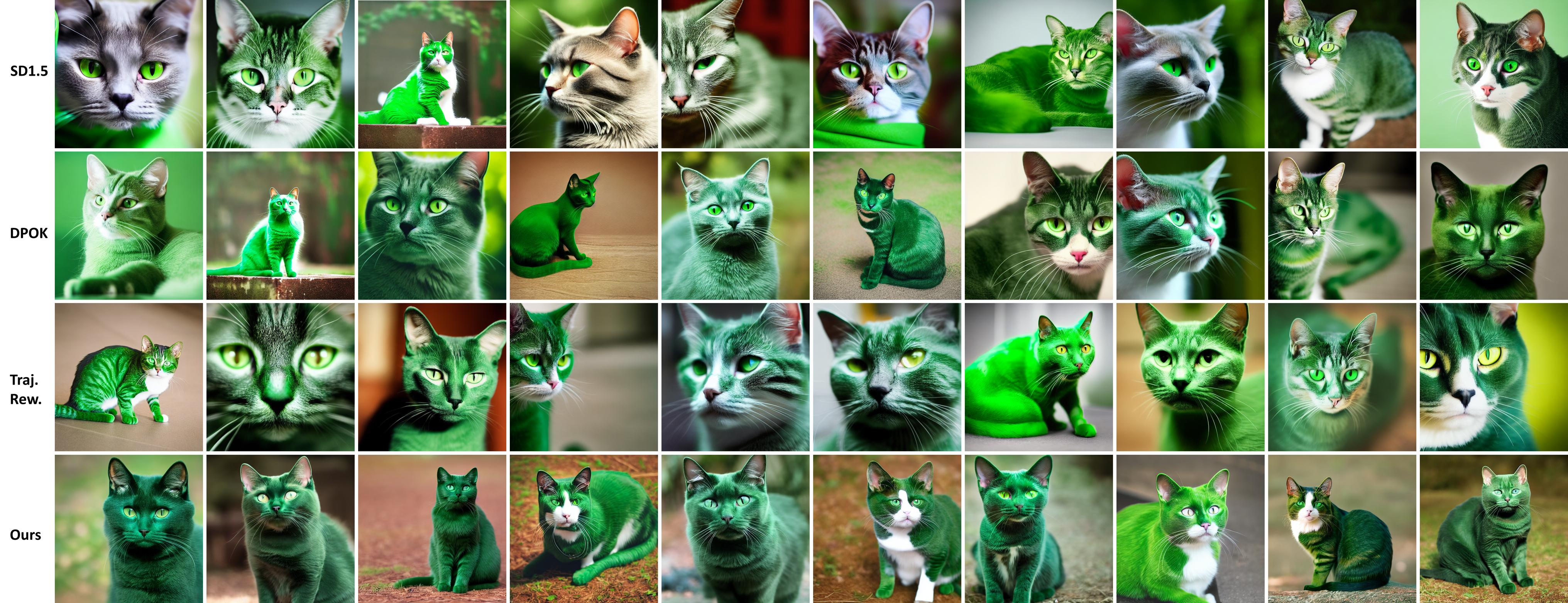}
     \vspace{-1.5em}
     \captionsetup{font=small}
        \caption{ 
        \small
        Single prompt experiment: \textit{randomly sampled} generated images for the prompt ``A green colored cat.'', from our method and the baselines in Fig.~\ref{fig:single_prompt_generated_images}.
        ``Traj. Rew.'' denotes the classical DPO-style objective of assuming trajectory-level reward (\Secref{sec:connect_with_dpo}).
        }
        % \label{fig:single_prompt_generated_images}    % no need for label in appendix for now
        % \vspace{-.5em}
\end{figure}

\begin{figure}[H]
     % \vspace{-1.0mm}
     \centering
\includegraphics[width=1.\textwidth]{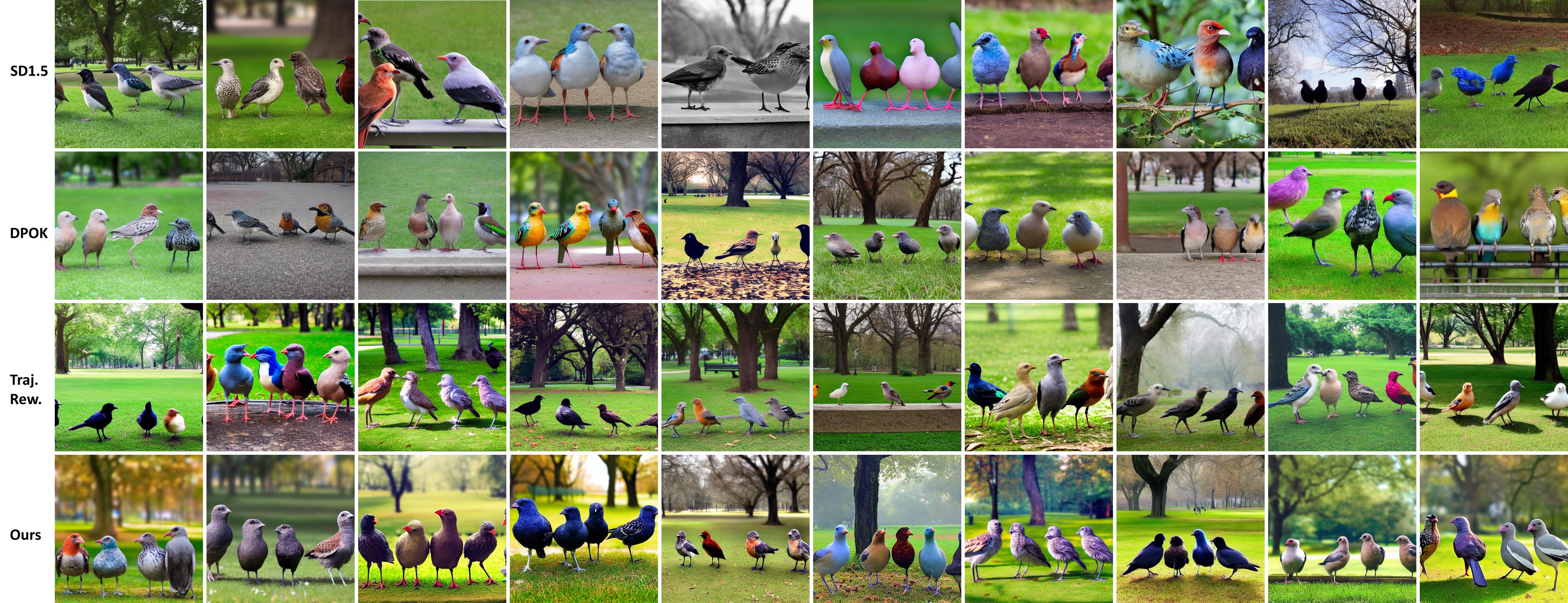}
     \vspace{-1.5em}
     \captionsetup{font=small}
        \caption{ 
        \small
        Single prompt experiment: \textit{randomly sampled} generated images for the prompt ``Four birds in the park.'', from our method and the baselines in Fig.~\ref{fig:single_prompt_generated_images}.
        ``Traj. Rew.'' denotes the classical DPO-style objective of assuming trajectory-level reward (\Secref{sec:connect_with_dpo}).
        }
        % \label{fig:single_prompt_generated_images}    % no need for label in appendix for now
        % \vspace{-.5em}
\end{figure}
\mbox{}

\begin{figure}[H]
     % \vspace{-1.0mm}
     \centering
\includegraphics[width=1.\textwidth]{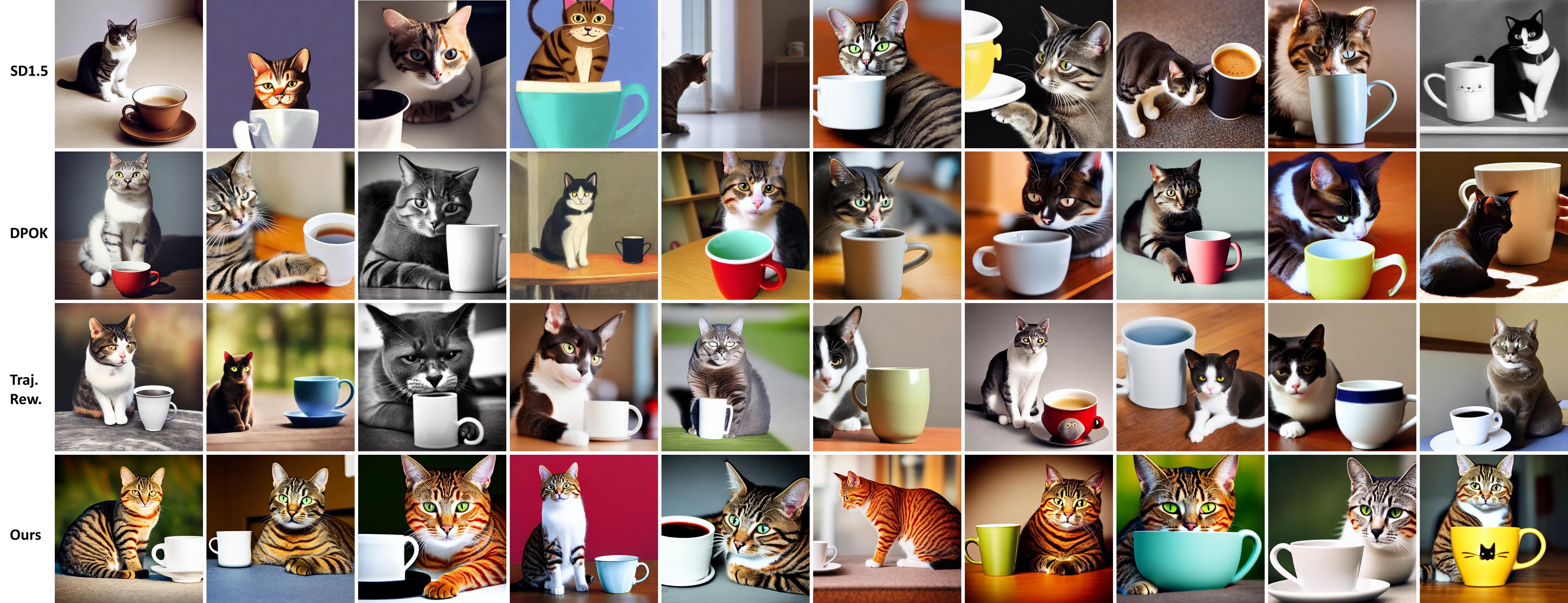}
     \vspace{-1.5em}
     \captionsetup{font=small}
        \caption{ 
        \small
        Single prompt experiment: \textit{randomly sampled} generated images for the prompt ``A cat and a cup.'', from our method and the baselines in Fig.~\ref{fig:single_prompt_generated_images}.
        ``Traj. Rew.'' denotes the classical DPO-style objective of assuming trajectory-level reward (\Secref{sec:connect_with_dpo}).
        }
        % \label{fig:single_prompt_generated_images}    % no need for label in appendix for now
        % \vspace{-.5em}
\end{figure}

\begin{figure}[H]
     % \vspace{-1.0mm}
     \centering
\includegraphics[width=1.\textwidth]{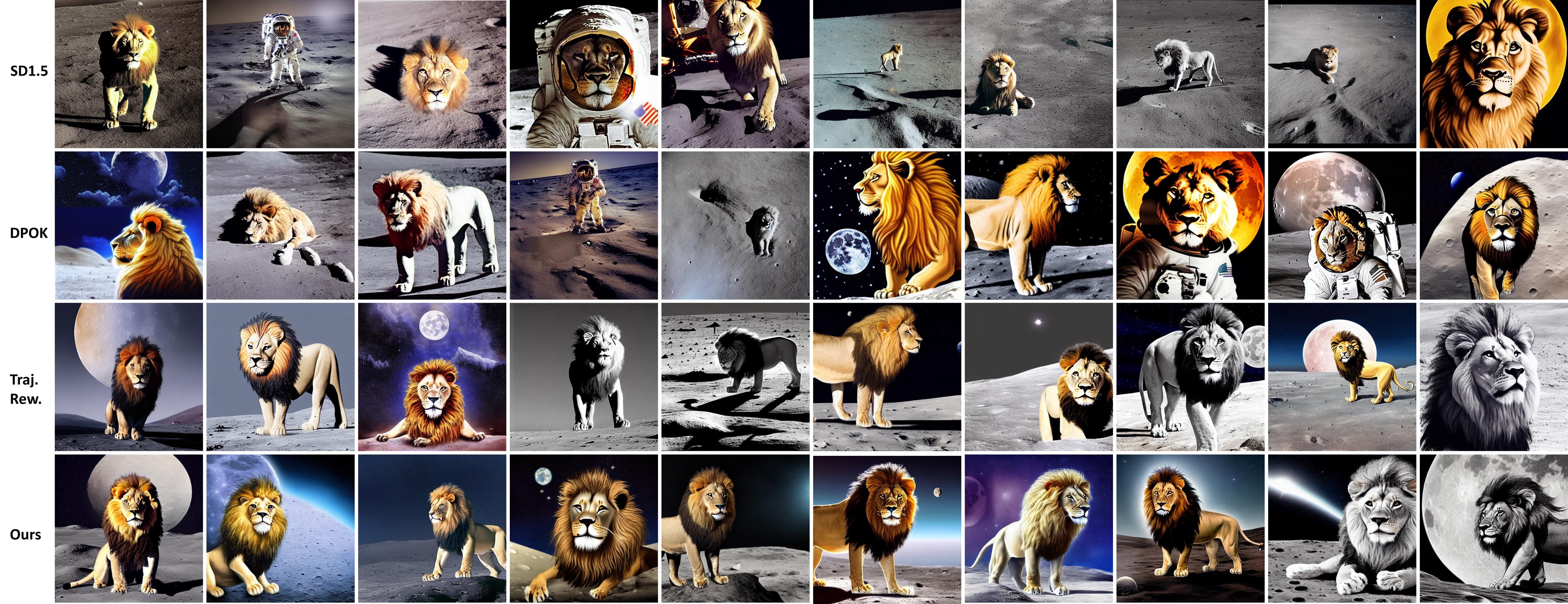}
     \vspace{-1.5em}
     \captionsetup{font=small}
        \caption{ 
        \small
        Single prompt experiment: \textit{randomly sampled} generated images for the prompt ``A lion on the moon.'', from our method and the baselines in Fig.~\ref{fig:single_prompt_generated_images}.
        ``Traj. Rew.'' denotes the classical DPO-style objective of assuming trajectory-level reward (\Secref{sec:connect_with_dpo}).
        }
        % \label{fig:single_prompt_generated_images}    % no need for label in appendix for now
        % \vspace{-.5em}
\end{figure}
\mbox{}

\subsection{More Images from the Multiple Prompt Experiment} \label{sec:multiple_prompt_more_images}

\begin{figure}[H]
     % \vspace{-1.0mm}
     \centering
\includegraphics[width=1.\textwidth]{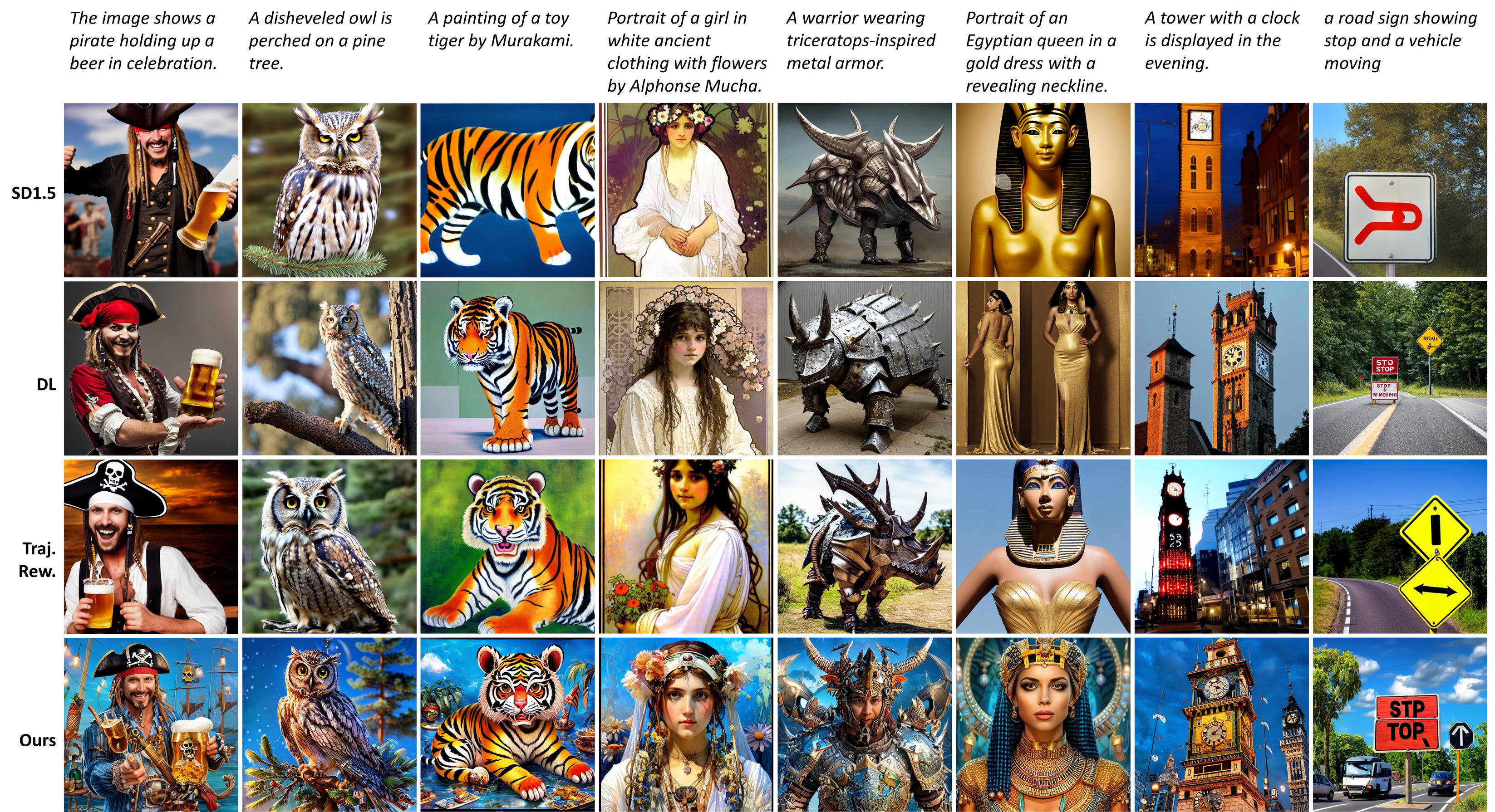}
     \vspace{-1.5em}
     \captionsetup{font=small}
        \caption{ 
        \small
        Multiple prompt experiment: generated images from our method and the baselines in Fig.~\ref{fig:multiple_prompt_generated_images} on \textit{randomly sampled} prompts from the HPSv2 test set. 
        ``DL'' denotes Dreamlike Photoreal 2.0, the best baseline from the HPSv2 paper.
        ``Traj. Rew.'' denotes the classical DPO-style objective of assuming trajectory-level reward (\Secref{sec:connect_with_dpo}).
        }
        % \label{fig:single_prompt_generated_images}    % no need for label in appendix for now
        \vspace{-1em}
\end{figure}

\begin{figure}[H]
     % \vspace{-1.0mm}
     \centering
\includegraphics[width=1.\textwidth]{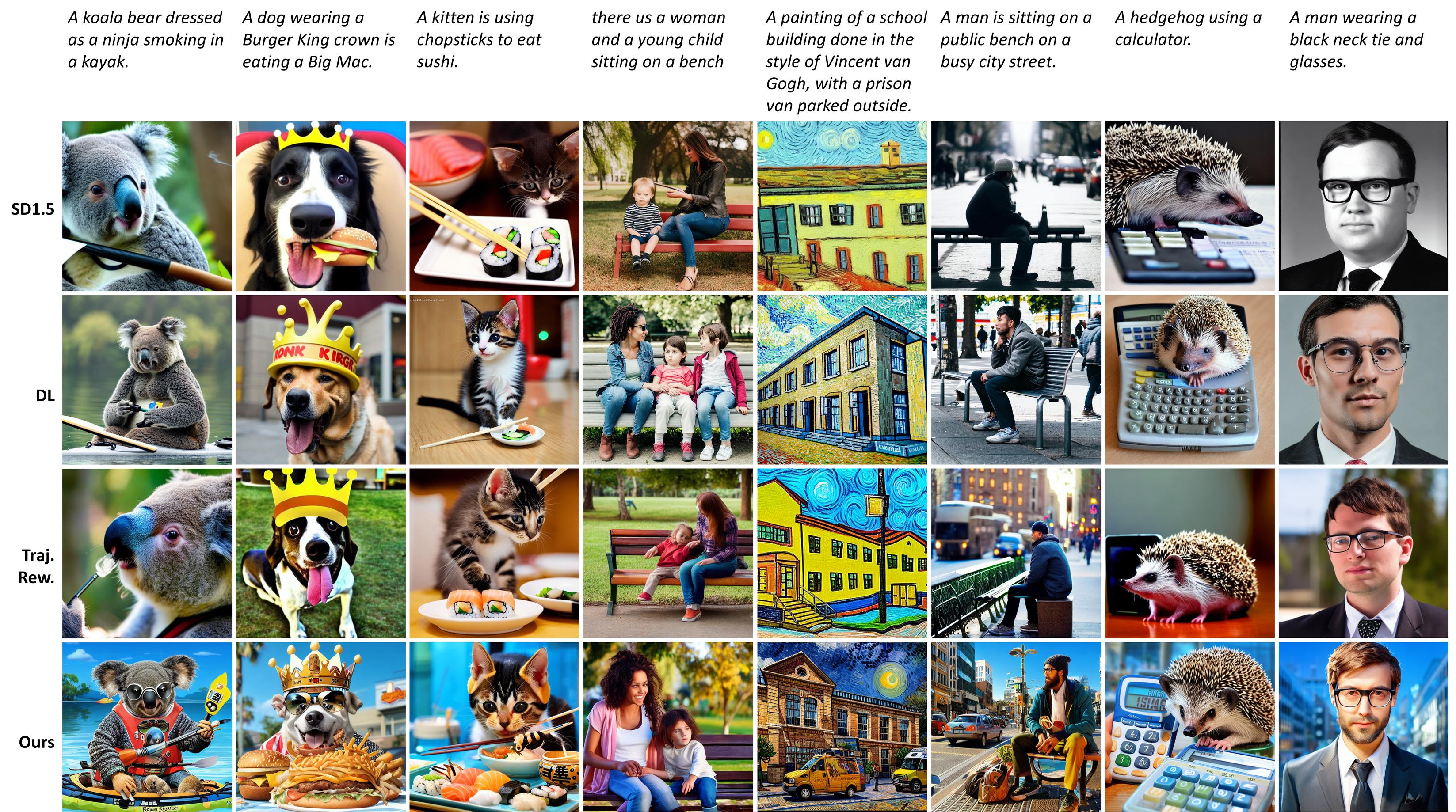}
     \vspace{-1.5em}
     \captionsetup{font=small}
        \caption{ 
        \small
        Multiple prompt experiment: generated images from our method and the baselines in Fig.~\ref{fig:multiple_prompt_generated_images} on \textit{randomly sampled} prompts from the HPSv2 test set. 
        ``DL'' denotes Dreamlike Photoreal 2.0, the best baseline from the HPSv2 paper.
        ``Traj. Rew.'' denotes the classical DPO-style objective of assuming trajectory-level reward (\Secref{sec:connect_with_dpo}).
        }
        % \label{fig:single_prompt_generated_images}    % no need for label in appendix for now
        % \vspace{-.5em}
\end{figure}

\begin{figure}[H]
     % \vspace{-1.0mm}
     \centering
\includegraphics[width=1.\textwidth]{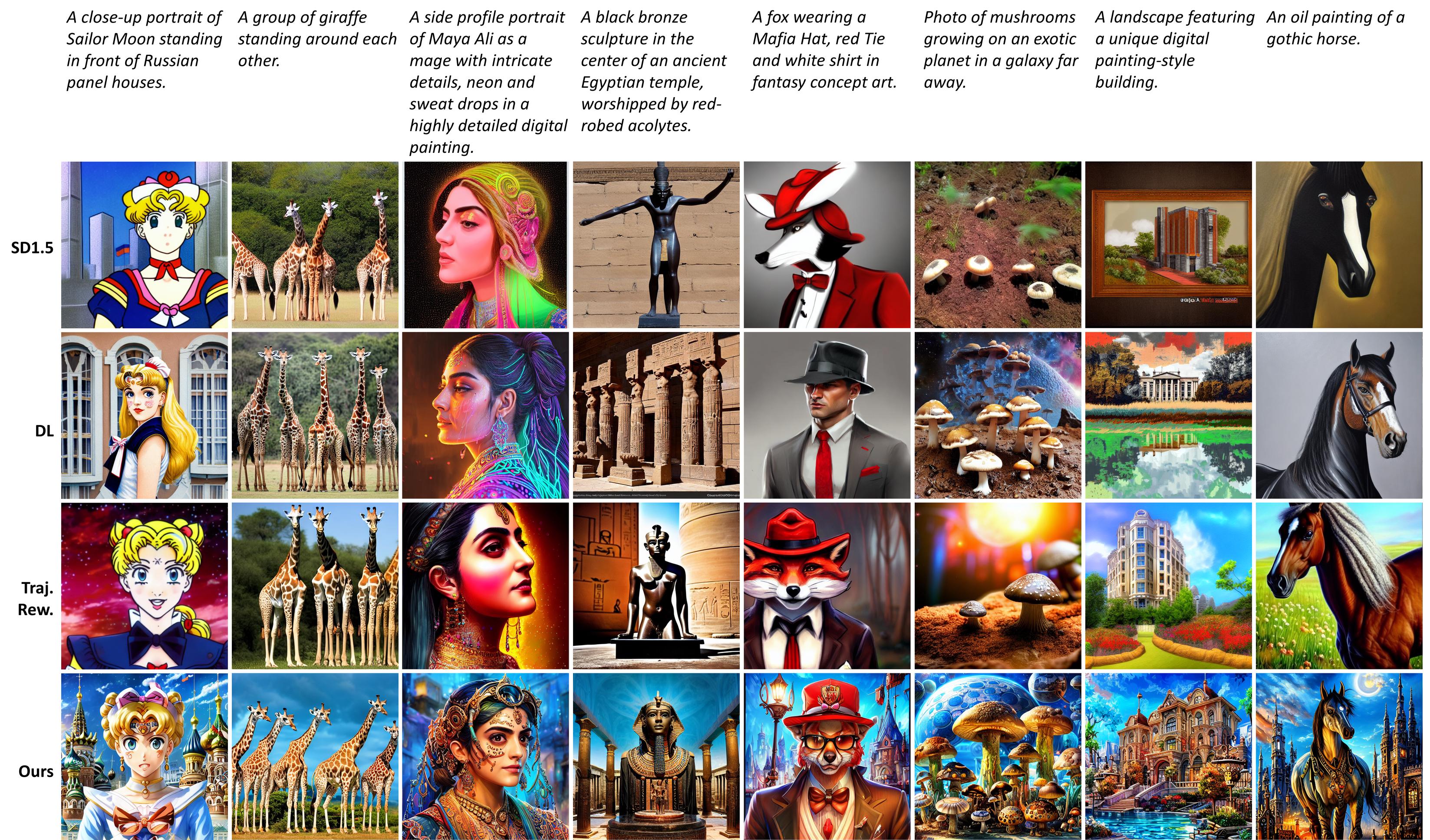}
     \vspace{-1.7em}
     \captionsetup{font=small}
        \caption{ 
        \small
        Multiple prompt experiment: generated images from our method and the baselines in Fig.~\ref{fig:multiple_prompt_generated_images} on \textit{randomly sampled} prompts from the HPSv2 test set. 
        ``DL'' denotes Dreamlike Photoreal 2.0, the best baseline from the HPSv2 paper.
        ``Traj. Rew.'' denotes the classical DPO-style objective of assuming trajectory-level reward (\Secref{sec:connect_with_dpo}).
        }
        % \label{fig:single_prompt_generated_images}    % no need for label in appendix for now
        \vspace{-1.4em}
\end{figure}

\begin{figure}[H]
     % \vspace{-1.0mm}
     \centering
\includegraphics[width=1.\textwidth]{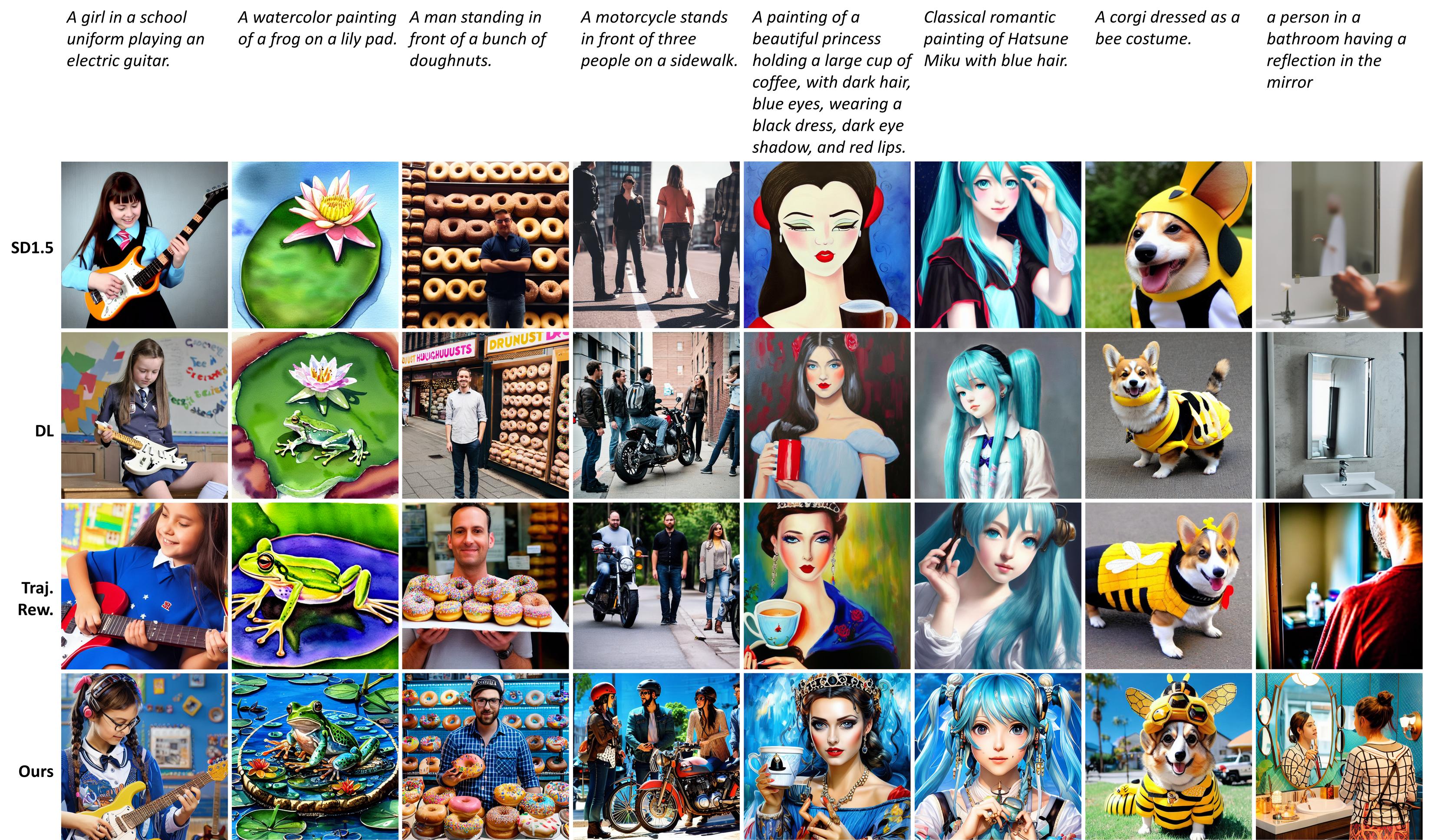}
     \vspace{-1.7em}
     \captionsetup{font=small}
        \caption{ 
        \small
        Multiple prompt experiment: generated images from our method and the baselines in Fig.~\ref{fig:multiple_prompt_generated_images} on \textit{randomly sampled} prompts from the HPSv2 test set. 
        ``DL'' denotes Dreamlike Photoreal 2.0, the best baseline from the HPSv2 paper.
        ``Traj. Rew.'' denotes the classical DPO-style objective of assuming trajectory-level reward (\Secref{sec:connect_with_dpo}).
        }
        % \label{fig:single_prompt_generated_images}    % no need for label in appendix for now
        % \vspace{-.5em}
\end{figure}

\begin{figure}[H]
     % \vspace{-1.0mm}
     \centering
\includegraphics[width=1.\textwidth]{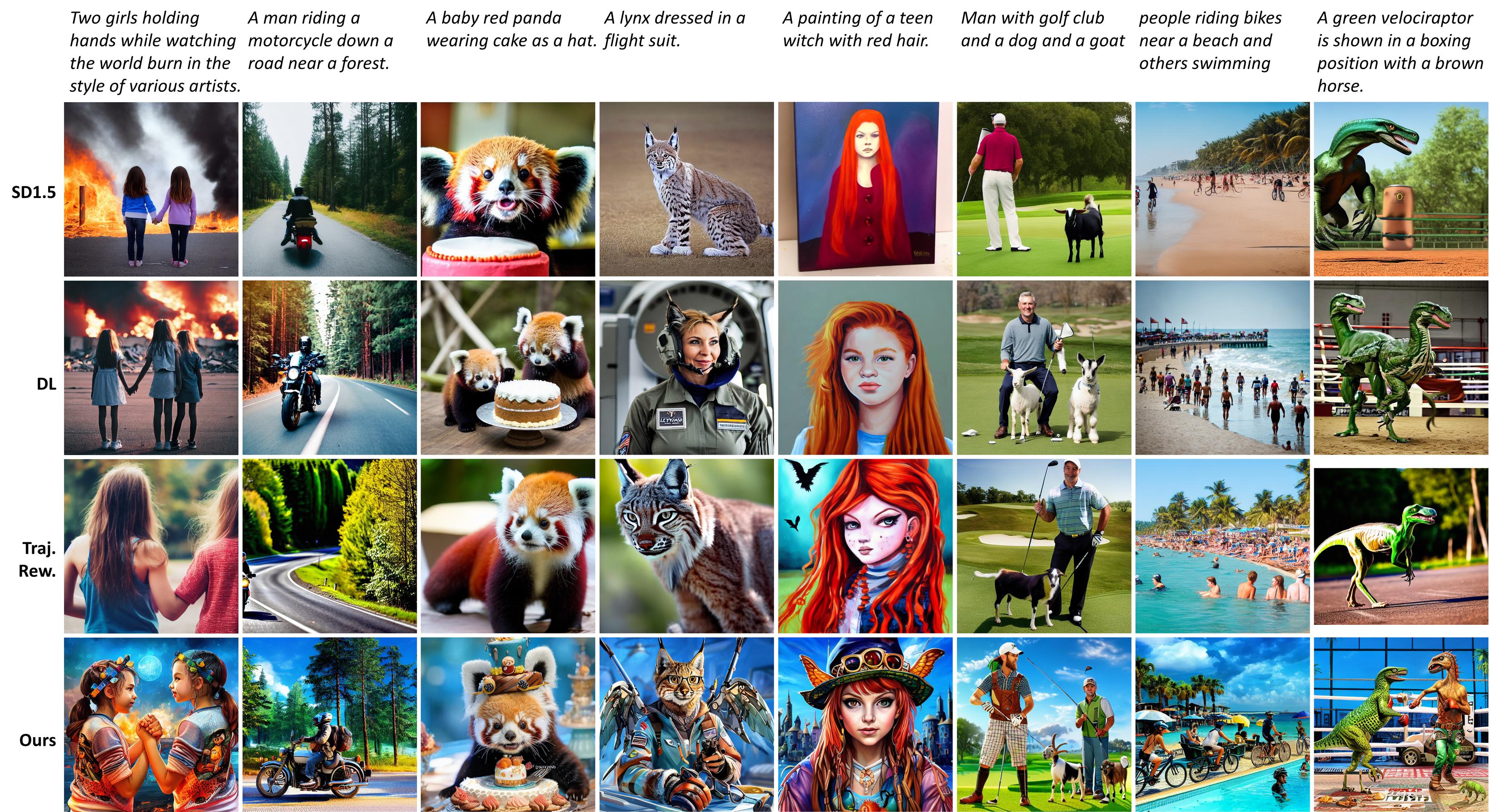}
     \vspace{-1.5em}
     \captionsetup{font=small}
        \caption{ 
        \small
        Multiple prompt experiment: generated images from our method and the baselines in Fig.~\ref{fig:multiple_prompt_generated_images} on \textit{randomly sampled} prompts from the HPSv2 test set. 
        ``DL'' denotes Dreamlike Photoreal 2.0, the best baseline from the HPSv2 paper.
        ``Traj. Rew.'' denotes the classical DPO-style objective of assuming trajectory-level reward (\Secref{sec:connect_with_dpo}).
        }
        % \label{fig:single_prompt_generated_images}    % no need for label in appendix for now
        % \vspace{-.5em}
\end{figure}

\subsection{More Generation Trajectories}\label{sec:more_traj_single_prompt}
% More Generation Trajectories from the Single Prompt Experiment
Recall that for all generation trajectories, we present the (decoded)  $\hat \vx_0$ predicted from the latents at the specified timesteps of the diffusion/T2I reverse chain.
A brief discussion on each figure is in its caption.

\begin{figure}[H]
     % \vspace{-1.0mm}
     \centering
\includegraphics[width=1.\textwidth]{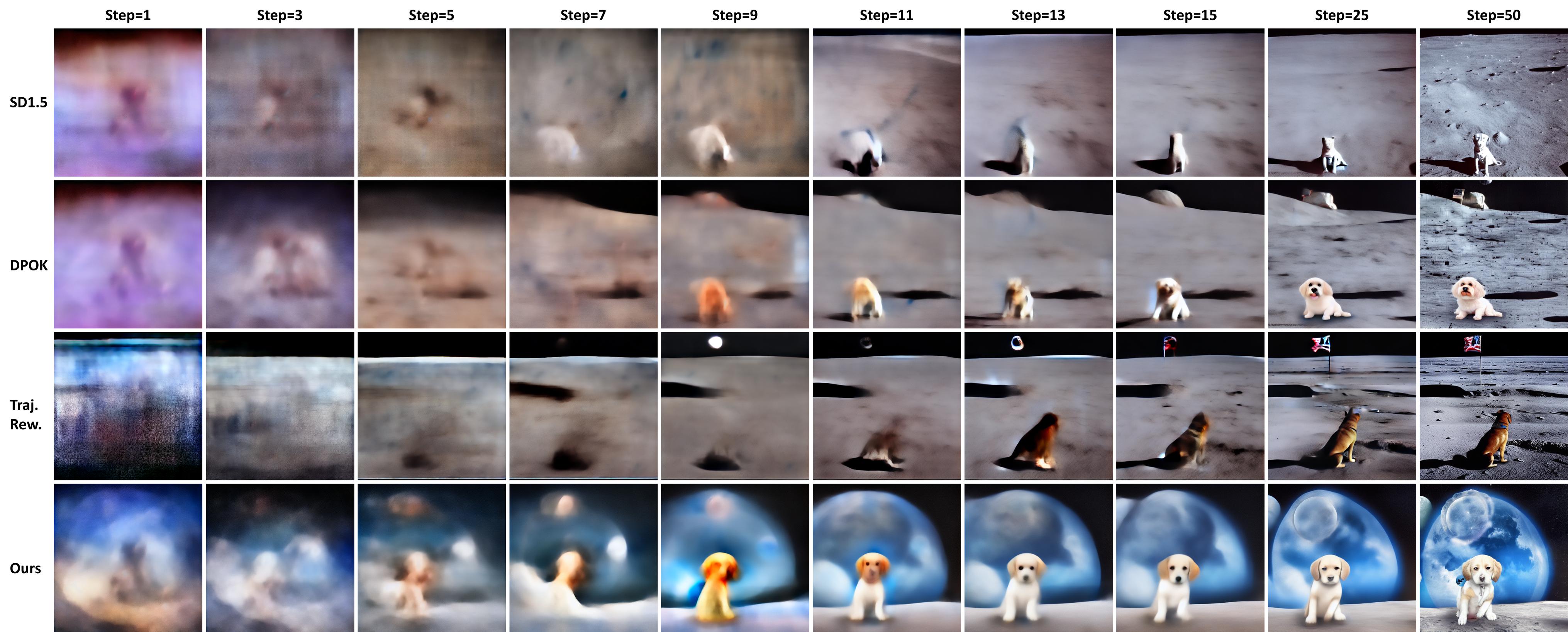}
     \vspace{-1.5em}
     \captionsetup{font=small}
        \caption{ 
        \small
        Generation trajectories for the prompt ``A dog on the moon.'', correspond to the images in Fig.~\ref{fig:single_prompt_generated_images} from our method and the baselines.
        ``Traj. Rew.'' denotes the classical DPO-style objective of assuming trajectory-level reward (\Secref{sec:connect_with_dpo}).
        Our method generates the required shape of a \textit{dog} as early as at Step $11$, when the shapes in the baselines are mostly unrecognizable.
        At Step $13$, our method is able to give a rather complete generation for the input prompt.
        Subsequent steps in the reverse chain are then allocated to polish the image details, leading to better final image.
        }
        % \label{fig:single_prompt_generated_images}    % no need for label in appendix for now
        % \vspace{-1em}
\end{figure}
\mbox{}

\begin{figure}[H]
     % \vspace{-1.0mm}
     \centering
\includegraphics[width=1.\textwidth]{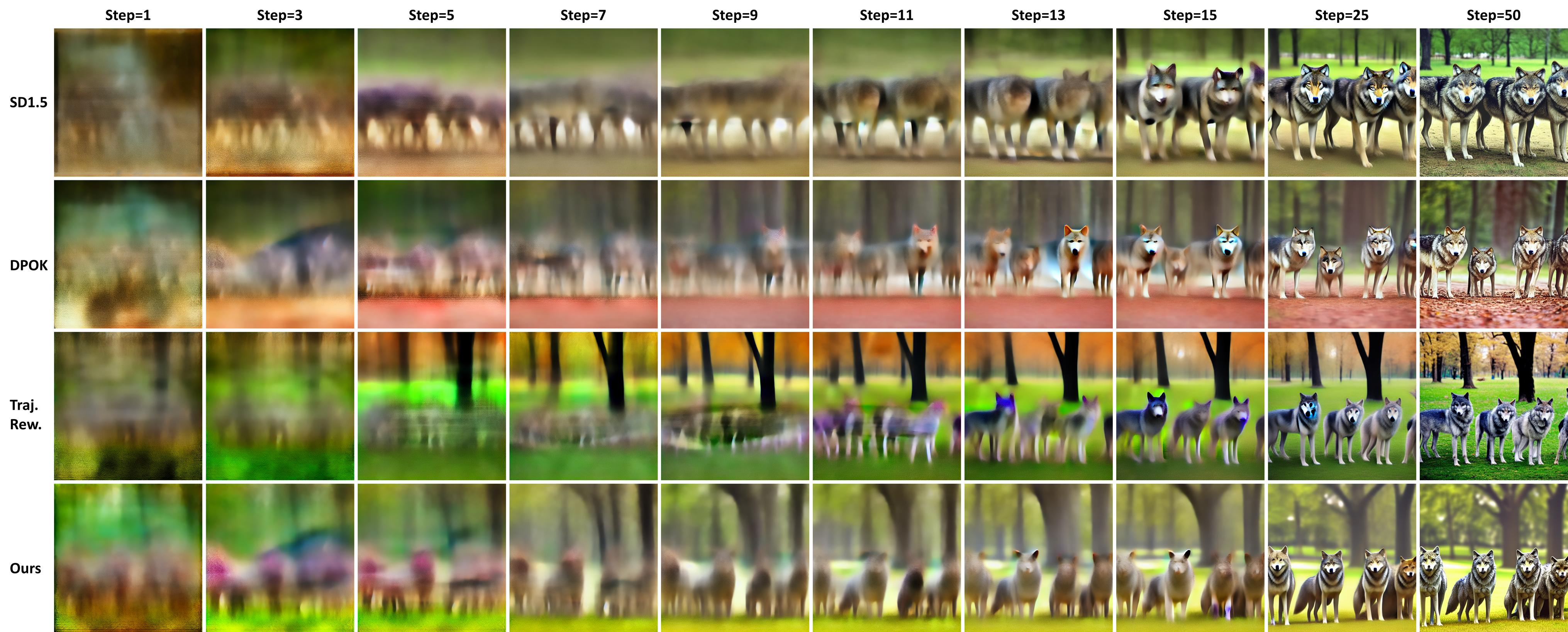}
     \vspace{-1.5em}
     \captionsetup{font=small}
        \caption{ 
        \small
        Generation trajectories for the prompt ``Four wolves in the park.'', correspond to the images in Fig.~\ref{fig:single_prompt_generated_images} from our method and the baselines.
        ``Traj. Rew.'' denotes the classical DPO-style objective of assuming trajectory-level reward (\Secref{sec:connect_with_dpo}).
        Our method generates outlines of the requires shapes (four wolves) as early as at Steps $9$ and $11$, earlier than the  baselines especially the ``Traj. Rew.''.
        }
        % \label{fig:single_prompt_generated_images}    % no need for label in appendix for now
        % \vspace{-1em}
\end{figure}
\mbox{}

\begin{figure}[H]
     % \vspace{-1.0mm}
     \centering
\includegraphics[width=1.\textwidth]{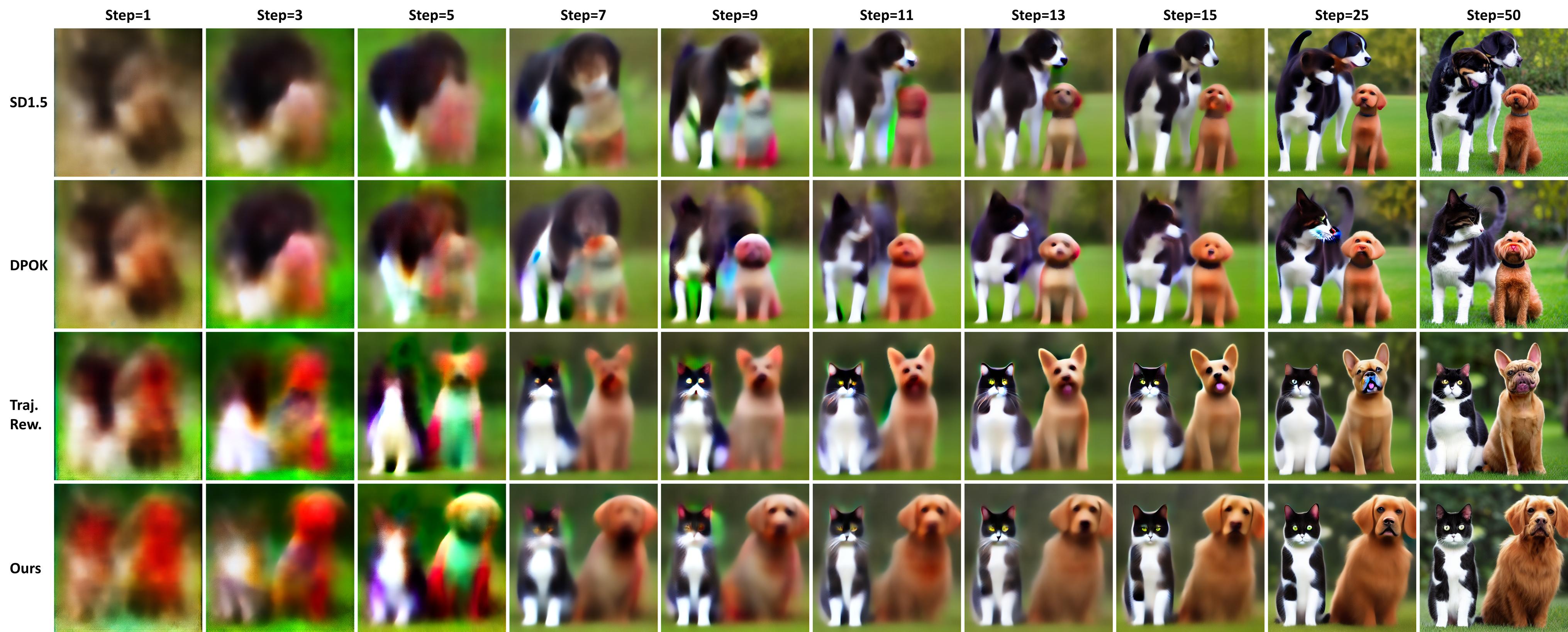}
     \vspace{-1.5em}
     \captionsetup{font=small}
        \caption{ 
        \small
        Generation trajectories for the prompt ``A cat and a dog.'', correspond to the images in Fig.~\ref{fig:single_prompt_generated_images} from our method and the baselines.
        ``Traj. Rew.'' denotes the classical DPO-style objective of assuming trajectory-level reward (\Secref{sec:connect_with_dpo}).
        Our method generates the outlines of the desired shapes as fast as at Steps $3$ and $5$, especially when compared to the baselines DPOK and raw SD1.5.
        This helps our method in generating a more reasonable and better final image.
        }
        % \label{fig:single_prompt_generated_images}    % no need for label in appendix for now
        % \vspace{-1em}
\end{figure}

%%%%%%%%%%%%%%%%%%%%%%%%%%%%%%%%%%%%%%%%%%%%%%%%%%%%%%%%%%%%%%%%%%%%%%%%%%%%%%%
%%%%%%%%%%%%%%%%%%%%%%%%%%%%%%%%%%%%%%%%%%%%%%%%%%%%%%%%%%%%%%%%%%%%%%%%%%%%%%%

\end{document}